%% file: main.tex
\documentclass[11pt]{article}

\usepackage{times}
\usepackage[margin=1in]{geometry}
\usepackage{natbib}
\setcitestyle{authoryear,round,citesep={;},aysep={,},yysep={;}}

\usepackage[utf8]{inputenc}
\usepackage[T1]{fontenc}
\usepackage{url}
\usepackage{booktabs}
\usepackage{amsfonts}
\usepackage{amsmath}
\usepackage{amssymb}
\usepackage{amsthm}
\usepackage{mathtools}
\usepackage{nicefrac}
\usepackage{microtype}
\usepackage{graphicx}
\usepackage{float}
\usepackage{placeins}
\usepackage{xcolor}
\usepackage{enumitem}
\usepackage{hyperref}
\hypersetup{
  colorlinks=true,
  linkcolor=blue,
  citecolor=blue,
  urlcolor=blue
}
\let\oldeqref\eqref
\renewcommand{\eqref}[1]{Eq.~\oldeqref{#1}}

\newtheorem{theorem}{Theorem}[section]
\newtheorem{proposition}[theorem]{Proposition}
\newtheorem{lemma}[theorem]{Lemma}
\newtheorem{corollary}[theorem]{Corollary}
\theoremstyle{definition}
\newtheorem{assumption}[theorem]{Assumption}
\theoremstyle{remark}

\title{A Fine-Grained Analysis of the LoRA Fine-Tuning Landscape with Implications for Data Selection}

\author{
Bowen Zhang\textsuperscript{1},
Changrui Fang\textsuperscript{1},
Xinsong Ma\textsuperscript{1},
Jiaye Teng\textsuperscript{2},
Ziye Ma\textsuperscript{1}\thanks{Corresponding author.}\\[0.5em]
\textsuperscript{1}City University of Hong Kong\\
\textsuperscript{2}Shanghai University of Finance and Economics
}
\date{}

\begin{document}

\maketitle

\begin{abstract}
Low-Rank Adaptation (LoRA) has become a standard approach for parameter-efficient fine-tuning, yet a fundamental practical question remains unresolved: how should the adapter rank be chosen? An overly small rank may lead to a poorly conditioned optimization landscape, whereas an unnecessarily large rank sacrifices the efficiency that motivates LoRA in the first place. Existing theoretical analyses provide only limited guidance on this trade-off, and their guarantees are typically established under restrictive theoretical settings. We address this gap by developing a substantially sharper landscape theory for LoRA, building on modern results from nonconvex low-rank matrix sensing. Our central insight is that the appropriate adapter rank should depend on the quality of the data-induced optimization geometry, rather than on the model alone. To formalize this connection, we introduce LoRA-RIP, a data-dependent restricted-isometry metric that characterizes the conditioning of the cross-entropy (CE) objective along LoRA-relevant low-rank directions. We prove that sufficient rank over-parameterization, with the required rank explicitly determined by the LoRA-RIP constant, eliminates spurious local minima, thereby extending existing RIP-based guarantees beyond the classical $1/3$ regime. This characterization further enables principled data selection under a fixed rank budget. Experiments across language and vision tasks support these theoretical predictions, showing that rank and data quality are two coupled resources that should be jointly considered for more efficient and reliable LoRA fine-tuning.
\end{abstract}

\section{Introduction}

Low-Rank Adaptation (LoRA) is a standard parameter-efficient method for fine-tuning large models~\citep{hu2021lora}. Given a pretrained weight matrix \(W_0\in\mathbb R^{n_1\times n_2}\), LoRA freezes \(W_0\) and learns an additive low-rank update \(W=W_0+\gamma M\), where \(M=UV^\top\), \(U\in\mathbb R^{n_1\times r}\), \(V\in\mathbb R^{n_2\times r}\), and $\gamma$ is a scaling factor. With \(r\ll\min(n_1,n_2)\), this factorization reduces the trainable parameter count for the adapted matrix from \(n_1n_2\) to \(r(n_1+n_2)\), lowering optimizer-state and memory requirements while often maintaining strong downstream performance. Realizing these efficiency gains, however, requires an appropriate choice of the adapter rank. Although larger ranks offer greater expressivity~\citep{zeng2024expressive}, their practical value cannot be judged from representational capacity alone. In linearized fine-tuning experiments, \citet{jang2024ntk} observe that lower ranks converge more slowly even when they ultimately reach the same loss. Conversely, increasing rank need not improve downstream performance. For example, \citet{albertRandLoRAFullRank2024} report that, when fine-tuning Phi3 on 15,000 examples, increasing the LoRA rank from $16$ to $32$ reduces the reported mean accuracy across eight commonsense reasoning tasks from $81.8\%$ to $80.3\%$. These observations illustrate the cost of an unsuitable rank choice: an overly restrictive rank can slow optimization, whereas an unnecessarily large rank can consume additional resources without corresponding performance gains. Adaptive methods such as AdaLoRA improve the allocation of rank across weight matrices, but still require a prespecified target budget~\citep{zhang2023adaptivebudget}. This leaves a complementary question: how should adapter rank be chosen to support effective fine-tuning without unnecessary parameter overhead? 


This question naturally leads us to nonconvex low-rank optimization:
LoRA learns a factorized update, with the adapter rank $r$ controlling
its degree of parameterization~\citep{burerNonlinearProgrammingAlgorithm2003}.
A long line of work has characterized the landscapes of factorized
problems and established conditions excluding spurious local
minima~\citep{NIPS2016_b139e104,ge2017nospurious,
zhu2018globaloptimality,zhangGeneralLowrankMatrix2021}.
In classical matrix sensing, the restricted
isometry property (RIP) is a key metric that characterizes the optimization landscape, and it measures how uniformly the sensing data
capture low-rank structure, making its constant a natural indicator
of the underlying data geometry~\citep{recht2010guaranteed}.
More recent advances show that rank over-parameterization can
improve the landscape by eliminating spurious local minima under
weaker RIP conditions~\citep{zhangImprovedGlobalGuarantees2025,
zhangSharpGlobalGuarantees2025}, directly connecting search rank
with data-induced geometry. Classical sensing, however, assumes a
fixed $\ell_2$ objective in matrix space, whereas LoRA settings have
cross-entropy loss coupling with the base model structure. Its
geometry therefore varies with the model's responses to the data
and the adapter update, making the extension of classical rank--RIP
analysis nontrivial. Specifically for the LoRA fine-tuning landscape,
\citet{jang2024ntk} establish no-spurious-minima guarantees in the
neural tangent kernel (NTK) regime, while \citet{liu2025landscape} further 
relate adapter rank to local optimization geometry under $\ell_2$
loss. These results demonstrate the relevance of landscape theory,
but their reliance on local linearization or squared-loss settings
limits the guidance they provide for nonlinear cross-entropy
fine-tuning. More recently, \citet{kim2025failsloudly} move beyond
linearization by adapting established low-rank optimization
arguments~\citep{zhu2018globaloptimality,ha2020equivalence},
showing how RIP-type conditions still govern the landscapes in nonlinear LoRA settings
and providing numerical evidence for their applicability.
Their global-optimality guarantee in the analyzed region
corresponds to a RIP constant below $1/3$, while larger values
allow a ``fails-loudly'' alternative. However, their analysis leaves two important aspects insufficiently
characterized: the CE-specific structure of the RIP constant in LoRA,
and how the adapter rank $r$ shapes the optimization landscape beyond
the small-RIP regime. Since RIP is strongly determined by the geometry
of the fine-tuning data, these two questions are closely connected:
understanding LoRA optimization requires relating data-induced RIP to
the role of rank in shaping the landscape. Clarifying this relationship is therefore a key step toward a more principled theory of rank choice in LoRA fine-tuning.

In this work, we build on theoretical tools from classical low-rank
optimization and matrix sensing
~\citep{zhangGeneralLowrankMatrix2021,
zhangSharpGlobalGuarantees2025,yalccin2023semidefinite} and extend them to develop a fine-grained landscape characterization
of LoRA under cross-entropy loss. We introduce a new data-dependent LoRA-RIP constant, denoted by $\delta_k^{\mathrm{LoRA}}$. Under the stated assumptions, we establish an explicit
relation between \(r/r_\star\) and
\(\delta_k^{\mathrm{LoRA}}\) that guarantees a landscape
free of spurious local minima, where $r_\star$ refers to the rank of the global solution. This implies that any algorithm that can converge to a
second-order stationary point (SOSP)\footnote{For definition of SOSP
please see Appendix~\ref{sec:app-sosp-definitions}} in polynomial time
can reach the global solution reliably. More importantly, our results hold in the regime where $\delta_k^{\mathrm{LoRA}}$
exceeds $1/3$, and
show that larger ranks permit larger LoRA-RIP values in the guarantee. Under common curvature bounds covering the ranks considered,
this characterization yields a sufficient constant-factor
rank requirement relative to \(r_\star\), with the factor
determined by the LoRA-RIP bound.
Once this requirement is met, further rank is unnecessary
for the same guarantee within that range. Therefore it means that in most LoRA settings, $r$ does not need to be excessively large to ensure good performance, preserving its parameter efficiency. Viewed from the data side, the same relationship suggests a complementary route to effective fine-tuning: improving the optimization geometry through data selection under a fixed rank budget. We therefore analyze the CE-specific structure of LoRA-RIP, revealing how certain properties of fine-tuning data affect its value and deriving corresponding data-selection principles. Together, these results connect rank requirements with data composition, providing theoretical guidance for choosing adapter rank and selecting data for more parameter-and data-efficient LoRA fine-tuning.

\paragraph{Main Contributions:}
\begin{itemize}[leftmargin=1.5em]
    \item We develop a theoretical framework linking adapter rank \(r\) and a newly introduced notion LoRA-RIP \(\delta_k^{\mathrm{LoRA}}\) to global optimality in the balanced LoRA landscape. We show that as long as $r/ r_\star >
    \left(2\delta_k^{\mathrm{LoRA}} / (1-\delta_k^{\mathrm{LoRA}}) \right)^2 $ , no spurious local minima exist in the landscape of \eqref{eq:main-balanced-lora-objective}, where $r_\star$ denotes the rank of the global solution $M^*$ in Assumption\ref{asm:main-regularized-target}. This result establishes a sufficient rank threshold for a benign landscape. Our experiments further illustrate rank-dependent optimization trends consistent with this characterization.
    \item We identify theoretically motivated data-selection principles such as latent-space coverage and pre-tuning softmax saturation. Our analysis links data composition to LoRA-RIP and shows how more favorable data geometry reduces the amount of rank-over-parameterization needed for a benign landscape. This yields a landscape-guided view of data selection in which data quality and adapter rank act as complementary resources.
\end{itemize}

\section{Related Work}

\paragraph{LoRA methods and theory.}
Low-Rank Adaptation (LoRA)~\citep{hu2021lora} has become a central parameter-efficient fine-tuning method, and numerous variants have been built on its structure, including quantized adapters~\citep{dettmers2023qlora}, adaptive rank
allocation~\citep{zhang2023adaptivebudget}, and modified factor
updates~\citep{hayou2024loraplus}. Other work studies geometric updates and stabilization
mechanisms~\citep{zhang2024riemannian,yen2025rite,wu2026stable,zhao2024galore}. Recent theory has studied LoRA from complementary angles. Expressivity results characterize what a
given adapter rank can represent~\citep{zeng2024expressive}. From the view of optimization, NTK-regime analysis gives benign
landscape guarantees for a linearized fine-tuning model~\citep{jang2024ntk}. Other works analyze
nonlinear LoRA landscapes or gradient-flow dynamics under structural assumptions
~\citep{liu2025landscape,kim2025failsloudly,pmlr-v258-xu25h}.

\paragraph{Low-rank optimization.}

Low-rank matrix optimization provides a geometric foundation for understanding factorized
objectives. A rank-constrained variable is often written as \(X=UV^\top\) by Burer-Monteiro factorization~\citep{burerNonlinearProgrammingAlgorithm2003}, which replaces the rank
constraint by a nonconvex optimization problem over the factors. For matrix recovery and related
objectives, restricted curvature or RIP-type conditions can imply benign landscape properties, such
as the absence of spurious local minima or strict-saddle geometry
~\citep{ge2017nospurious,zhu2018globaloptimality}. Similar geometric analyses have also been
developed for more general low-rank objectives~\citep{zhangGeneralLowrankMatrix2021}. A more
recent direction studies rank over-parameterization: when the search rank exceeds the target rank
\(r_\star\), the required amount of over-parameterization can be expressed through RIP quantities~\citep{zhangImprovedGlobalGuarantees2025,zhangSharpGlobalGuarantees2025}.


\paragraph{Data selection and training geometry.} As innovations in model architecture hit a plateau, the strategic selection of training data becomes increasingly crucial. Existing work includes Shapley-based data valuation~\citep{ghorbani2019data, wang2024data}, gradient-based data attribution and selection~\citep{pruthi2020estimating,xia2024less,wang2026opus} and similarity-based data filtering~\citep{sachdeva2026how,das2026matched}. Meanwhile, recent theoretical work explores the statistical properties~\citep{allouah2026distributional,rezaei2026highdimensional}. Complementary work studies training geometry through Hessian spectra and cross-entropy dynamics
~\citep{pmlr-v97-ghorbani19b,
garrodDiagonalizingSoftmaxHadamard2025a}.
In this work, we connect data composition to LoRA rank requirements through
the optimization geometry of the cross-entropy objective.

\section{{Preliminaries}}

\paragraph{One-layer LoRA Fine-tuning}
We freeze all pretrained parameters and apply a trainable
LoRA update to a single weight matrix
\(W_0\in\mathbb R^{n_1\times n_2}\).
A LoRA update takes the form
\(
W=W_0+M,
\quad
M=UV^\top,
\quad
U\in\mathbb R^{n_1\times r},
\quad
V\in\mathbb R^{n_2\times r}.
\)
The LoRA scaling factor can often be absorbed into the factors, so we adopt the normalized
format \(M=UV^\top\) throughout. Given a training set
\(S=\{(x_i,y_i)\}_{i=1}^N\), define the empirical loss as
\(
\widehat L(M)
:=
\frac1N\sum_{i=1}^N
\ell\bigl(f_{W_0+M}(x_i),y_i\bigr)
\). With weight decay $(\lambda>0)$ on the LoRA factors, the standard regularized LoRA objective is
\begin{equation}
\label{eq:main-decay-only-objective}
\widetilde h_\lambda(U,V)
:=
\widehat L(UV^\top)
+
\frac{\lambda}{2}
\bigl(
\|U\|_F^2+\|V\|_F^2
\bigr).
\end{equation}
The factorizations are not unique: if $R$ is invertible, then \((UR, V R^{-\top})\) represents the
same product as $UV^\top$. 
Thus the same product can have an infinite number of factorized representations. For the landscape analysis, we introduce a balancing term 
\(
    \|U^\top U - V^\top V\|_F^2
\)
. This balancing term improves the factor-space geometry without
restricting the set of representable updates. Such balancing
principles are well established in asymmetric low-rank optimization
and have recently been explored in LoRA
~\citep{zhu2018globaloptimality,castinBalancedLoRARemoving2026}. This leads to our main formulation:
\begin{equation}
\label{eq:main-balanced-lora-objective}
h_{\mu,\lambda}(U,V)
:=
\widehat L(UV^\top)
+
\frac{\mu}{4}\|U^\top U-V^\top V\|_F^2
+
\frac{\lambda}{2}
\bigl(
\|U\|_F^2+\|V\|_F^2
\bigr),
\end{equation}
where \(\mu>0\) is the balancing weight. In Section~\ref{sec:rank-thresholds}, we first establish results for \eqref{eq:main-balanced-lora-objective} and then transfer them back to the standard objective \eqref{eq:main-decay-only-objective}.

\paragraph{The Traditional Restricted Isometry Property(RIP) Notion} 
In classic matrix sensing, one often considers the quadratic loss of 
\(
    \|\mathcal A(M)-y\|_2^2 /2 ,
\)
where $\mathcal A$ is a linear operator consisting of a number (for example $m$) of sensing matrices $\{ A_i \}_{i = 1}^m$ that each take an inner product with $M$ to give a $m$-dimensional vector. Here we assume the target $y = \mathcal A(M^*)$, the observation of the ground truth low-rank matrix $M^*$. In such classical settings, it is well understood that the landscape and recovery guarantees are mostly determined by one metric, the restricted isometry property (RIP) constant ~\citep{recht2010guaranteed,ge2017nospurious}. We say \(\mathcal A\) satisfies 
 RIP with constant \(\delta_k\) if
\[
    (1-\delta_k)\|H\|_F^2
    \le
    \|\mathcal A(H)\|_2^2
    \le
    (1+\delta_k)\|H\|_F^2
\]
for every low-rank matrix \(H\) with \(\operatorname{rank}(H)\le k\). The RIP constant can be understood as the conditioning of the Hessian restricted to the low-rank manifold of rank at most $k$, with smaller values
corresponding to more uniform restricted curvature. Under the simple $\ell_2$ loss, the Hessian will stay constant over the entire landscape, but for our LoRA objective \eqref{eq:main-balanced-lora-objective}, Hessians vary across points, complicating direct RIP-based arguments.

\paragraph{Low-rank global minimizer exists}
One important distinction between our LoRA formulation and classical low-rank matrix recovery is that LoRA does not assume the existence of a “ground-truth” matrix $M^*$. Nevertheless, the optimization objective is analogous: in both settings, we seek global minima because they correspond to the best achievable empirical fit and, ultimately, the solutions expected to yield the strongest performance. In low-rank recovery, this target coincides with the ground-truth matrix $M^*$ by construction, whereas in LoRA fine-tuning it corresponds to the globally optimal low-rank adaptation under the training objective. Motivated by evidence that fine-tuning can
exploit low-dimensional parameterizations
~\citep{aghajanyanIntrinsicDimensionalityExplains2021a},
we make the following assumption that the global solution to the matrix lasso objective $F_\lambda(M)$ has rank $r_\star$, which is lower than the LoRA search rank $r$:
\begin{assumption}[Existence of low-rank global minimizer]
\label{asm:main-regularized-target}
\begin{equation}
\label{eq:main-regularized-target}
M_\star^\lambda
\in
\operatorname*{argmin}_{M \in \mathbb{R}^{n_1\times n_2}} F_\lambda(M),
\quad
\operatorname{rank}(M_\star^\lambda)=r_\star,
\quad
r\ge r_\star.
\end{equation}
\end{assumption}
Here, $F_\lambda(M):=\widehat L(M)+\lambda\|M\|_*$ is the matrix lasso formulation of the main objective $h_{\mu,\lambda}(U,V)$, and in Appendix \ref{sec:app-proof-prop-factor-decay} we prove that $\inf_{U,V} h_{\mu,\lambda}(U,V)
=
\inf_{\operatorname{rank}(M)\le r} F_\lambda(M)$.

\section{LoRA-RIP and Rank Thresholds for Benign Landscape}
\label{sec:rank-thresholds}

\subsection{Definition of LoRA-RIP}


\begin{assumption}[Rank-$k$ restricted strong convexity / restricted strong smoothness (RSC/RSS)]
\label{asm:main-rsc-rss}
For a candidate LoRA update \(M=UV^\top\), define the line segment as
\(M_t=M_\star^\lambda+t(M-M_\star^\lambda)\), \(t\in[0,1]\).
For a fixed rank order \(k\ge r+r_\star\), assume that
\(\widehat L\) has restricted curvature bounds
\(0<\alpha_k\le\beta_k<\infty\) along this segment:
\begin{equation}
\label{eq:main-rank-k-rsc-rss}
\alpha_k \|H\|_F^2
\le
\langle \nabla^2 \widehat L(M_t)[H], H \rangle
\le
\beta_k \|H\|_F^2
\end{equation}
for all \(t\in[0,1]\) and all directions \(H\) with
$\operatorname{rank}(H)\le k$.
\end{assumption}

\paragraph{LoRA-RIP constant}
Under Assumption~\ref{asm:main-rsc-rss}, we define the rank-\(k\) LoRA-RIP
constant by
\[
\delta_k^{\mathrm{LoRA}}
:=
\frac{\beta_k-\alpha_k}{\beta_k+\alpha_k}.
\]
For clarity, we state the main result with the balancing weight
\(\mu=\alpha_k/2\). The corresponding guarantee for general
\(0<\mu<\alpha_k\) is given in Appendix~\ref{sec:app-general-mu}.



\subsection{Main Landscape Results}
We now state the main landscape guarantee for the balanced LoRA objective. For notational convenience, we first define an augmented variable \(X:=[\begin{smallmatrix} U\\ V \end{smallmatrix}]\in\mathbb{R}^{(n_1+n_2)\times r}\), which contains both our LoRA factors. Writing \(
J:=\operatorname{Diag}(I_{n_1},-I_{n_2})
\), we define the augmented objective as
\begin{equation}
\label{eq:main-stacked-balanced-objective}
g_{\mu,\lambda}(X)
:=
4h_{\mu,\lambda}(U,V)
=
4\widehat L(UV^\top)
+
2\lambda\|X\|_F^2
+
\mu\|X^\top JX\|_F^2 .
\end{equation}
Accordingly, we also define the augmented global solution as \(Z:=[\begin{smallmatrix} U_\star\\ V_\star \end{smallmatrix}]\in\mathbb{R}^{(n_1+n_2)\times r_\star}\) and let the lifted matrix error be \(E:=XX^\top-ZZ^\top\). Since the \((1,2)\)-block of \(XX^\top\) is \(UV^\top\), exact recovery follows once we prove that
no second-order stationary point (SOSP) can have \(E\neq0\).  The proof follows the negative-certificate
strategy of \citet{zhang2019sharp, zhangSharpGlobalGuarantees2025}. Assuming \(E\neq0\), we build a
nonnegative combination of first- and second-order stationarity tests that must be nonnegative at an SOSP, but is in fact strictly negative. By contradiction, it follows that \(E = 0\), certifying global optimality. The exact certificate we construct is
\[
\psi_{\lambda,\widehat L}(t_1,t_2)
:=
t_1\,\bigl\langle \nabla g_{\mu,\lambda}(X),D_0\bigr\rangle
+
\frac{t_2}{2}\sum_{i=1}^{r_\star}
\nabla^2 g_{\mu,\lambda}(X)[D_i,D_i],
\qquad t_1,t_2\ge0 .
\]
At an SOSP, the gradient term vanishes and every Hessian quadratic form is nonnegative, so
\(\psi_{\lambda,\widehat L}(t_1,t_2)\ge0\) for all \(t_1,t_2\ge0\), while our main result shows that this is
impossible when \(E\neq0\).  The construction decomposes the lifted error into a tangent part
visible to the current factor span and a normal part missed by it:
\[
E=E_T+E_N,
\qquad
E_N=-(I-XX^\dagger)ZZ^\top(I-XX^\dagger).
\]
The direction \(D_0\) is chosen so that
$
D_0X^\top+XD_0^\top=-E_T,
$
and therefore its gradient pairing tests the tangent error (see
\eqref{eq:app-D0-identity}).  The directions \(D_i\) are built from
\(Z_\perp=(I-XX^\dagger)Z\) and satisfy
$
\sum_{i=1}^{r_\star}D_iD_i^\top=-E_N,
$
so their Hessian terms probe the missing normal directions (see \eqref{eq:app-Di-identity}).  The
weights \(t_1,t_2\) then choose the right mixture of these two probes.  If the error is mostly
tangent, \(D_0\) already gives a negative certificate; if it is mostly normal, the directions \(D_i\) expose
negative curvature; and in the mixed case the proof uses an interpolating certificate.  The
geometric heart of the argument is Lemma~\ref{lem:app-admissible-geometry}, derived from the scalar
clipping and blockwise Eckart--Young bounds in Lemmas~\ref{lem:app-scalar-clipping} and
\ref{lem:app-blockwise-ey}.  It shows that the extra \(r-r_\star\) LoRA directions are sufficient to
make some certificate strictly negative whenever
$
(\beta_k-\alpha_k)/\beta_k
<
1/(1+\sqrt{r_\star/r}) .
$
All directions have rank at most \(k\), so the RSC/RSS bounds in
Assumption~\ref{asm:main-rsc-rss} apply. Hence \(E\neq0\) contradicts second-order stationarity. Therefore, every SOSP must satisfy
\(XX^\top=ZZ^\top\). With that, we're ready to introduce our main landscape result:
\begin{theorem}[Global optimality of SOSPs for balanced regularized LoRA]
\label{thm:main-exact-sosp}
Suppose Assumption~\ref{asm:main-regularized-target} holds
with \(1\le r_\star\le r\). Fix
\(k\ge r+r_\star\), \(0<\alpha_k<\beta_k<\infty\) and
\(\mu=\alpha_k/2\).
Every SOSP \((U,V)\) of \(h_{\mu,\lambda}\) whose product
\(M=UV^\top\) satisfies Assumption~\ref{asm:main-rsc-rss}
with these constants is globally optimal and satisfies
\(UV^\top=M_\star^\lambda\) whenever
\begin{equation}
\label{eq:main-rank-threshold}
\frac{r}{r_\star}
    >
    \left(
        \frac{2\delta_k^{\mathrm{LoRA}}}
             {1-\delta_k^{\mathrm{LoRA}}}
    \right)^2 .
\end{equation}

Specifically, if
\(X:=[\begin{smallmatrix} U\\ V \end{smallmatrix}]\in\mathbb{R}^{(n_1+n_2)\times r}\) and
\(Z:=[\begin{smallmatrix} U_\star\\ V_\star \end{smallmatrix}]
\in\mathbb{R}^{(n_1+n_2)\times r_\star}\), we have \(XX^\top = ZZ^\top\).
\end{theorem}

In particular, if the same curvature bounds hold for all rank-at-most-\(r\) candidate updates, the rank condition rules out spurious local minima throughout the factorized landscape. The proof is given in Appendix \ref{sec:app-proof-theorem-exact-sosp}. 
The guarantee transfers to the standard LoRA objective because standard stationary points are already balanced.

\begin{proposition}[Stationarity of standard LoRA implies balance]
\label{prop:main-decay-only-balance}
Let \((U,V)\) be a first-order stationary point of the standard LoRA objective
in ~\eqref{eq:main-decay-only-objective}. Then
\(
    U^\top U = V^\top V .
\)
Consequently, if \((U,V)\) is a second-order stationary point of
~\eqref{eq:main-decay-only-objective}, then it is also a second-order
stationary point of the balanced objective in
~\eqref{eq:main-balanced-lora-objective}.
\end{proposition}

Combining Theorem~\ref{thm:main-exact-sosp} with
Proposition~\ref{prop:main-decay-only-balance} gives the following transfer.

\begin{corollary}[Global optimality of SOSPs for standard LoRA]
\label{cor:main-decay-only-transfer}
Under Assumption~\ref{asm:main-regularized-target} and the rank condition of Theorem~\ref{thm:main-exact-sosp}, every SOSP of the standard LoRA
objective \eqref{eq:main-decay-only-objective} whose product satisfies Assumption~\ref{asm:main-rsc-rss} with the same constants
is globally optimal and satisfies \(UV^\top = M_\star^\lambda\).
\end{corollary}

The proofs of
Proposition~\ref{prop:main-decay-only-balance} and
Corollary~\ref{cor:main-decay-only-transfer} are given in
Appendix~\ref{sec:app-proof-standard-lora-transfer}. The rank condition reveals a continuous tradeoff between
parameterization and restricted geometry, rather than only a
single sufficient threshold. For a fixed LoRA-RIP level, it
quantifies how much rank over-parameterization is sufficient to
exclude spurious local minima. Conversely, for a fixed adapter
rank, the condition can be written as
\(
\delta_k^{\mathrm{LoRA}}
<
\frac{\sqrt{r/r_\star}}
     {2+\sqrt{r/r_\star}},
\)
showing that additional rank enlarges the range of restricted
geometries covered by the guarantee. This dual view also makes
the role of data geometry explicit: improving restricted
conditioning reduces the amount of parameterization needed for
the same global-optimality guarantee. Thus rank and data-induced
geometry act as complementary resources in shaping the LoRA
optimization landscape.

\section{Towards Improved Data Selection via Landscape Analysis}
\label{sec:data-dependent-curvature-lora}

\subsection{Decomposition of the LoRA-RIP metric}
\label{subsec:data-lora-rip-metric}
We next explain how the LoRA-RIP condition depends on the choice of training data.  Let \(\nu\)
denote the empirical distribution over the supervised prediction instances used for fine-tuning.
Changing the dataset, reweighting groups, or duplicating examples can be represented as changing \(\nu\), and therefore
changes the empirical cross-entropy objective
\begin{equation}
    \widehat L_\nu(M)
:=
\mathbb E_{x\sim \nu}
\left[
-\left\langle q_x,\log \operatorname{softmax}(z_x(M))\right\rangle
\right].
\end{equation}
Here \(z_x(M)\in\mathbb R^K\) is the pre-softmax logit vector produced by the model with LoRA update
matrix \(M\), \(q_x\) is the target distribution, and
\(p_x(M)=\operatorname{softmax}(z_x(M))\).  This formulation is important because the curvature
constants in the landscape theorem are exactly curvature constants of \(\widehat L_\nu\) along
low-rank LoRA directions.  Thus the LoRA-RIP constant is not only a property of the architecture; it
also depends on which examples are averaged in \(\widehat L_\nu\).  We write
\(M_{\nu,\star}^{\lambda}\) for the corresponding regularized target in
Assumption~\ref{asm:main-regularized-target}, and
\(r_{\nu,\star}:=\operatorname{rank}(M_{\nu,\star}^{\lambda})\). For a LoRA direction \(H\), define the first-order logit response
$
A_{x,M}[H]:=Dz_x(M)[H].
$
Since cross-entropy is invariant to adding the same constant to every logit, only the centered
response \(P_\perp A_{x,M}[H]\), with
\(P_\perp=I-\mathbf 1\mathbf 1^\top /K\), contributes meaningful curvature.  We measure how well
the data observe the direction \(H\) by
$
Q_\nu(M,H):=\mathbb E_{x\sim\nu}\|P_\perp A_{x,M}[H]\|_2^2 .
$
Large \(Q_\nu(M,H)\) means that the selected data are sensitive to this LoRA direction; small
\(Q_\nu(M,H)\) means that the direction is nearly invisible.  Good data should therefore give a
large coverage floor over all relevant low-rank directions, without placing excessive
curvature on only a few dominant directions. For cross-entropy, the Hessian decomposes into a Gauss--Newton term and a residual term:
\begin{equation}
\label{eq:main-ce-hessian-decomp}
\left\langle
\nabla^2\widehat L_\nu(M)[H],H
\right\rangle
=
G_\nu(M,H)+R_\nu(M,H).
\end{equation}
To define these terms, let
\[
\mathcal B_{x,M}[H,H]:=D^2z_x(M)[H,H],
\qquad
\Omega_x(M)
:=
\operatorname{Diag}(p_x(M))-p_x(M)p_x(M)^\top .
\]
Here \(\mathcal B_{x,M}[H,H]\) is the second-order logit response, and
\(\Omega_x(M)\) is the softmax covariance.  Then
\[
\begin{aligned}
G_\nu(M,H)
&:=\mathbb E_{x\sim\nu}
\langle A_{x,M}[H],\Omega_x(M)A_{x,M}[H]\rangle,\\
R_\nu(M,H)
&:=\mathbb E_{x\sim\nu}
\langle p_x(M)-q_x,\mathcal B_{x,M}[H,H]\rangle.
\end{aligned}
\]
The term \(G_\nu\) captures curvature created by first-order logit changes, weighted by the local
softmax covariance.  The term \(R_\nu\) is the residual Hessian term caused by the nonlinearity of
the logit map \(M\mapsto z_x(M)\).  In particular, if the logits were affine in \(M\), then
\(\mathcal B_{x,M}[H,H]=0\), and hence \(R_\nu(M,H)=0\). Highly concentrated softmax predictions limit the softmax-side weighting of first-order responses,
motivating the joint consideration of saturation and directional coverage (see
Appendix~\ref{sec:app-data-ce-hessian} and
Appendix~\ref{sec:app-data-softmax-saturation}).

For fixed \(\nu\), \(r\), over the admissible family \(\mathcal T_k(\nu,r)\), defined precisely in
Appendix~\ref{sec:app-section5-proofs}, we summarize the data-side curvature by five quantities.
The constants \(c_{\nu,k}\) and \(C_{\nu,k}\) are the minimum and maximum of the coverage
\(Q_\nu(M,H)\).  Thus \(c_{\nu,k}\) is the weakest observed LoRA-relevant direction, while
\(C_{\nu,k}\) measures the strongest observed direction.  The constants
\(\omega_{\nu,k}^-\) and \(\omega_{\nu,k}^+\) measure the lower and upper softmax-weighting strength
on these observed logit responses, and \(\rho_{\nu,k}\) bounds the residual term \(R_\nu\).  The
data-side restricted curvature condition is
\[
c_{\nu,k}>0,
\qquad
0<\omega_{\nu,k}^-\le \omega_{\nu,k}^+<\infty,
\qquad
\rho_{\nu,k}<\omega_{\nu,k}^-c_{\nu,k}.
\]
This is a restricted observability condition, not a global strong-convexity assumption for
cross-entropy.  It only requires the selected data to see the low-rank LoRA directions used by the
theorem, and it requires the residual Hessian term to stay below the Gauss--Newton curvature floor. Under this condition, the cross-entropy loss has valid restricted curvature constants
\[
\alpha_{\nu,k}^{\mathrm{CE}}
=
\omega_{\nu,k}^-c_{\nu,k}-\rho_{\nu,k},
\qquad
\beta_{\nu,k}^{\mathrm{CE}}
=
\omega_{\nu,k}^+C_{\nu,k}+\rho_{\nu,k}.
\]
Substituting these constants into the LoRA-RIP normalization gives the conservative certificate
\begin{equation}
\label{eq:main-ce-lora-rip-certificate}
\delta_{\nu,k}^{\mathrm{LoRA}}
\le
\overline\delta_{\nu,k}^{\mathrm{LoRA}}
:=
\frac{
\omega_{\nu,k}^+C_{\nu,k}
-
\omega_{\nu,k}^-c_{\nu,k}
+
2\rho_{\nu,k}
}{
\omega_{\nu,k}^+C_{\nu,k}
+
\omega_{\nu,k}^-c_{\nu,k}
}.
\end{equation}
Thus data selection can improve the LoRA-RIP certificate by increasing the weakest observed
LoRA-relevant directions, avoiding excessive concentration on a few dominant directions, and
reducing the residual nonlinear curvature.  Duplicate or imbalanced data can hurt because they
reinforce already-visible directions instead of improving the missing-direction floor, formalized in Appendix~\ref{sec:app-data-descriptor-identities}.

\subsection{Data Selection Criteria Based on LoRA-RIP}
\label{subsec:main-data-selection-curvature-design}

The certificate above turns data selection into the problem of improving the data-side LoRA-RIP
certificate, which leads to several selection rules we discuss below:

\paragraph{Latent-space coverage.}
Here latent-space coverage refers to whether hidden-state directions are both sensitive to data change and visible through the downstream Jacobian in the projected logit space.
Coverage primarily affects \(c_{\nu,k}\) and \(C_{\nu,k}\). For a rank-one direction
\(H=uv^\top\), with \(u\in\mathbb R^{n_1}\) and \(v\in\mathbb R^{n_2}\),
\(
A_{x,M}[uv^\top]
=
J_{x,M}u\langle v,a_x(M)\rangle ,
\) where \(a_x(M)\in\mathbb R^{n_2}\) is the hidden state entering this matrix, and
\(J_{x,M}\in\mathbb R^{K\times n_1}\) is the Jacobian from the adapted-layer output to the logits.
Therefore,

\begin{equation}
\label{eq:main-rank-one-projected-coverage}
Q_\nu(M,uv^\top)
=
\mathbb E_\nu
\left[
\langle v,a_x(M)\rangle^2
\,
\|P_\perp J_{x,M}u\|_2^2
\right].
\end{equation}

This formula shows that coverage has two sides: the selected examples must excite the input-side
hidden direction \(v\), and the downstream Jacobian must make the output-side direction \(u\)
visible in the shift-orthogonal logit space. If a direction is weakly represented, then
\(Q_\nu(M,H)\) can be small and \(c_{\nu,k}\) decreases. If examples are repeatedly drawn from the
same mode, \(C_{\nu,k}\) can increase without improving \(c_{\nu,k}\), which worsens the
LoRA-RIP certificate.

This perspective can be formalized through group-balanced selection. Suppose
\(
\nu_\pi=\sum_{g=1}^G\pi_g\nu_g
\)
is a mixture over groups formed by hidden-state, gradient, or logit-response similarity, where \(g\in\{1,\ldots,G\}\) indexes groups of examples, \(\nu_g\) is the empirical distribution
within group \(g\), and \(\pi_g\ge0\) is the sampling mass assigned to that group with
\(\sum_{g=1}^G\pi_g=1\). Then, for
fixed \(M,H\),
\[
Q_{\nu_\pi}(M,H)
=
\sum\nolimits_{g=1}^G
\pi_g Q_{\nu_g}(M,H).
\]

Appendix~\ref{sec:app-data-descriptor-identities} formalizes how such group masses affect the
coverage floor and ceiling. Although the infimum and supremum defining \(c_{\nu,k}\) and \(C_{\nu,k}\) do not decompose
linearly, this identity shows how group weights move curvature mass across directions. Increasing
the weight of groups that cover weak directions can raise \(c_{\nu,k}\); oversampling redundant
groups can mainly raise \(C_{\nu,k}\). This motivates allocating data across complementary response
groups while avoiding severe undercoverage. We use approximately
balanced group quotas as a practical implementation of this
principle in our experiments. The clustering criteria shall be proxies for LoRA-relevant low-rank directions, such as hidden states, gradients with respect to the adapted matrix, or
projected logit-response sketches. In Section~\ref{sec:experiments}, we use the gradient sketches as the practical proxy. Appendix~\ref{sec:app-data-gradient-sketch-proxy} explains why
such sketches are aligned with the Gauss--Newton coverage operator.

\paragraph{Softmax saturation.}
Saturation primarily affects \(\omega_{\nu,k}^-\) and \(\omega_{\nu,k}^+\). The relevant object is
the softmax covariance \(\Omega_x(M)\), not the loss value alone.
For any logit response \(s=A_{x,M}[H]\), the Gauss--Newton contribution is
\(\langle s,\Omega_x(M)s\rangle\), which is upper bounded by
\(\lambda_{\max}(\Omega_x(M))\|s\|_2^2\). Therefore, when \(\Omega_x(M)\) is small, the same
covered direction produces only weak softmax-weighted curvature.
If
\(p_{\max}=1-\epsilon\), then
\[
\lambda_{\max}(\Omega_x(M))
\le
\operatorname{tr}(\Omega_x(M))
=
1-\|p_x(M)\|_2^2
\le
2\epsilon-\epsilon^2 .
\]
Appendix~\ref{sec:app-data-softmax-saturation} gives the variance identity behind
\(\Omega_x(M)\) and the corresponding curvature-leverage bound. Thus highly confident examples contribute little softmax-weighted curvature even if their LoRA-relevant
directions are covered; they can reduce \(\omega_{\nu,k}^-\). On the other hand, extremely
ambiguous examples are not automatically optimal: their curvature may concentrate on already-covered
directions, increasing \(\omega_{\nu,k}^+C_{\nu,k}\), or they may correspond to noisy examples that
increase the residual gap. These observations motivate confidence-aware selection alongside
directional coverage, favoring examples that provide nontrivial softmax curvature without excessively concentrating it on a few directions.




\section{Experiments}
\label{sec:experiments}

We run some numerical experiments to validate our qualitative view on data selection and landscape results. Relevant code can be found in the supplementary materials.

\paragraph{Experimental setup.}
We consider one language task and one vision task. For NLP, we fine-tune MPNet-base
~\citep{songMPNetMaskedPermuted2020} on MNLI~\citep{williams-etal-2018-broad}, a natural-language
inference task in GLUE~\citep{wang-etal-2018-glue}. For vision, we fine-tune a ViT-B/16 model
~\citep{dosovitskiyImageWorth16x162020} on CIFAR-100~\citep{krizhevskyLearningMultipleLayers}. In
each run, we apply LoRA to the query projection in layer/block 10. Each experiment controls all other
training choices except for the single factor being tested. Full protocols and complete results for the main-text experiments
are provided in Appendix~\ref{sec:app-experiments}.
We further offer
curvature diagnostics, multi-seed evaluations, and robustness
studies across datasets and adaptation settings in Appendix~\ref{sec:app-supplementary-experiments}.

\paragraph{Rank sweep.}
The rank sweep isolates the role of the search rank while keeping
the data distribution fixed. Because our theory concerns the
cross-entropy objective rather than classification accuracy alone,
we report validation CE as the primary metric.
Figure~\ref{fig:exp-rank-sweep-training} shows an extended sweep
over \(r\in\{1,4,16,32,64,128,256,512\}\).
Increasing rank leads to clear reductions in validation CE on MNLI,
while the gains on CIFAR-100 are considerably smaller, although
higher ranks still provide modest improvements overall.
These trends are qualitatively consistent with the optimization role of rank highlighted by our rank--LoRA-RIP analysis.
The landscape diagnostics in
Appendix~\ref{sec:app-landscape-diagnostics} provide complementary
evidence: higher-rank runs reach lower-loss states with more favorable
local second-order geometry in the shared diagnostic subspace.
Together, these results connect the empirical effect of rank to the
optimization perspective developed in our theory.

\begin{figure}[!htbp]
    \centering
    \includegraphics[width=\linewidth]{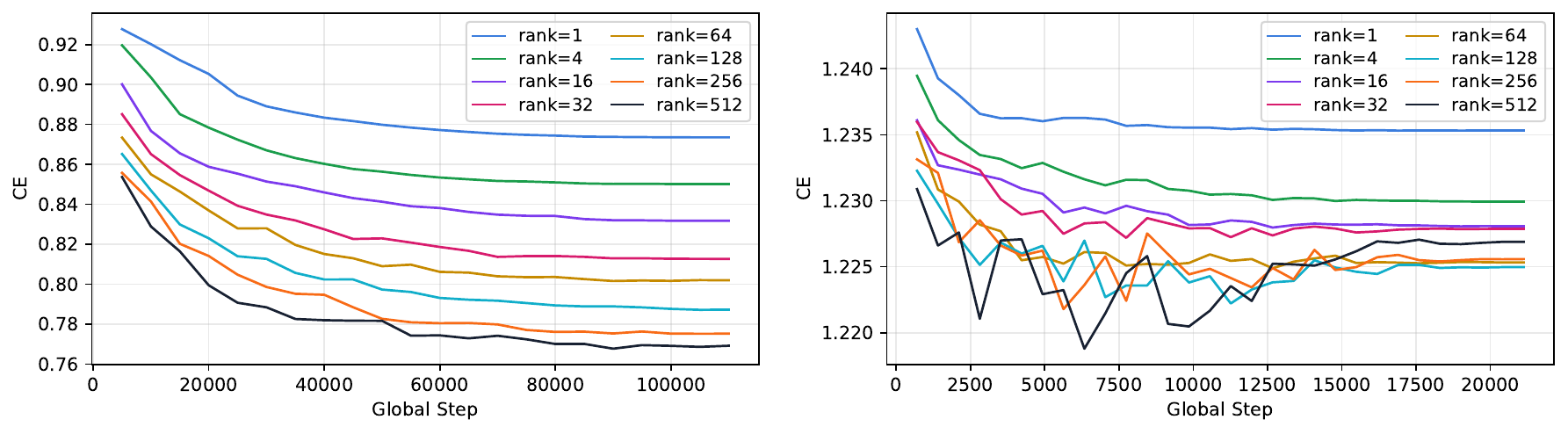}
    \caption{Rank-sweep validation CE. Left: MNLI with MPNet. Right: CIFAR-100 with ViT-B/16.}
    \label{fig:exp-rank-sweep-training}
\end{figure}
\begin{table}[!htbp]
\centering
\small
\caption{Representative rank-sweep summary at the best validation checkpoint.}
\label{tab:exp-rank-heldout}
{\setlength{\tabcolsep}{4.5pt}
\begin{tabular}{rrrr@{\qquad}rrrr}
\toprule
Rank & \multicolumn{3}{c}{MNLI} & Rank & \multicolumn{3}{c}{CIFAR-100} \\
\cmidrule(lr){2-4}\cmidrule(lr){6-8}
& Val acc& Test acc& Val CE & & Val acc& Test acc& Val CE \\
\midrule
4   & 0.6199 & 0.6167 & 0.8501 & 4   & 0.7392 & 0.7267 & 1.2299 \\
16  & 0.6360 & 0.6324 & 0.8318 & 16  & 0.7400 & 0.7251 & 1.2291 \\
64  & 0.6547 & 0.6506 & 0.8018 & 64  & 0.7414 & 0.7279 & 1.2254 \\
256 & 0.6696 & 0.6634 & 0.7752 & 256 & 0.7422 & 0.7318 & 1.2244 \\
\bottomrule
\end{tabular}
}
\end{table}

\paragraph{Gradient-coverage intervention.}
The second experiment tests the data-selection prediction from
Section~\ref{subsec:main-data-selection-curvature-design} on MNLI. We keep the subset size and label
counts fixed, and change only how examples are distributed across gradient groups. These groups are
built from a random-projected margin-Jacobian feature at the adapted matrix, so they act as a
practical proxy for which LoRA directions the data cover; see
Appendix~\ref{sec:app-data-gradient-sketch-proxy}. We compare balanced coverage with light and
heavy long-tail coverage. Figure~\ref{fig:exp-gradient-coverage-steps} shows that balanced coverage performs best overall.
Long-tail sampling hurts validation accuracy, especially under heavy long-tail sampling and at
smaller ranks. This matches the LoRA-RIP view: if too many examples come from the same gradient
groups, the data repeatedly cover the same directions while leaving other useful directions weakly
observed. In the notation of Section~\ref{subsec:main-data-selection-curvature-design}, this worsens
the balance between the coverage floor \(c_{\nu,k}\) and ceiling \(C_{\nu,k}\).

The gap shrinks as rank increases. By \(r=64\), balanced and light long-tail sampling have similar
validation accuracy, although balanced sampling still gives better test accuracy. The remaining gains at larger ranks underscore the complementary roles of adapter rank and data composition,
consistent also with our landscape results. Additional multi-seed robustness and curvature diagnostics are provided in
Appendices~\ref{sec:app-multiinit},
\ref{sec:app-data-diagnostics}, and
\ref{sec:app-trajectory-diagnostics}.

\begin{figure}[!htbp]
    \centering
    \includegraphics[width=\linewidth]{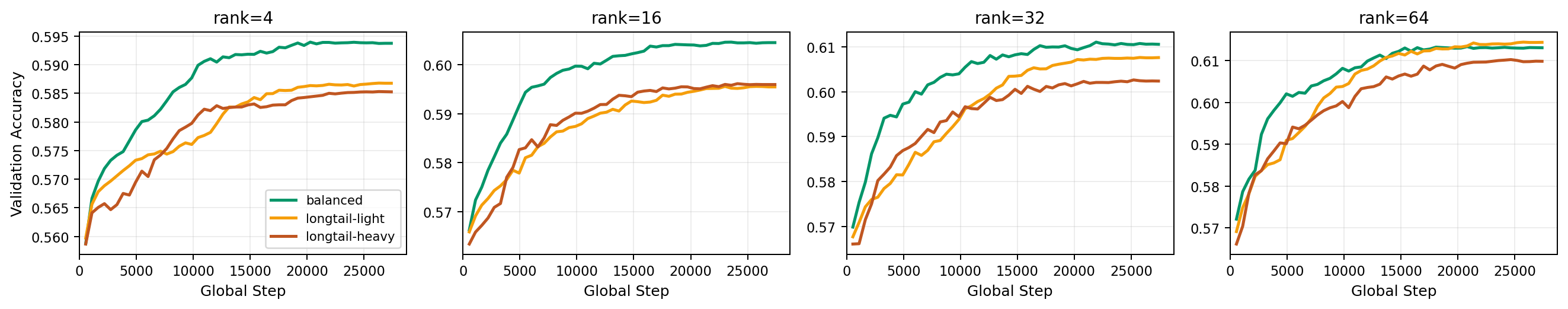}
    \caption{Validation-accuracy trajectories for the MNLI gradient-coverage intervention.}
    \label{fig:exp-gradient-coverage-steps}
\end{figure}
\paragraph{Saturation-band intervention.}
The third experiment evaluates confidence-aware selection while keeping
subset size, label counts, and gradient-group quotas fixed. Within each
group, we select easy, middle, or near-boundary examples using the shared
reference model; Appendix~\ref{sec:app-saturation-band} describes the
construction. Figure~\ref{fig:exp-saturation-band-steps} shows that both
Middle and Boundary outperform Easy in the controlled MNLI comparison.
The multi-seed results in Appendix~\ref{sec:app-saturation-diagnostics}
also show higher mean validation accuracy for both less-saturated bands
than for Easy at every tested rank, while their relative ordering varies
with rank. These performance results are qualitatively consistent with
the limited softmax-side curvature leverage of highly saturated examples
in Section~\ref{subsec:main-data-selection-curvature-design}.
The practical implication is to use confidence alongside directional
coverage and avoid concentrating the data budget on highly saturated
examples.

\begin{figure}[!htbp]
    \centering
    \includegraphics[width=\linewidth]{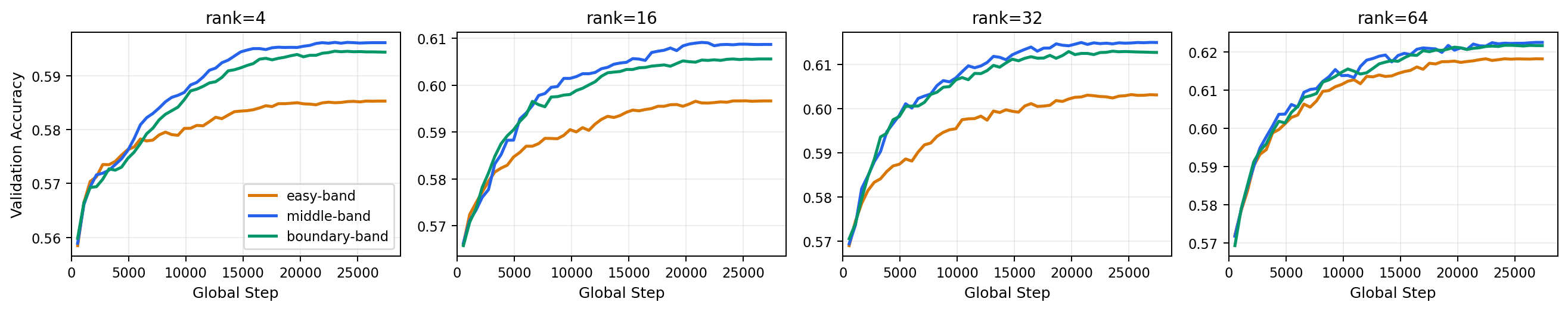}
    \caption{Validation-accuracy trajectories for the MNLI saturation-band intervention.}
    \label{fig:exp-saturation-band-steps}
\end{figure}
\FloatBarrier
\section{Conclusion}

We studied how adapter rank and fine-tuning data jointly shape
LoRA's optimization landscape under cross-entropy loss.
Under the stated assumptions, rank over-parameterization
relative to the target rank rules out spurious local minima,
with the required factor explicitly determined by our
data-dependent LoRA-RIP constant $\delta_k^{\mathrm{LoRA}}$. The guarantee also covers
LoRA-RIP values above $1/3$, quantifying how additional rank
can compensate for less favorable geometry to achieve a benign landscape.
Our framework further motivates data selection under a fixed
rank budget, yielding principles based on latent-space
coverage and pre-tuning softmax saturation.
Experiments support these insights into the interplay
between rank and data composition.
However, most of the bounds presented in this work are sufficient, and an important future direction is to shorten their gap with necessary ones.

\bibliographystyle{plainnat}
\bibliography{references,references_diagnostics}

\clearpage
\appendix
\setcounter{figure}{0}
\setcounter{table}{0}
\renewcommand{\thefigure}{\thesection.\arabic{figure}}
\renewcommand{\thetable}{\thesection.\arabic{table}}

\section{Appendix Overview}

The appendix contains the detailed proofs for the matrix-space reformulation and the
rank--LoRA-RIP theorem in Sections~3--5:
\begin{itemize}[leftmargin=1.5em]
    \item Appendix~\ref{sec:app-sosp-definitions} gives the full definition of first- and second-order stationary point.
    \item Appendix~\ref{sec:app-proof-prop-factor-decay} proves that $\ell_2$ regularization induces a nuclear-norm target in Assumption \ref{asm:main-regularized-target}.
    \item Appendix~\ref{sec:app-proof-theorem-exact-sosp} gives the detailed proof of Theorem~\ref{thm:main-exact-sosp}.
    \item Appendix~\ref{sec:app-envelope-formulation} records the admissible-class / universal-envelope formulation behind the certificate argument.
    \item Appendix~\ref{sec:app-proof-standard-lora-transfer} proves Proposition~\ref{prop:main-decay-only-balance} and Corollary~\ref{cor:main-decay-only-transfer}.
    \item Appendix~\ref{sec:app-section5-proofs} derives the cross-entropy Hessian identity, the data-side LoRA-RIP certificate, the coverage and saturation facts behind Section~\ref{sec:data-dependent-curvature-lora}.
    \item Appendix~\ref{sec:app-experiments} gives the experimental protocols, compute resources,
    full result tables, and additional plots for Section~\ref{sec:experiments}.
    \item Appendix~\ref{sec:app-supplementary-experiments} presents landscape diagnostics,
    multi-seed stability, and cross-backbone, cross-dataset, and adapter-location evaluations.
    \item Appendix~\ref{app-limit} gives the limitations.
\end{itemize}

\input{appendix_proofs}
\input{appendix_experiments}


\end{document}

%% file: appendix_proofs.tex
\section{Detailed Proofs for Sections 3 and 4}
\subsection{Stationarity notions}
\label{sec:app-sosp-definitions}

Let
\(
g:\mathbb{R}^{m\times n}\to\mathbb{R}
\)
be twice-continuously differentiable. We say that
\(
Y\in\mathbb{R}^{m\times n}
\)
is a first-order stationary point if
\(
\nabla g(Y)=0.
\)
We say that
\(
Y\in\mathbb{R}^{m\times n}
\)
is a second-order stationary point (SOSP) if
\(
\nabla g(Y)=0
\)
and
\(
\nabla^2 g(Y)[H,H]\ge 0
\)
for any direction
\(
H\in\mathbb{R}^{m\times n}.
\)
We say that \(Y\) is a local minimum if there exists an open ball \(B\) containing \(Y\) such that
\(
g(Y)\le g(Y')
\)
for every \(Y'\in B\). \(Y\) is a global minimum if
\(
g(Y)\le g(Y')
\)
for every \(Y'\in\mathbb{R}^{m\times n}\). A spurious local minimum is a local minimum that is
not global. A standard
result shows that gradient descent with random initialization avoids strict saddles almost surely;
when the iterates converge, this makes SOSPs the relevant local optimality objects
~\citep{leeGradientDescentOnly2016}. In particular, every local minimum is an SOSP.

\subsection{Proof that $l_2$ regularization induces a nuclear-norm target}
\label{sec:app-proof-prop-factor-decay}

\begin{proof}
Fix a matrix \(M\in\mathbb R^{n_1\times n_2}\) with \(\operatorname{rank}(M)\le r\). For every
factorization \(UV^\top=M\), the balancing term is nonnegative, and the standard variational
characterization of the nuclear norm gives
\[
\|M\|_*
\le
\frac12\bigl(\|U\|_F^2+\|V\|_F^2\bigr).
\]
Therefore
\[
\frac{\mu}{4}\,\|U^\top U - V^\top V\|_F^2
+
\frac{\lambda}{2}\,\bigl(\|U\|_F^2+\|V\|_F^2\bigr)
\ge
\lambda\|M\|_*.
\]
Taking the infimum over all pairs \((U,V)\) with \(UV^\top=M\) yields the lower bound
\[
\inf_{U,V:\,UV^\top=M}
\left[
\frac{\mu}{4}\,\|U^\top U - V^\top V\|_F^2
+
\frac{\lambda}{2}\,\bigl(\|U\|_F^2+\|V\|_F^2\bigr)
\right]
\ge
\lambda\|M\|_*.
\]

To prove the reverse inequality, let
\[
M=P\Sigma Q^\top
\]
be a compact SVD, where \(\Sigma\in\mathbb R^{s\times s}\) and \(s=\operatorname{rank}(M)\le r\).
Define the square-root lift
\[
U:=[P\Sigma^{1/2},\,0_{n_1\times(r-s)}],
\qquad
V:=[Q\Sigma^{1/2},\,0_{n_2\times(r-s)}].
\]
Then \(UV^\top=M\), and
\[
U^\top U = V^\top V = \operatorname{diag}(\Sigma,0_{r-s}),
\]
so the balancing penalty vanishes. Moreover,
\[
\frac{\lambda}{2}\bigl(\|U\|_F^2+\|V\|_F^2\bigr)
=
\frac{\lambda}{2}\bigl(\operatorname{tr}\Sigma+\operatorname{tr}\Sigma\bigr)
=
\lambda\,\operatorname{tr}(\Sigma)
=
\lambda\|M\|_*.
\]
Hence the infimum over factorizations with fixed product \(M\) is \(\lambda\|M\|_*\).

Now minimize first over all factor pairs with fixed product \(M\), and then over all matrices with
\(\operatorname{rank}(M)\le r\). This gives
\[
\inf_{U,V} h_{\mu,\lambda}(U,V)
=
\inf_{\operatorname{rank}(M)\le r}
\bigl[\widehat L(M)+\lambda\|M\|_*\bigr]
=
\inf_{\operatorname{rank}(M)\le r} F_\lambda(M),
\]
which is the matrix-space equivalence used in the main text.
\end{proof}

\subsection{Detailed Proof of Theorem~\ref{thm:main-exact-sosp}}
\label{sec:app-proof-theorem-exact-sosp}

The proof is organized around one lifted error matrix.  We compare the candidate factors with the
balanced lift of the regularized matrix target, and then show that any nonzero lifted error creates
a descent-or-negative-curvature certificate.  This is why the proof is carried out in the stacked
factor variable \(X\), rather than directly in the product \(UV^\top\): the SOSP condition is a
factor-space condition, while the curvature assumption is a matrix-space condition.  The stacked
notation lets us pass between the two.

We retain a general balancing weight \(0<\mu<\alpha_k\) throughout the certificate
construction and specialize to \(\mu=\alpha_k/2\) when concluding the main theorem.
The corresponding general-\(\mu\) guarantee is stated in
Appendix~\ref{sec:app-general-mu}.

Throughout this subsection, fix a candidate pair \((U,V)\), write
\[
M:=UV^\top,
\qquad
\Delta:=M-M_\star^\lambda,
\qquad
X:=\begin{bmatrix}U\\V\end{bmatrix},
\qquad
Z:=\begin{bmatrix}U_\star\\V_\star\end{bmatrix},
\qquad
E:=XX^\top-ZZ^\top,
\]
and let
\[
J:=\operatorname{diag}(I_{n_1},-I_{n_2}).
\]
Recall that \(U_\star V_\star^\top=M_\star^\lambda\) and \(U_\star^\top U_\star=V_\star^\top V_\star\),
so
\[
Z^\top J Z = 0.
\]
If \(E=0\), then the \((1,2)\)-block of \(XX^\top=ZZ^\top\) already gives
\(UV^\top=M_\star^\lambda\).  The proof therefore assumes \(E\neq0\) and derives a contradiction
with second-order stationarity.

\paragraph{KKT slack for the regularized target.}
The theorem targets the nuclear-norm-regularized matrix-space minimizer \(M_\star^\lambda\), so the
correct stationarity condition is a KKT condition rather than \(\nabla \widehat L(M_\star^\lambda)=0\).
The slack form below is the device that keeps the regularized target in the proof: it creates a
positive-semidefinite term that vanishes on the target lift \(Z\), and hence can only help when we
upper-bound the certificate.

\begin{proposition}[KKT slack form of the regularized target]
\label{prop:app-kkt-slack}
There exists a matrix
\[
S_\star=P_\star Q_\star^\top + W_\star \in \partial \|M_\star^\lambda\|_*
\]
such that
\begin{equation}
\label{eq:app-kkt-stationarity}
\nabla \widehat L(M_\star^\lambda)+\lambda S_\star = 0,
\end{equation}
where
\[
P_\star^\top W_\star = 0,
\qquad
W_\star Q_\star = 0,
\qquad
\|W_\star\|_{\mathrm{op}}\le 1.
\]
If
\[
\widetilde S_\star:=
\begin{bmatrix}
0 & S_\star\\
S_\star^\top & 0
\end{bmatrix},
\]
then
\begin{equation}
\label{eq:app-slack-psd-annihilate}
I-\widetilde S_\star \succeq 0,
\qquad
(I-\widetilde S_\star)Z=0.
\end{equation}
\end{proposition}

\begin{proof}
Since \(M_\star^\lambda\) minimizes \(F_\lambda(M)=\widehat L(M)+\lambda\|M\|_*\) and \(\widehat L\)
is differentiable, Fermat's rule gives
\[
0\in \nabla \widehat L(M_\star^\lambda)+\lambda\,\partial\|M_\star^\lambda\|_*.
\]
Hence \eqref{eq:app-kkt-stationarity} holds for some
\(S_\star\in\partial\|M_\star^\lambda\|_*\). Writing the compact SVD
\(
M_\star^\lambda=P_\star\Sigma_\star Q_\star^\top
\),
the standard subgradient formula for the nuclear norm yields
\[
S_\star=P_\star Q_\star^\top+W_\star,
\qquad
P_\star^\top W_\star=0,
\qquad
W_\star Q_\star=0,
\qquad
\|W_\star\|_{\mathrm{op}}\le 1.
\]

Because \(\|\widetilde S_\star\|_{\mathrm{op}}=\|S_\star\|_{\mathrm{op}}\le 1\), we have
\(I-\widetilde S_\star\succeq 0\). Moreover,
\[
S_\star V_\star
=
(P_\star Q_\star^\top+W_\star)Q_\star\Sigma_\star^{1/2}
=
P_\star\Sigma_\star^{1/2}
=
U_\star,
\]
because \(W_\star Q_\star=0\). Similarly,
\[
S_\star^\top U_\star = V_\star.
\]
Therefore
\[
\widetilde S_\star Z
=
\begin{bmatrix}
S_\star V_\star\\
S_\star^\top U_\star
\end{bmatrix}
=
\begin{bmatrix}
U_\star\\
V_\star
\end{bmatrix}
=
Z,
\]
which proves \eqref{eq:app-slack-psd-annihilate}.
\end{proof}

\paragraph{Lifted operators and differential identities.}
We next rewrite the first- and second-order conditions of the factorized objective in lifted
coordinates.  The averaged operator below comes from expanding the matrix-space loss along the
segment from \(M_\star^\lambda\) to \(M=UV^\top\); the pointwise operator is the Hessian that appears
in the SOSP condition at the current factor pair.  Both inherit the same lower and upper curvature
constants after lifting, up to the balancing curvature.

Define the segment-average Hessian
\[
\overline H_\Delta[H]
:=
\int_0^1 \nabla^2 \widehat L(M_\star^\lambda+t\Delta)[H]\,dt,
\]
its block lift
\[
\widetilde{\overline H}_\Delta[E]
:=
\begin{bmatrix}
0 & \overline H_\Delta[E_{12}]\\
\overline H_\Delta[E_{12}]^\top & 0
\end{bmatrix},
\]
and the composite lifted operator
\[
\mathcal H^{\mathrm{avg}}_{\mu,\Delta}[E]
:=
2\widetilde{\overline H}_\Delta[E]+2\mu J E J.
\]
Also define
\[
g_{\mu,\lambda}(X):=4\,h_{\mu,\lambda}(U,V),
\qquad
G_M:=\nabla \widehat L(M),
\qquad
\widetilde G_M:=
\begin{bmatrix}
0 & G_M\\
G_M^\top & 0
\end{bmatrix}.
\]
For a lifted direction
\(
D=\begin{bmatrix}\Delta U\\ \Delta V\end{bmatrix},
\)
set
\[
\widetilde H_M[E]
:=
\begin{bmatrix}
0 & \nabla^2 \widehat L(M)[E_{12}]\\
\nabla^2 \widehat L(M)[E_{12}]^\top & 0
\end{bmatrix},
\qquad
\mathcal H^{\mathrm{pt}}_{\mu,X}[E]
:=
2\widetilde H_M[E]+2\mu J E J.
\]
A direct differentiation gives
\begin{align}
\label{eq:app-gradient-identity}
\bigl\langle \nabla g_{\mu,\lambda}(X),D\bigr\rangle
&=
\Bigl\langle
2\widetilde G_M + 2\mu JXX^\top J + 2\lambda I,\;
DX^\top + X D^\top
\Bigr\rangle,\\
\label{eq:app-hessian-identity}
\frac12\bigl\langle \nabla^2 g_{\mu,\lambda}(X)[D],D\bigr\rangle
&=
\Bigl\langle
2\widetilde G_M + 2\mu JXX^\top J + 2\lambda I,\;
DD^\top
\Bigr\rangle
\\
&\quad
+ \frac12\Bigl\langle
\mathcal H^{\mathrm{pt}}_{\mu,X}[DX^\top+XD^\top],\;
DX^\top+XD^\top
\Bigr\rangle.
\end{align}

By the pathwise rank-\(k\) RSC/RSS assumption for the candidate product \(M\), the curvature bounds
apply along the segment \(\{M_\star^\lambda+t\Delta:t\in[0,1]\}\) and at the current point
\(M\) corresponding to \(t=1\). In particular, let
\begin{align*}
m_\mu&:=2\min\{\mu,\alpha_k-\mu\},\quad
L_\mu:=2\max\{\mu,\beta_k-\mu\},\\
\kappa_\mu&:=\frac{L_\mu}{m_\mu},\quad
s_\mu:=\frac{L_\mu+m_\mu}{2},\quad
d_\mu:=\frac{L_\mu-m_\mu}{2}.
\end{align*}
For a lifted symmetric matrix
\[
B=
\begin{bmatrix}
B_{11} & B_{12}\\
B_{12}^\top & B_{22}
\end{bmatrix}
\]
with \(\operatorname{rank}(B_{12})\le k\), the RSC/RSS assumption and the identity
\(\langle JBJ,B\rangle=\|B_{11}\|_F^2+\|B_{22}\|_F^2-2\|B_{12}\|_F^2\) give
\[
\begin{aligned}
\bigl\langle \mathcal H^{\mathrm{avg}}_{\mu,\Delta}[B],B\bigr\rangle
&=
4\int_0^1
\langle\nabla^2\widehat L(M_\star^\lambda+t\Delta)[B_{12}],B_{12}\rangle\,dt
+2\mu\langle JBJ,B\rangle.
\end{aligned}
\]
Therefore
\[
\begin{aligned}
\bigl\langle \mathcal H^{\mathrm{avg}}_{\mu,\Delta}[B],B\bigr\rangle
&\ge
2\mu(\|B_{11}\|_F^2+\|B_{22}\|_F^2)
+(4\alpha_k-4\mu)\|B_{12}\|_F^2,\\
\bigl\langle \mathcal H^{\mathrm{avg}}_{\mu,\Delta}[B],B\bigr\rangle
&\le
2\mu(\|B_{11}\|_F^2+\|B_{22}\|_F^2)
+(4\beta_k-4\mu)\|B_{12}\|_F^2.
\end{aligned}
\]
Since
\(\|B\|_F^2=\|B_{11}\|_F^2+\|B_{22}\|_F^2+2\|B_{12}\|_F^2\), this implies
\begin{equation}
\label{eq:app-lifted-curvature}
m_\mu\|E\|_F^2
\le
\bigl\langle \mathcal H^{\mathrm{avg}}_{\mu,\Delta}[E],E\bigr\rangle
\le
L_\mu\|E\|_F^2
\end{equation}
for every lifted symmetric matrix \(E\) whose off-diagonal block has rank at most \(k\). The same
calculation at the point \(M\) gives
\begin{equation}
\label{eq:app-pointwise-curvature}
m_\mu\|H\|_F^2
\le
\bigl\langle \mathcal H^{\mathrm{pt}}_{\mu,X}[H],H\bigr\rangle
\le
L_\mu\|H\|_F^2
\end{equation}
for every lifted direction \(H\) whose off-diagonal block has rank at most \(k\).

The next decomposition is the algebraic hinge of the proof.  It rewrites the first-order expression
at \(X\) as the sum of a curvature term applied to the lifted error \(E\) and a PSD slack
\(N_{\mu,\lambda}\).  The slack is harmless in the later upper bounds because it is positive
semidefinite and annihilates \(Z\).

\begin{lemma}[Exact decomposition with KKT slack]
\label{lem:app-exact-decomposition}
Define
\[
N_{\mu,\lambda}
:=
2\mu JZZ^\top J + 2\lambda(I-\widetilde S_\star).
\]
Then
\begin{equation}
\label{eq:app-exact-decomposition}
2\widetilde G_M + 2\mu JXX^\top J + 2\lambda I
=
\mathcal H^{\mathrm{avg}}_{\mu,\Delta}[E] + N_{\mu,\lambda}.
\end{equation}
Moreover,
\begin{equation}
\label{eq:app-N-psd}
N_{\mu,\lambda}\succeq 0,
\qquad
N_{\mu,\lambda}Z=0.
\end{equation}
\end{lemma}

\begin{proof}
By the fundamental theorem of calculus,
\[
G_M = \nabla \widehat L(M_\star^\lambda)+\overline H_\Delta[\Delta].
\]
Using \eqref{eq:app-kkt-stationarity}, this becomes
\[
G_M = -\lambda S_\star + \overline H_\Delta[\Delta].
\]
Since \(E_{12}=\Delta\), we have
\[
2\widetilde G_M
=
-2\lambda \widetilde S_\star + 2\widetilde{\overline H}_\Delta[E].
\]
Also,
\[
JXX^\top J = JEJ + JZZ^\top J.
\]
Substituting these identities yields
\[
2\widetilde G_M + 2\mu JXX^\top J + 2\lambda I
=
2\widetilde{\overline H}_\Delta[E] + 2\mu JEJ + 2\mu JZZ^\top J + 2\lambda(I-\widetilde S_\star),
\]
which is \eqref{eq:app-exact-decomposition}.

By \eqref{eq:app-slack-psd-annihilate}, \(I-\widetilde S_\star\succeq 0\) and
\((I-\widetilde S_\star)Z=0\). Also, \((JZ)(JZ)^\top\succeq 0\) and
\[
(JZ)(JZ)^\top Z = JZ(Z^\top J Z)=0
\]
because \(Z^\top J Z=0\). Therefore \eqref{eq:app-N-psd} follows.
\end{proof}

\subsubsection{Admissible Loss Class and Universal Envelope}
\label{sec:app-envelope-formulation}
The direct theorem proof can be carried out using only the actual empirical loss \(\widehat L\).
For the Section~\ref{sec:data-dependent-curvature-lora} interface, however, it is helpful to separate the certificate calculation from
the specific loss and retain only the matrix-space target condition and curvature bounds actually used by the argument.

For the fixed pair \((X,Z)\), define the matrix-space path
\[
M_t:=M_\star^\lambda+t\Delta,
\qquad
t\in[0,1].
\]
Relative to the fixed curvature parameters \((\alpha_k,\beta_k)\) and a rank level
\(k\ge r+r_\star\), define the admissible loss class
\begin{equation}
\label{eq:app-admissible-loss-class}
\mathcal F_{\lambda,\alpha_k,\beta_k}^{(k)}(X,Z)
:=
\left\{
\begin{aligned}
h\in C^2(\mathbb R^{n_1\times n_2};\mathbb R)\ :\ &
M_\star^\lambda\in\operatorname*{argmin}_M\bigl[h(M)+\lambda\|M\|_*\bigr],\\
&
\alpha_k\|H\|_F^2
\le
\langle \nabla^2 h(M_t)[H],H\rangle
\le
\beta_k\|H\|_F^2,\\
&
\text{for all }t\in[0,1]\text{ and all }H\text{ with }\operatorname{rank}(H)\le k
\end{aligned}
\right\}.
\end{equation}
For each \(h\in\mathcal F_{\lambda,\alpha_k,\beta_k}^{(k)}(X,Z)\), let
\[
g_{\mu,\lambda}^h(X)
:=
4\left(
h(UV^\top)
+
\frac{\mu}{4}\|U^\top U-V^\top V\|_F^2
+
\frac{\lambda}{2}\bigl(\|U\|_F^2+\|V\|_F^2\bigr)
\right)
\]
denote the lifted objective obtained by replacing \(\widehat L\) with \(h\).
The corresponding fixed-loss certificate is
\begin{equation}
\label{eq:app-certificate-h}
\psi_{\lambda,h}(t_1,t_2)
:=
t_1\,\bigl\langle \nabla g_{\mu,\lambda}^h(X),D_0\bigr\rangle
+
\frac{t_2}{2}\sum_{i=1}^{r_\star}
\bigl\langle \nabla^2 g_{\mu,\lambda}^h(X)[D_i],D_i\bigr\rangle,
\end{equation}
and the universal envelope is
\begin{equation}
\label{eq:app-envelope}
\Psi_\lambda(t_1,t_2)
:=
\sup_{h\in\mathcal F_{\lambda,\alpha_k,\beta_k}^{(k)}(X,Z)}
\psi_{\lambda,h}(t_1,t_2).
\end{equation}
Here the curvature parameters \((\alpha_k,\beta_k)\) are fixed throughout the subsection, so we
suppress them from the notation of \(\Psi_\lambda\).
Under Assumptions~\ref{asm:main-regularized-target} and~\ref{asm:main-rsc-rss}, the actual loss
\(\widehat L\) belongs to \(\mathcal F_{\lambda,\alpha_k,\beta_k}^{(k)}(X,Z)\). Hence
\[
\psi_{\lambda,\widehat L}(t_1,t_2)\le \Psi_\lambda(t_1,t_2).
\]
The detailed proof below still bounds \(\psi_{\lambda,\widehat L}\) directly. We keep the envelope
formulation because it makes explicit which ingredients of the proof are loss-specific and which
are purely geometric, and this is the separation reused later in
Section~\ref{sec:data-dependent-curvature-lora}.

\paragraph{Certificate family.}
We now construct the directions used to test second-order stationarity.  The lifted error has a part
seen by the current column space of \(X\) and a part orthogonal to it.  The direction \(D_0\) probes
the former through the first-order condition, while the directions \(D_i\) probe the latter through
the second-order condition.  The scalar weights \(t_1,t_2\) choose how much of each probe to use.

Let
\[
\Pi_X:=XX^\dagger,
\qquad
Z_\perp:=(I-\Pi_X)Z.
\]
Split the lifted error into tangent and normal components:
\[
E_T:=\Pi_XE+E\Pi_X-\Pi_X E\Pi_X,
\qquad
E_N:=(I-\Pi_X)E(I-\Pi_X)=-Z_\perp Z_\perp^\top,
\]
so that \(E=E_T+E_N\) and \(\langle E_T,E_N\rangle=0\).

Choose a unit vector \(u\in\mathbb R^r\) satisfying
\[
u^\top X^\top X u = \lambda_{\min}(X^\top X),
\]
and define
\[
D_0
:=
\left(I-\frac12XX^\dagger\right)Z(X^\dagger Z)^\top - \frac12 X,
\qquad
D_i:=Z_\perp e_i u^\top,\quad i=1,\dots,r_\star.
\]
Direct expansion gives the basic identities
\begin{equation}
\label{eq:app-D0-identity}
D_0X^\top + X D_0^\top = -E_T,
\end{equation}
\begin{equation}
\label{eq:app-Di-identity}
\sum_{i=1}^{r_\star} D_iD_i^\top = -E_N.
\end{equation}

In the nontrivial case \(X\neq 0\) and \(Z_\perp\neq 0\), define
\[
a:=\frac{\|Z_\perp Z_\perp^\top\|_F}{\|E\|_F},
\qquad
b_{\mathrm{geo}}
:=
\frac{\lambda_{\min}(X^\top X)\,\operatorname{tr}(Z_\perp Z_\perp^\top)}
{\|E\|_F\,\|Z_\perp Z_\perp^\top\|_F}.
\]
For \(\tau\in[0,a]\), let
\[
t_1(\tau):=\sqrt{\frac{1-\tau^2}{1-a^2}},
\qquad
t_2(\tau):=\frac{\tau}{a},
\qquad
F_\tau:=t_1(\tau)E_T+t_2(\tau)E_N.
\]
Then
\begin{equation}
\label{eq:app-Ftau-identities}
\|F_\tau\|_F=\|E\|_F,
\qquad
\langle E,F_\tau\rangle
=
\bigl(\sqrt{1-a^2}\sqrt{1-\tau^2}+a\tau\bigr)\|E\|_F^2.
\end{equation}
All averaged-curvature applications below are admissible under Assumption~\ref{asm:main-rsc-rss}:
the lifted matrices \(E\), \(F_\tau\), and their linear combinations have column space contained in
\(\operatorname{span}(X,Z)\), hence rank at most \(r+r_\star\), and therefore their \((1,2)\)-blocks
have rank at most \(r+r_\star\le k\). For the pointwise Hessian terms, each
\(\dot E_i:=D_iX^\top+XD_i^\top\) has rank at most two, so its \((1,2)\)-block also has rank at most
\(2\le k\).

Specializing \eqref{eq:app-certificate-h} to the actual empirical loss \(\widehat L\), define
\begin{equation}
\label{eq:app-certificate}
\psi_\lambda(t_1,t_2):=\psi_{\lambda,\widehat L}(t_1,t_2).
\end{equation}
If \((U,V)\) is an SOSP of \(h_{\mu,\lambda}\), then
\begin{equation}
\label{eq:app-certificate-nonnegative}
\psi_\lambda(t_1,t_2)\ge 0
\qquad\text{for all } t_1,t_2\ge 0.
\end{equation}

For \(\eta\ge 0\), define
\[
\phi_\eta(\tau;a,b_{\mathrm{geo}})
:=
\sqrt{1-a^2}\sqrt{1-\tau^2}
+
\bigl(a-(1+\eta)b_{\mathrm{geo}}\bigr)\tau
-
\eta.
\]

The remaining lemmas justify that this certificate family is rich enough.  The scalar clipping and
blockwise Eckart--Young bounds are purely geometric; they quantify how the extra \(r-r_\star\)
factor directions contribute to the certificate.  This is the only place where the rank threshold
enters the proof.

\begin{equation}
\label{eq:app-phi-ratio}
\phi_\eta(\tau;a,b_{\mathrm{geo}})
=
\bigl(1+b_{\mathrm{geo}}\tau\bigr)
\left[
\frac{\sqrt{1-a^2}\sqrt{1-\tau^2}+a\tau-b_{\mathrm{geo}}\tau}
{1+b_{\mathrm{geo}}\tau}
- \eta
\right].
\end{equation}

\begin{lemma}[Scalar clipping bound]
\label{lem:app-scalar-clipping}
Let \(p\) be positive, let \(u_1,\dots,u_p\ge s\ge0\), and let
\(d\in\mathbb R_+^p\) satisfy \(S:=\mathbf 1^\top d>0\) and
\(Q:=\|d\|_2^2>0\). Then
\[
\sum_{j=1}^p u_j^2-\sum_{j=1}^p (u_j-d_j)_+^2
\ge
\begin{cases}
\displaystyle s^2S^2/Q, & sS\le Q,\\[0.4em]
\displaystyle 2sS-Q, & sS>Q.
\end{cases}
\]
\end{lemma}

\begin{proof}
If \(s=0\), both claimed lower bounds reduce to zero and are immediate. Assume \(s>0\), and write
\(x_j:=d_j/s\), \(T:=\mathbf 1^\top x\), and \(R:=\|x\|_2^2\).
For fixed \(d\ge0\), the map
\(\varphi_d(t):=t^2-(t-d)_+^2\) is nondecreasing on \(t\ge0\). Hence
\[
\sum_{j=1}^p u_j^2-\sum_{j=1}^p (u_j-d_j)_+^2
\ge
s^2\left[p-\sum_{j=1}^p (1-x_j)_+^2\right].
\]

First suppose \(T\le R\). Let \(\gamma:=T/R\le1\) and \(w:=\gamma x\). Then
\(\mathbf 1^\top w=\|w\|_2^2=T^2/R\). We claim
\[
p-\|w\|_2^2\ge \|(1-w)_+\|_2^2.
\]
Indeed, using \(\mathbf 1^\top w=\|w\|_2^2\),
\[
\sum_{w_j\le1}w_j(1-w_j)=\sum_{w_j>1}w_j(w_j-1)
\ge
\sum_{w_j>1}(w_j-1),
\]
which is exactly the displayed inequality after expanding both sides. Since
\(\gamma\le1\), we also have \((1-w_j)_+\ge(1-x_j)_+\) for every \(j\). Therefore
\[
p-\sum_{j=1}^p(1-x_j)_+^2
\ge
T^2/R.
\]
Multiplying by \(s^2\) gives \(s^2S^2/Q\).

Now suppose \(T>R\). For each coordinate,
\[
1-(1-x_j)_+^2\ge 2x_j-x_j^2:
\]
the two sides are equal when \(x_j\le1\), while for \(x_j>1\) this reduces to
\((x_j-1)^2\ge0\). Summing gives
\[
p-\sum_{j=1}^p(1-x_j)_+^2\ge 2T-R.
\]
Multiplying by \(s^2\) yields \(2sS-Q\).
\end{proof}

\begin{lemma}[Blockwise regularized Eckart--Young bound]
\label{lem:app-blockwise-ey}
Let \(A\succeq0\) be an \(r\times r\) matrix with eigenvalues
\(s_1\ge\cdots\ge s_r=s\), and let \(D\succeq0\) be an
\(r_\star\times r_\star\) matrix with eigenvalues \(0\le d_1\le\cdots\le d_{r_\star}\), not all
zero. For any \(Y\in\mathbb R^{r\times r_\star}\), define
\[
\mathcal E(A,Y,D)
:=
\|A-YY^\top\|_F^2
+2\langle D,Y^\top Y\rangle
+\|D\|_F^2.
\]
Set \(S:=\operatorname{tr}D=\sum_j d_j\) and \(Q:=\|D\|_F^2=\sum_j d_j^2\). Then
\[
\mathcal E(A,Y,D)
\ge
Q+\frac{r}{r_\star}\frac{(sS)^2}{Q}
\quad\text{whenever }sS\le Q,
\]
and
\[
\mathcal E(A,Y,D)
\ge
2sS-Q+\frac{r-r_\star}{r_\star}\frac{(sS)^2}{Q}
\quad\text{whenever }sS>Q.
\]
\end{lemma}

\begin{proof}
Let \(p:=r_\star\). The objective is invariant under orthogonal changes of basis, so diagonalize
\(A\) and \(D\). Let \(z_1\ge\cdots\ge z_p\ge0\) be the eigenvalues of \(Y^\top Y\). The trace
inequalities give
\[
\langle A,YY^\top\rangle\le \sum_{j=1}^p s_jz_j,
\qquad
\langle D,Y^\top Y\rangle\ge \sum_{j=1}^p d_jz_j,
\]
where the second inequality pairs the largest \(z_j\)'s with the smallest eigenvalues of \(D\).
Thus, for every \(Y\),
\[
\mathcal E(A,Y,D)
\ge
Q+\sum_{i=1}^r s_i^2
+\sum_{j=1}^p\left[z_j^2-2(s_j-d_j)z_j\right].
\]
Minimizing the right-hand side over \(z_j\ge0\) gives
\[
\mathcal E(A,Y,D)
\ge
Q+\sum_{i=p+1}^r s_i^2
+\sum_{j=1}^p\left[s_j^2-(s_j-d_j)_+^2\right].
\]
Since \(s_i\ge s\), the inactive directions contribute
\(\sum_{i=p+1}^r s_i^2\ge(r-p)s^2\), and
Lemma~\ref{lem:app-scalar-clipping} applied to \(u_j=s_j\) yields
\[
\mathcal E(A,Y,D)
\ge
\begin{cases}
\displaystyle Q+(r-p)s^2+s^2S^2/Q, & sS\le Q,\\[0.4em]
\displaystyle 2sS+(r-p)s^2, & sS>Q.
\end{cases}
\]
Finally, Cauchy's inequality gives \(S^2\le pQ\). Hence
\((r-p)s^2\ge \frac{r-p}{p}\frac{(sS)^2}{Q}\), which proves the first regime. The same
inequality also gives
\[
2sS+(r-p)s^2
\ge
2sS-Q+\frac{r-p}{p}\frac{(sS)^2}{Q},
\]
proving the second regime.
\end{proof}

\begin{lemma}[Admissible geometry for \((a,b_{\mathrm{geo}})\)]
\label{lem:app-admissible-geometry}
In the nontrivial case \(X\neq 0\) and \(Z_\perp\neq 0\), the scalars \(a\) and
\(b_{\mathrm{geo}}\) satisfy
\begin{equation}
\label{eq:app-admissible-geometry}
a^2+\frac{r}{r_\star}b_{\mathrm{geo}}^2
\le
1+\bigl[(b_{\mathrm{geo}}-a)_+\bigr]^2.
\end{equation}
\end{lemma}

\begin{proof}
This lemma isolates the only step where the search rank \(r\) enters the SOSP recovery argument. We
reduce the geometry to a block decomposition, and then use the regularized Eckart--Young mechanism
to extract the extra \(r-r_\star\) degrees of freedom into an explicit coefficient \(r/r_\star\).

Because \(a\) and \(b_{\mathrm{geo}}\) are invariant under left orthogonal changes of coordinates,
we may assume that the lifted factors are written in block form as
\[
X=
\begin{bmatrix}
X_1\\
0
\end{bmatrix},
\qquad
Z=
\begin{bmatrix}
Z_1\\
Z_2
\end{bmatrix},
\]
where \(X_1\in\mathbb R^{r\times r}\), \(Z_1\in\mathbb R^{r\times r_\star}\), and
\(Z_2\in\mathbb R^{(n-r)\times r_\star}\) with \(n:=n_1+n_2\). Let
\[
s_i:=\lambda_i(X_1X_1^\top),\qquad i=1,\dots,r,
\]
and
\[
d_i:=\lambda_i(Z_2^\top Z_2),\qquad i=1,\dots,r_\star.
\]
Then
\[
s_r=\lambda_{\min}(X^\top X),
\qquad
\|d\|_2=\|Z_\perp Z_\perp^\top\|_F,
\qquad
1^\top d=\operatorname{tr}(Z_\perp Z_\perp^\top),
\]
so the two geometric scalars become
\begin{equation}
\label{eq:app-a-bgeo-via-sd}
a=\frac{\|d\|_2}{\|E\|_F},
\qquad
b_{\mathrm{geo}}
=
\frac{s_r\,1^\top d}{\|E\|_F\,\|d\|_2}.
\end{equation}

Now expand the lifted error blockwise:
\[
\|E\|_F^2
=
\|X_1X_1^\top-Z_1Z_1^\top\|_F^2
+
2\langle Z_2^\top Z_2,\; Z_1^\top Z_1\rangle
+
\|Z_2^\top Z_2\|_F^2.
\]
The right-hand side is exactly \(\mathcal E(A,Z_1,D)\) from
Lemma~\ref{lem:app-blockwise-ey}, with
\[
A:=X_1X_1^\top,\qquad D:=Z_2^\top Z_2,\qquad s:=s_r.
\]
Moreover, by \eqref{eq:app-a-bgeo-via-sd}, the regimes \(b_{\mathrm{geo}}\le a\) and
\(b_{\mathrm{geo}}>a\) are equivalent to \(s_r\,\mathbf 1^\top d\le\|d\|_2^2\) and
\(s_r\,\mathbf 1^\top d>\|d\|_2^2\), respectively. Applying
Lemma~\ref{lem:app-blockwise-ey} therefore yields
\begin{align}
\label{eq:app-geometry-lb1}
\|E\|_F^2
&\ge
\frac{r}{r_\star}\,
\frac{(s_r\,1^\top d)^2}{\|d\|_2^2}
+
\|d\|_2^2
\qquad\text{when } b_{\mathrm{geo}}\le a,\\
\label{eq:app-geometry-lb2}
\|E\|_F^2
&\ge
2s_r\,1^\top d-\|d\|_2^2
+
\frac{r-r_\star}{r_\star}\,
\frac{(s_r\,1^\top d)^2}{\|d\|_2^2}
\qquad\text{when } b_{\mathrm{geo}}> a.
\end{align}
The factor \(r/r_\star\) is thus a geometric consequence of the block decomposition: the
\(r-r_\star\) inactive singular directions of \(X\) contribute extra \(s_r^2\)-scale curvature,
and \((1^\top d)^2\le r_\star\|d\|_2^2\) converts this surplus into the same scale as the normal
block.

When \(b_{\mathrm{geo}}\le a\), dividing \eqref{eq:app-geometry-lb1} by \(\|E\|_F^2\) and using
\eqref{eq:app-a-bgeo-via-sd} gives
\[
1\ge a^2+\frac{r}{r_\star}b_{\mathrm{geo}}^2.
\]
When \(b_{\mathrm{geo}}>a\), dividing \eqref{eq:app-geometry-lb2} by \(\|E\|_F^2\) gives
\[
1
\ge
2ab_{\mathrm{geo}}-a^2+\frac{r-r_\star}{r_\star}b_{\mathrm{geo}}^2
=
a^2+\frac{r}{r_\star}b_{\mathrm{geo}}^2-(b_{\mathrm{geo}}-a)^2.
\]
Thus the first regime gives \eqref{eq:app-admissible-geometry} with
\((b_{\mathrm{geo}}-a)_+=0\), and the second gives
\eqref{eq:app-admissible-geometry} with \((b_{\mathrm{geo}}-a)_+=b_{\mathrm{geo}}-a\). This proves
\eqref{eq:app-admissible-geometry}.
\end{proof}

\begin{lemma}[Positivity of the one-dimensional gain]
\label{lem:app-phi-positive}
If
\[
\eta < \frac{1}{1+\sqrt{r_\star/r}},
\]
then there exists some \(\tau\in[0,a]\) such that
\[
\phi_\eta(\tau;a,b_{\mathrm{geo}})>0.
\]
\end{lemma}

\begin{proof}
Introduce the ratio
\[
\delta(a,b_{\mathrm{geo}})
:=
\sup_{0\le \tau\le a}
\frac{\sqrt{1-a^2}\sqrt{1-\tau^2}+a\tau-b_{\mathrm{geo}}\tau}
{1+b_{\mathrm{geo}}\tau}
\]
By \eqref{eq:app-phi-ratio}, finding \(\tau\in[0,a]\) with
\(\phi_\eta(\tau;a,b_{\mathrm{geo}})>0\) is equivalent to showing
\[
\delta(a,b_{\mathrm{geo}})>\eta.
\]

Once Lemma~\ref{lem:app-admissible-geometry} is in place, the proof becomes a one-dimensional
optimization over the admissible \((a,b_{\mathrm{geo}})\)-region. The required scalar maximization is
the same elementary two-region reduction used in sharp overparameterization analyses
~\citep[Lemma~5.4]{zhangImprovedGlobalGuarantees2025}; for completeness we spell out the two cases
below.

If
\[
b_{\mathrm{geo}}\ge \frac{a}{1+\sqrt{1-a^2}},
\]
then the closed-form optimizer shows that the smallest admissible value of
\(\delta(a,b_{\mathrm{geo}})\) on this branch is attained on the boundary
\[
b_{\mathrm{geo}}=\frac{a}{1+\sqrt{1-a^2}}.
\]
Thus it remains to analyze the boundary and the complementary low-\(b_{\mathrm{geo}}\) branch.

If instead
\[
b_{\mathrm{geo}}\le \frac{a}{1+\sqrt{1-a^2}},
\]
then in particular \(b_{\mathrm{geo}}\le a\), so we may write
\[
a=\gamma b_{\mathrm{geo}},
\qquad
\gamma\ge 1.
\]
Set
\[
\varrho:=\sqrt{r/r_\star}.
\]
In this region the plus-term in \eqref{eq:app-admissible-geometry} vanishes, and the admissible
constraint becomes
\[
(\gamma^2+\varrho^2)b_{\mathrm{geo}}^2\le 1.
\]
On this branch, the scalar maximization reduces to
\[
\delta(a,b_{\mathrm{geo}})
=
\frac{1-2ab_{\mathrm{geo}}+b_{\mathrm{geo}}^2}{1-b_{\mathrm{geo}}^2}
=
\frac{1-(2\gamma-1)b_{\mathrm{geo}}^2}{1-b_{\mathrm{geo}}^2}.
\]
For every fixed \(\gamma\ge 1\), this quantity is decreasing in \(b_{\mathrm{geo}}^2\), so its
minimum over the admissible region is attained when the constraint is tight:
\[
b_{\mathrm{geo}}^2=\frac{1}{\gamma^2+\varrho^2}.
\]
Substituting this value gives the one-variable function
\[
\delta(a,b_{\mathrm{geo}})
\ge
\frac{\gamma^2+\varrho^2-2\gamma+1}{\gamma^2+\varrho^2-1}.
\]
A direct derivative computation shows that the right-hand side is minimized at
\[
\gamma=1+\varrho.
\]
Evaluating there yields
\[
\delta(a,b_{\mathrm{geo}})
\ge
\frac{\varrho}{1+\varrho}
=
\frac{1}{1+\sqrt{r_\star/r}}.
\]
Since the high-\(b_{\mathrm{geo}}\) branch attains its worst value on the same boundary, this lower
bound also controls the first region.

Therefore the worst admissible value of \(\delta(a,b_{\mathrm{geo}})\) is
\[
\delta(a,b_{\mathrm{geo}})
\ge
\frac{1}{1+\sqrt{r_\star/r}}.
\]
Hence, if \(\eta<1/(1+\sqrt{r_\star/r})\), there exists some \(\widehat\tau\in[0,a]\) such that
\[
\frac{\sqrt{1-a^2}\sqrt{1-\widehat\tau^2}+a\widehat\tau-b_{\mathrm{geo}}\widehat\tau}
{1+b_{\mathrm{geo}}\widehat\tau}
>
\eta.
\]
Multiplying by \(1+b_{\mathrm{geo}}\widehat\tau>0\) and using \eqref{eq:app-phi-ratio} yields
\[
\phi_\eta(\widehat\tau;a,b_{\mathrm{geo}})>0.
\]
\end{proof}

\begin{lemma}[Trivial certificate cases]
\label{lem:app-trivial-cases}
Assume \(E\neq 0\).
\begin{enumerate}[leftmargin=1.5em]
    \item If \(X\neq 0\) and \(Z_\perp=0\), then
    \[
    \psi_\lambda(1,0)\le -m_\mu\|E\|_F^2 < 0.
    \]
    \item If \(X=0\), then
    \[
    \psi_\lambda(0,1)\le -m_\mu\|E\|_F^2 < 0.
    \]
\end{enumerate}
\end{lemma}

\begin{proof}
These two cases do not need the full rank-threshold geometry.  If the error is entirely tangent to
the current factor span, \(D_0\) already detects it.  If \(X=0\), the second-order probes \(D_i\)
detect the whole error.  In both cases the PSD slack can be discarded in the favorable direction.

If \(X\neq 0\) and \(Z_\perp=0\), then \(E=E_T\). Using
\eqref{eq:app-gradient-identity}, \eqref{eq:app-exact-decomposition}, and
\eqref{eq:app-D0-identity},
\[
\psi_\lambda(1,0)
=
\bigl\langle \nabla g_{\mu,\lambda}(X),D_0\bigr\rangle
=
-\bigl\langle \mathcal H^{\mathrm{avg}}_{\mu,\Delta}[E]+N_{\mu,\lambda},E\bigr\rangle.
\]
Because \(N_{\mu,\lambda}\succeq 0\) and \(N_{\mu,\lambda}Z=0\),
\[
\langle N_{\mu,\lambda},E\rangle
=
\langle N_{\mu,\lambda},XX^\top-ZZ^\top\rangle
=
\langle N_{\mu,\lambda},XX^\top\rangle
\ge 0.
\]
Together with \eqref{eq:app-lifted-curvature}, this gives
\[
\psi_\lambda(1,0)
\le
-\bigl\langle \mathcal H^{\mathrm{avg}}_{\mu,\Delta}[E],E\bigr\rangle
\le
-m_\mu\|E\|_F^2.
\]

If \(X=0\), then \(E=E_N\). Using \eqref{eq:app-hessian-identity},
\eqref{eq:app-exact-decomposition}, and \eqref{eq:app-Di-identity},
\[
\psi_\lambda(0,1)
=
\left\langle
\mathcal H^{\mathrm{avg}}_{\mu,\Delta}[E]+N_{\mu,\lambda},
\sum_{i=1}^{r_\star}D_iD_i^\top
\right\rangle
=
-\bigl\langle \mathcal H^{\mathrm{avg}}_{\mu,\Delta}[E]+N_{\mu,\lambda},E\bigr\rangle,
\]
because the pointwise Hessian term vanishes when \(X=0\). The same argument yields
\[
\psi_\lambda(0,1)\le -m_\mu\|E\|_F^2.
\]
\end{proof}

\begin{lemma}[Nontrivial certificate upper bound]
\label{lem:app-nontrivial-upper-bound}
Assume \(E\neq 0\), \(X\neq 0\), and \(Z_\perp\neq 0\). Then for every \(\tau\in[0,a]\),
\begin{equation}
\label{eq:app-nontrivial-upper-bound}
\psi_\lambda\bigl(t_1(\tau),t_2(\tau)\bigr)
\le
-s_\mu\,\phi_{\vartheta_\mu}(\tau;a,b_{\mathrm{geo}})\,\|E\|_F^2,
\end{equation}
where
\(
\vartheta_\mu=(\kappa_\mu-1)/(\kappa_\mu+1)=d_\mu/s_\mu
\).
\end{lemma}

\begin{proof}
This is the nontrivial case where neither the first-order nor the second-order probe alone is
guaranteed to be enough.  We combine them through \(F_\tau\).  The lower curvature bound controls
the averaged term in the direction of \(F_\tau\), the upper curvature bound controls the pointwise
Hessian terms generated by \(D_i\), and the PSD slack again has the right sign.  The function
\(\phi_{\vartheta_\mu}\) records the net gain after these contributions are combined.

Let \(\dot E_i:=D_iX^\top+XD_i^\top\). Using
\eqref{eq:app-gradient-identity}, \eqref{eq:app-hessian-identity},
\eqref{eq:app-D0-identity}, \eqref{eq:app-Di-identity}, and
\eqref{eq:app-exact-decomposition}, we obtain
\[
\psi_\lambda\bigl(t_1(\tau),t_2(\tau)\bigr)
=
-\Bigl\langle \mathcal H^{\mathrm{avg}}_{\mu,\Delta}[E]+N_{\mu,\lambda},F_\tau\Bigr\rangle
+
\frac{t_2(\tau)}{2}\sum_{i=1}^{r_\star}
\bigl\langle \mathcal H^{\mathrm{pt}}_{\mu,X}[\dot E_i],\dot E_i\bigr\rangle.
\]
Because \(N_{\mu,\lambda}\succeq 0\) and \(N_{\mu,\lambda}Z=0\), the same PSD-slack argument as
above gives
\[
-\bigl\langle N_{\mu,\lambda},F_\tau\bigr\rangle\le 0.
\]
Also, as noted in the certificate construction, each \(\dot E_i\) has admissible off-diagonal rank,
so \eqref{eq:app-pointwise-curvature} yields
\[
\bigl\langle \mathcal H^{\mathrm{pt}}_{\mu,X}[\dot E_i],\dot E_i\bigr\rangle
\le
L_\mu\|\dot E_i\|_F^2.
\]

By polarization and \eqref{eq:app-lifted-curvature},
\[
\bigl\langle \mathcal H^{\mathrm{avg}}_{\mu,\Delta}[E],F_\tau\bigr\rangle
\ge
s_\mu\langle E,F_\tau\rangle - d_\mu\|E\|_F^2.
\]
Substituting this bound and \eqref{eq:app-Ftau-identities}, and using the geometric identity
\[
\frac{t_2(\tau)}{2}\sum_{i=1}^{r_\star}\|\dot E_i\|_F^2
=
\tau\,b_{\mathrm{geo}}\|E\|_F^2,
\]
we get
\begin{align*}
\psi_\lambda\bigl(t_1(\tau),t_2(\tau)\bigr)
&\le
-\Bigl[s_\mu\bigl(\sqrt{1-a^2}\sqrt{1-\tau^2}+a\tau\bigr)-d_\mu\Bigr]\|E\|_F^2
+
L_\mu\tau b_{\mathrm{geo}}\|E\|_F^2\\
&=
-s_\mu\,\phi_{\vartheta_\mu}(\tau;a,b_{\mathrm{geo}})\,\|E\|_F^2.
\end{align*}
\end{proof}

\begin{proof}[Proof of Theorem~\ref{thm:main-exact-sosp}]
Assume toward contradiction that \(E=XX^\top-ZZ^\top\neq 0\).

Second-order stationarity implies that every nonnegative combination of the certificates above must
be nonnegative.  The lemmas show that, under the rank condition, every possible nonzero shape of
\(E\) admits one certificate that is strictly negative.  This contradiction is the recovery
argument.

If \(X=0\), then Lemma~\ref{lem:app-trivial-cases}(2) gives
\[
\psi_\lambda(0,1)<0,
\]
contradicting \eqref{eq:app-certificate-nonnegative}.

If \(X\neq 0\) but \(Z_\perp=0\), then Lemma~\ref{lem:app-trivial-cases}(1) gives
\[
\psi_\lambda(1,0)<0,
\]
again contradicting \eqref{eq:app-certificate-nonnegative}.

It remains to consider the nontrivial case \(X\neq 0\) and \(Z_\perp\neq 0\). Since the main text
fixes \(\mu=\alpha_k/2\), the lifted curvature ratio satisfies
\[
\vartheta_\mu
=
\frac{\kappa_\mu-1}{\kappa_\mu+1}
=
\frac{\beta_k-\alpha_k}{\beta_k}.
\]
The rank condition in Theorem~\ref{thm:main-exact-sosp} is equivalent to
\[
\vartheta_\mu < \frac{1}{1+\sqrt{r_\star/r}},
\]
so Lemma~\ref{lem:app-phi-positive} gives some \(\tau\in[0,a]\) such that
\[
\phi_{\vartheta_\mu}(\tau;a,b_{\mathrm{geo}})>0.
\]
Applying Lemma~\ref{lem:app-nontrivial-upper-bound} at this \(\tau\) yields
\[
\psi_\lambda\bigl(t_1(\tau),t_2(\tau)\bigr)
\le
-s_\mu\,\phi_{\vartheta_\mu}(\tau;a,b_{\mathrm{geo}})\,\|E\|_F^2
<0,
\]
contradicting \eqref{eq:app-certificate-nonnegative}. Therefore \(E=0\).

Hence
\[
XX^\top=ZZ^\top.
\]
Taking the \((1,2)\)-block gives
\[
UV^\top = U_\star V_\star^\top = M_\star^\lambda.
\]
Moreover, \(XX^\top=ZZ^\top\) implies that \(X=ZR\) for some matrix
\(R\in\mathbb R^{r_\star\times r}\) with \(RR^\top=I_{r_\star}\). Hence
\(U=U_\star R\) and \(V=V_\star R\), so \(U^\top U=V^\top V\) and
\[
\frac12\bigl(\|U\|_F^2+\|V\|_F^2\bigr)
=
\|M_\star^\lambda\|_*.
\]
Thus the factor pair attains
\[
h_{\mu,\lambda}(U,V)
=
\widehat L(M_\star^\lambda)+\lambda\|M_\star^\lambda\|_*
=
F_\lambda(M_\star^\lambda),
\]
which is the global infimum for $F_\lambda(M)$ in Assumption~\ref{asm:main-regularized-target}. Therefore
\((U,V)\) is globally optimal.
\end{proof}

\subsubsection{General balancing weights}
\label{sec:app-general-mu}

The preceding certificate argument also yields an explicit rank condition for any balancing
weight in \(0<\mu<\alpha_k\). We record this extension using the lifted curvature constants
introduced above.

\begin{corollary}[SOSP recovery for general balancing weights]
\label{cor:app-general-mu}
Suppose Assumption~\ref{asm:main-regularized-target} holds with
\(1\le r_\star\le r\) and \(\lambda>0\). Let
\(0<\alpha_k<\beta_k<\infty\), fix \(0<\mu<\alpha_k\), and set
\[
m_\mu:=2\min\{\mu,\alpha_k-\mu\},\qquad
L_\mu:=2\max\{\mu,\beta_k-\mu\},\qquad
\kappa_\mu:=\frac{L_\mu}{m_\mu}.
\]
Every SOSP \((U,V)\) of \(h_{\mu,\lambda}\) whose product \(M=UV^\top\) satisfies
Assumption~\ref{asm:main-rsc-rss} with these constants and \(k\ge r+r_\star\)
is globally optimal and satisfies \(UV^\top=M_\star^\lambda\) whenever
\begin{equation}
\label{eq:app-general-mu-rank-condition}
\frac{r}{r_\star}>\left(\frac{\kappa_\mu-1}{2}\right)^2.
\end{equation}
Moreover, the corresponding stacked factors satisfy \(XX^\top=ZZ^\top\).
\end{corollary}

\begin{proof}
For \(0<\mu<\alpha_k\), the lifted curvature bounds
\eqref{eq:app-lifted-curvature}--\eqref{eq:app-pointwise-curvature} hold with
\(m_\mu>0\), and the KKT slack in Lemma~\ref{lem:app-exact-decomposition} remains
positive semidefinite. Put \(t:=\sqrt{r_\star/r}>0\). Since \(\kappa_\mu\ge1\),
\eqref{eq:app-general-mu-rank-condition} is equivalent to
\[
(\kappa_\mu-1)t<2
\quad\Longleftrightarrow\quad
\vartheta_\mu:=\frac{\kappa_\mu-1}{\kappa_\mu+1}
<\frac{1}{1+t}.
\]
If an SOSP had \(E=XX^\top-ZZ^\top\ne0\),
Lemma~\ref{lem:app-trivial-cases} would give a strictly negative certificate when
\(X=0\) or \(Z_\perp=0\). In the remaining case,
Lemma~\ref{lem:app-phi-positive} gives a \(\tau\in[0,a]\) with
\(\phi_{\vartheta_\mu}(\tau;a,b_{\mathrm{geo}})>0\), and
Lemma~\ref{lem:app-nontrivial-upper-bound} yields
\[
\psi_\lambda\bigl(t_1(\tau),t_2(\tau)\bigr)
\le -s_\mu\phi_{\vartheta_\mu}(\tau;a,b_{\mathrm{geo}})\|E\|_F^2<0,
\qquad s_\mu:=\frac{L_\mu+m_\mu}{2}>0.
\]
Both cases contradict \eqref{eq:app-certificate-nonnegative}. Hence \(E=0\), and
balanced recovery and global optimality follow exactly as in the concluding step of the
preceding theorem proof.
\end{proof}

For the main-text choice \(\mu=\alpha_k/2\), these constants reduce to
\[
m_\mu=\alpha_k,\qquad L_\mu=2\beta_k-\alpha_k,\qquad
\frac{\kappa_\mu-1}{2}
=\frac{\beta_k-\alpha_k}{\alpha_k}
=\frac{2\delta_k^{\mathrm{LoRA}}}{1-\delta_k^{\mathrm{LoRA}}}.
\]
Thus \eqref{eq:app-general-mu-rank-condition} recovers the main-text rank condition
\eqref{eq:main-rank-threshold}. Proposition~\ref{prop:main-decay-only-balance} and the
transfer argument below also apply to this general-\(\mu\) guarantee.

\subsection{Proofs for the Transfer to Standard LoRA}
\label{sec:app-proof-standard-lora-transfer}

\begin{proof}[Proof of Proposition~\ref{prop:main-decay-only-balance}]
Let \(M:=UV^\top\) and \(G:=\nabla \widehat L(M)\). First-order stationarity of
\(\widetilde h_\lambda\) gives
\[
GV+\lambda U=0,
\qquad
G^\top U+\lambda V=0.
\]
Left-multiplying these equations by \(U^\top\) and \(V^\top\), respectively, and subtracting yields
\[
U^\top G V - V^\top G^\top U + \lambda(U^\top U - V^\top V)=0.
\]
If \(A:=U^\top G V\), then the first term is \(A-A^\top\), which is skew-symmetric, whereas
\(U^\top U - V^\top V\) is symmetric. Their sum is zero, so both terms must vanish. Hence
\[
U^\top U = V^\top V.
\]

Now define
\[
R(U,V):=\frac14\|U^\top U-V^\top V\|_F^2,
\qquad
h_{\mu,\lambda}=\widetilde h_\lambda+\mu R.
\]
At every balanced point, \(R(U,V)=0\), so \((U,V)\) is a global minimizer of the nonnegative
\(C^2\) function \(R\). Therefore
\[
\nabla R(U,V)=0,
\qquad
\nabla^2 R(U,V)\succeq 0.
\]
If \((U,V)\) is an SOSP of \(\widetilde h_\lambda\), then adding \(\mu R\) preserves both
first-order stationarity and second-order nonnegativity. Thus \((U,V)\) is also an SOSP of
\(h_{\mu,\lambda}\) for every \(\mu\ge 0\).
\end{proof}

\begin{proof}[Proof of Corollary~\ref{cor:main-decay-only-transfer}]
Fix the same parameter \(\mu\in(0,\alpha_k)\) as in Theorem~\ref{thm:main-exact-sosp}. By
Proposition~\ref{prop:main-decay-only-balance}, every SOSP of the standard LoRA objective
\eqref{eq:main-decay-only-objective} is also an SOSP of the balanced objective
\eqref{eq:main-balanced-lora-objective}. Theorem~\ref{thm:main-exact-sosp} therefore gives
\[
UV^\top=M_\star^\lambda.
\]
At this point, the balancing penalty vanishes, so
\[
\widetilde h_\lambda(U,V)
=
\widehat L(M_\star^\lambda)+\frac{\lambda}{2}\bigl(\|U\|_F^2+\|V\|_F^2\bigr)
=
\widehat L(M_\star^\lambda)+\lambda\|M_\star^\lambda\|_*,
\]
where the last equality follows from the balanced square-root characterization of $M_\star^\lambda$. Hence the standard LoRA factor pair attains the
same value as the optimal rank-constrained matrix-space problem, so it is globally optimal for the
standard LoRA objective as well.
\end{proof}

\section{Additional Proofs for Section~\ref{sec:data-dependent-curvature-lora}}
\label{sec:app-section5-proofs}

This appendix section follows the notation of Section~\ref{sec:data-dependent-curvature-lora}.
The empirical distribution \(\nu\) over supervised prediction instances defines
\(\widehat L_\nu\). We first specify the certificate region and admissible family underlying
the data-side bounds, then derive the cross-entropy identities and curvature certificate.
Recall that \(x\) denotes a supervised prediction instance induced by the original examples,
\(z_x(M)\) is the model's pre-softmax logit vector for this instance,
\(p_x(M)=\operatorname{softmax}(z_x(M))\) is the predicted probability vector, and \(q_x\) is the
target probability vector. We also write
\[
Q_\nu(M,H):=\mathbb E_\nu\|P_\perp A_{x,M}[H]\|_2^2
\]
for the projected, unweighted coverage form.

\subsection{Certificate Region and Admissible Family}
\label{sec:app-data-admissible-family}

Fix the empirical distribution \(\nu\), search rank \(r\), and a regularized target
\(M_{\nu,\star}^{\lambda}\) from Assumption~\ref{asm:main-regularized-target}, with
\(r_{\nu,\star}=\operatorname{rank}(M_{\nu,\star}^{\lambda})\le r\).
Let
\[
\mathcal M_r
:=\{M\in\mathbb R^{n_1\times n_2}:\operatorname{rank}(M)\le r\}.
\]
The certificate region collects the target--candidate segments in
Assumption~\ref{asm:main-rsc-rss}:
\begin{equation}
\label{eq:app-data-certificate-region}
\mathcal M_{\mathrm{cert}}(\nu,r)
:=\bigl\{(1-t)M_{\nu,\star}^{\lambda}+tM:
M\in\mathcal M_r,\ t\in[0,1]\bigr\}.
\end{equation}
For \(k\ge r+r_{\nu,\star}\), define the admissible family
\begin{equation}
\label{eq:app-data-admissible-family}
\mathcal T_k(\nu,r)
:=\left\{\begin{aligned}
(\widetilde M,H):\ &\widetilde M\in\mathcal M_{\mathrm{cert}}(\nu,r),\\
&H\in\mathbb R^{n_1\times n_2},\quad
\operatorname{rank}(H)\le k,\quad \|H\|_F=1
\end{aligned}\right\}.
\end{equation}
Thus the matrix points include both endpoints and the full intervening segments, while
\(H\) ranges over all unit rank-at-most-\(k\) directions. The segment points can have rank
up to \(r+r_{\nu,\star}\), even though each candidate update has rank at most \(r\).
For a candidate-specific application, the same definitions may be restricted to the single
segment joining that candidate to the target; the resulting bounds then apply to that
candidate. Unless stated otherwise, the quantities below use the full family
\(\mathcal T_k(\nu,r)\) defined above.

The coverage bounds and nonlinear residual bound are
\begin{align}
\label{eq:app-data-coverage-extrema}
c_{\nu,k}&:=\inf_{(M,H)\in\mathcal T_k(\nu,r)}Q_\nu(M,H),
\qquad C_{\nu,k}:=\sup_{(M,H)\in\mathcal T_k(\nu,r)}Q_\nu(M,H),\\
\label{eq:app-data-residual-extremum}
\rho_{\nu,k}&:=\sup_{(M,H)\in\mathcal T_k(\nu,r)}|R_\nu(M,H)|.
\end{align}
To define the softmax-weighting bounds, set
\[
\mathcal T_k^+(\nu,r)
:=\{(M,H)\in\mathcal T_k(\nu,r):Q_\nu(M,H)>0\}.
\]
Whenever this set is nonempty, let
\begin{equation}
\label{eq:app-data-softmax-extrema}
\omega_{\nu,k}^-:=\inf_{(M,H)\in\mathcal T_k^+(\nu,r)}
\frac{G_\nu(M,H)}{Q_\nu(M,H)},
\qquad
\omega_{\nu,k}^+:=\sup_{(M,H)\in\mathcal T_k^+(\nu,r)}
\frac{G_\nu(M,H)}{Q_\nu(M,H)}.
\end{equation}
All extrema are understood as infima or suprema; attainment is not required. Their dependence
on \(r\) through the admissible family is suppressed in the subscripts, as in the main text.
The data-side certificate is invoked when
\[
0<c_{\nu,k}\le C_{\nu,k}<\infty,\qquad
0<\omega_{\nu,k}^-\le\omega_{\nu,k}^+<\infty,\qquad
\rho_{\nu,k}<\omega_{\nu,k}^-c_{\nu,k}.
\]
In particular, \(c_{\nu,k}>0\) implies
\(\mathcal T_k^+(\nu,r)=\mathcal T_k(\nu,r)\), so all directional ratios used in the
certificate are well-defined. The proof in
Appendix~\ref{sec:app-proof-data-side-curvature-certificate} then supplies uniform restricted
curvature bounds on these same target--candidate segments.

\subsection{Cross-Entropy Hessian Decomposition}
\label{sec:app-data-ce-hessian}

\begin{proposition}[Cross-entropy Hessian]
\label{prop:app-data-ce-hessian}
Let \(z_x(M)\in\mathbb R^K\) be the pre-softmax logits for supervised instance \(x\), let
\(p_x(M):=\operatorname{softmax}(z_x(M))\) be the predicted probabilities, and let
\(q_x\in\Delta^{K-1}\) be the target probabilities. Define
\[
\ell_x(M):=-\langle q_x,\log p_x(M)\rangle.
\]
For any matrix direction \(H\), define the first- and second-order logit responses
\[
A_{x,M}[H]:=Dz_x(M)[H],
\qquad
\mathcal B_{x,M}[H,H]:=D^2z_x(M)[H,H],
\]
and define the softmax covariance matrix
\(\Omega_x(M):=\operatorname{Diag}(p_x(M))-p_x(M)p_x(M)^\top\). Then
\[
D^2\ell_x(M)[H,H]
=
\langle A_{x,M}[H],\Omega_x(M)A_{x,M}[H]\rangle
+
\langle p_x(M)-q_x,\mathcal B_{x,M}[H,H]\rangle.
\]
Consequently,
\[
\langle\nabla^2\widehat L_\nu(M)[H],H\rangle
=
G_\nu(M,H)+R_\nu(M,H),
\]
with \(G_\nu\) and \(R_\nu\) as defined in Section~\ref{sec:data-dependent-curvature-lora}.
\end{proposition}

\begin{proof}
For \(\phi(z,q):=-\langle q,\log\operatorname{softmax}(z)\rangle\), one has
\[
\nabla_z\phi(z,q)=p-q,
\qquad
\nabla_z^2\phi(z,q)=\operatorname{Diag}(p)-pp^\top.
\]
Applying the second-order chain rule to \(\ell_x(M)=\phi(z_x(M),q_x)\) gives the displayed
identity. Since \(\Omega_x(M)\mathbf 1=0\), common shifts of all logits are ignored by the
Gauss--Newton term. Averaging over \(x\sim\nu\) gives the identity for \(\widehat L_\nu\).
\end{proof}

\subsection{Proof of the Data-Side Curvature Certificate}
\label{sec:app-proof-data-side-curvature-certificate}

We derive the conservative LoRA-RIP certificate used in
\eqref{eq:main-ce-lora-rip-certificate}. Fix an admissible pair
\((M,H)\in\mathcal T_k(\nu,r)\). Since \(\|H\|_F=1\), the definitions in
Appendix~\ref{sec:app-data-admissible-family} give
\[
c_{\nu,k}\le Q_\nu(M,H)\le C_{\nu,k}.
\]
Moreover, the data-side restricted curvature condition in
Section~\ref{subsec:data-lora-rip-metric} gives
\(c_{\nu,k}>0\), so every admissible pair has \(Q_\nu(M,H)\ge c_{\nu,k}>0\). The ratios defining
\(\omega_{\nu,k}^{\pm}\) are therefore well-defined on this family, and
\[
\omega_{\nu,k}^- Q_\nu(M,H)
\le
G_\nu(M,H)
\le
\omega_{\nu,k}^+ Q_\nu(M,H).
\]
Hence
\[
\omega_{\nu,k}^-c_{\nu,k}
\le
G_\nu(M,H)
\le
\omega_{\nu,k}^+C_{\nu,k}.
\]
Together with \(|R_\nu(M,H)|\le\rho_{\nu,k}\), the Hessian decomposition in
Proposition~\ref{prop:app-data-ce-hessian} implies
\[
\omega_{\nu,k}^-c_{\nu,k}-\rho_{\nu,k}
\le
\langle\nabla^2\widehat L_\nu(M)[H],H\rangle
\le
\omega_{\nu,k}^+C_{\nu,k}+\rho_{\nu,k}.
\]
By homogeneity, the same bounds hold for non-unit directions after multiplying by
\(\|H\|_F^2\). Therefore the data induce valid restricted curvature constants
\[
\alpha_{\nu,k}^{\mathrm{CE}}
=
\omega_{\nu,k}^-c_{\nu,k}-\rho_{\nu,k},
\qquad
\beta_{\nu,k}^{\mathrm{CE}}
=
\omega_{\nu,k}^+C_{\nu,k}+\rho_{\nu,k}.
\]
Substituting these constants into the symmetric LoRA-RIP normalization
\((\beta-\alpha)/(\beta+\alpha)\) gives
\[
\delta_{\nu,k}^{\mathrm{LoRA}}
\le
\overline\delta_{\nu,k}^{\mathrm{LoRA}}
=
\frac{
\omega_{\nu,k}^+C_{\nu,k}
-
\omega_{\nu,k}^-c_{\nu,k}
+
2\rho_{\nu,k}
}{
\omega_{\nu,k}^+C_{\nu,k}
+
\omega_{\nu,k}^-c_{\nu,k}
},
\]
which is \eqref{eq:main-ce-lora-rip-certificate}. Plugging this conservative value into
Theorem~\ref{thm:main-exact-sosp} yields the rank budget stated in the main text.

\subsection{Mode Imbalance and Duplicate Sampling}
\label{sec:app-data-descriptor-identities}

\begin{proposition}[Group mass controls coverage conditioning]
\label{prop:app-longtail-effect}
A group mixture means the following two-stage sampling distribution: first choose a group
\(g\in\{1,\ldots,G\}\) with probability \(\pi_g\), and then draw \(x\sim\nu_g\). Equivalently,
\(\nu_\pi=\sum_{g=1}^G\pi_g\nu_g\). Let \(\mathcal T\) be the admissible comparison family for
the mixture, for example \(\mathcal T_k(\nu_\pi,r)\). For every fixed \((M,H)\in\mathcal T\),
\[
Q_{\nu_\pi}(M,H)
=
\sum_{g=1}^G\pi_g Q_{\nu_g}(M,H),
\qquad
G_{\nu_\pi}(M,H)
=
\sum_{g=1}^G\pi_g G_{\nu_g}(M,H).
\]
Let \((M_+,H_+)\in\mathcal T\) be a direction strongly excited by group \(g_+\), and define
\[
U_+:=Q_{\nu_{g_+}}(M_+,H_+).
\]
Let \((M_-,H_-)\in\mathcal T\) be a direction that is weakly covered under the mixture. For a
possibly different group \(g_-\), define its weak-direction contribution and the leakage from the
other groups by
\[
L_-:=Q_{\nu_{g_-}}(M_-,H_-),
\qquad
\varepsilon_-(\pi):=\sum_{j\ne g_-}\pi_jQ_{\nu_j}(M_-,H_-).
\]
The leakage term \(\varepsilon_-(\pi)\) is small when the groups are well separated in their
projected logit-response directions.
Then
\[
\sup_{(M,H)\in\mathcal T}Q_{\nu_\pi}(M,H)
\ge
\pi_{g_+}U_+,
\qquad
\inf_{(M,H)\in\mathcal T}Q_{\nu_\pi}(M,H)
\le
Q_{\nu_\pi}(M_-,H_-)=
\pi_{g_-}L_-+\varepsilon_-(\pi).
\]
Consequently, if the denominator below is positive,
\[
\frac{\sup_{(M,H)\in\mathcal T}Q_{\nu_\pi}(M,H)}
     {\inf_{(M,H)\in\mathcal T}Q_{\nu_\pi}(M,H)}
\ge
\frac{\pi_{g_+}U_+}{\pi_{g_-}L_-+\varepsilon_-(\pi)} .
\]
In words, a dataset can have poor coverage conditioning when too much mass is concentrated on a
mode that already gives large coverage response, while another mode receives too little mass and is
not covered by the remaining groups. Duplicate samples are the limiting case in which one mode is a
single atom: they reinforce the projected logit-response directions excited by that atom rather than
increasing coverage uniformly.
\end{proposition}

\begin{proof}
Linearity of expectation gives
\[
Q_{\nu_\pi}(M,H)
=
\sum_{j=1}^G\pi_jQ_{\nu_j}(M,H),
\qquad
G_{\nu_\pi}(M,H)
=
\sum_{j=1}^G\pi_jG_{\nu_j}(M,H).
\]
Evaluating the \(Q\)-identity at \((M_+,H_+)\) and using nonnegativity gives the upper-side bound.
Evaluating it at
\((M_-,H_-)\) gives
\(\inf_{(M,H)\in\mathcal T}Q_{\nu_\pi}(M,H)\le
Q_{\nu_\pi}(M_-,H_-)=\pi_{g_-}L_-+\varepsilon_-(\pi)\).
The ratio bound follows by combining the two inequalities.
\end{proof}

\subsection{Gradient-Sketch Proxy for Experiments}
\label{sec:app-data-gradient-sketch-proxy}

At a fixed operating point \(M_0\), the Gauss--Newton term can be viewed as a covariance on
vectorized LoRA-update directions after they pass through the logit Jacobian. If
\(h=\operatorname{vec}(H)\) and \(J_{x,M_0}\) is the matrix of the
linear map \(H\mapsto \Omega_x(M_0)^{1/2}A_{x,M_0}[H]\), then
\[
G_\nu(M_0,H)
=
h^\top
\left(\mathbb E_{x\sim\nu}J_{x,M_0}^\top J_{x,M_0}\right)
h .
\]
Thus any low-dimensional sketch that preserves the main modes of the logit Jacobian gives
an empirical proxy for the weighted coverage spectrum. The gradient groups used in
Section~\ref{sec:experiments} should be read in this sense: they cluster examples by a
margin-Jacobian sketch at a shared reference point. These sketches are not estimates of the exact
constants \(c_{\nu,k}\), \(C_{\nu,k}\), or \(\omega_{\nu,k}^{\pm}\), but they are aligned with the
directions that enter the data-side certificate.

\subsection{Directional Softmax Saturation and Curvature Leverage}
\label{sec:app-data-softmax-saturation}

\begin{proposition}[Softmax covariance as variance]
\label{prop:app-data-softmax-variance}
For \(p\in\Delta^{K-1}\), \(\Omega(p):=\operatorname{Diag}(p)-pp^\top\), and any
\(u\in\mathbb R^K\),
\[
u^\top \Omega(p)u
=
\sum_{c=1}^K p_cu_c^2-\left(\sum_{c=1}^K p_cu_c\right)^2
=
\min_{a\in\mathbb R}\sum_{c=1}^K p_c(u_c-a)^2.
\]
In particular, \(\Omega(p)\mathbf 1=0\), so only centered logit perturbations contribute to
Gauss--Newton curvature.
\end{proposition}

\begin{proof}
The first equality is the usual variance identity under a categorical distribution with probabilities
\(p\). The minimizing constant is \(a=\sum_c p_cu_c\). The identity
\(\Omega(p)\mathbf 1=0\) follows immediately from the definition of \(\Omega(p)\).
\end{proof}

\begin{proposition}[Easy examples have limited Gauss--Newton leverage]
\label{prop:app-data-easy-gn-leverage}
For \(p\in\Delta^{K-1}\), let \(p_{\max}=1-\delta\) with \(\delta\in[0,1]\). Then
\[
\lambda_{\max}(\Omega(p))
\le
\operatorname{tr}(\Omega(p))
=
1-\|p\|_2^2
\le
2\delta-\delta^2 .
\]
Consequently, for the single-example Gauss--Newton contribution
\[
G_x(M,H):=
\langle A_{x,M}[H],\Omega_x(M)A_{x,M}[H]\rangle ,
\]
we have
\[
G_x(M,H)
\le
(2\delta_x-\delta_x^2)\|A_{x,M}[H]\|_2^2,
\qquad
\delta_x:=1-\max_c p_{x,c}(M).
\]
Thus a highly confident example cannot contribute much softmax-side Gauss--Newton curvature unless
its logit-Jacobian response along the direction is large.
\end{proposition}

\begin{proof}
Since \(\Omega(p)\) is positive semidefinite,
\(\lambda_{\max}(\Omega(p))\le\operatorname{tr}(\Omega(p))\). The trace identity follows from
\(\operatorname{tr}(\operatorname{Diag}(p)-pp^\top)=\sum_c p_c-\sum_c p_c^2=1-\|p\|_2^2\).
If \(p_{\max}=1-\delta\), then \(\|p\|_2^2\ge(1-\delta)^2\), which gives the final inequality.
Applying the operator-norm bound to \(u=A_{x,M}[H]\) proves the single-example claim.
\end{proof}

\begin{proposition}[Directional saturation decomposition]
\label{prop:app-data-directional-saturation}
For any admissible \((M,H)\) with \(Q_\nu(M,H)>0\), define
\[
\sigma_\nu(M,H):=G_\nu(M,H)/Q_\nu(M,H).
\]
Then \(G_\nu(M,H)=\sigma_\nu(M,H)Q_\nu(M,H)\), and hence
\begin{align*}
\inf_{\mathcal T_k(\nu,r)}G_\nu(M,H)
&=
\inf_{\mathcal T_k(\nu,r)}\sigma_\nu(M,H)Q_\nu(M,H),\\
\sup_{\mathcal T_k(\nu,r)}G_\nu(M,H)
&=
\sup_{\mathcal T_k(\nu,r)}\sigma_\nu(M,H)Q_\nu(M,H).
\end{align*}
In a binary-margin parameterization, let
\(m_x(M):=z_{x,1}(M)-z_{x,2}(M)\), let \(\phi_{x,M}\) be the Frobenius-Riesz representative
of the linear map \(H\mapsto Dm_x(M)[H]\), and set
\(s_x(M):=p_{x,1}(M)p_{x,2}(M)\). Then this becomes the scalar-weighted covariance form
\[
G_\nu(M,H)
=
\mathbb E_\nu\bigl[s_x(M)\langle \phi_{x,M},H\rangle^2\bigr],
\qquad
Q_\nu(M,H)
=
\frac12\mathbb E_\nu\langle \phi_{x,M},H\rangle^2,
\]
for a softmax saturation weight \(s_x(M)\ge0\). If \(\sigma_\nu(M,H)\) is approximately constant
over theorem-relevant directions, it mostly rescales both Gauss--Newton curvature extremes and does
not improve the pure Gauss--Newton ratio. Directional saturation
improves this curvature ratio only when the weighted mass increases the weak-direction floor
more than it increases the ceiling.
\end{proposition}

\begin{proof}
The first two identities are definitions. The binary-margin identities follow from
Proposition~\ref{prop:app-data-softmax-variance}. For \(K=2\),
\[
P_\perp A_{x,M}[H]
=\frac12Dm_x(M)[H]\begin{bmatrix}1\\-1\end{bmatrix},
\qquad
\|P_\perp A_{x,M}[H]\|_2^2
=\frac12\bigl(Dm_x(M)[H]\bigr)^2,
\]
which explains the factor \(1/2\) in \(Q_\nu\) under the centered-logit normalization used
throughout the paper. The Gauss--Newton form retains the weight
\(s_x=p_{x,1}p_{x,2}\). The final statement follows by comparing the pure
Gauss--Newton ratio: a nearly constant \(\sigma_\nu\) cancels in this
ratio, whereas direction-dependent \(\sigma_\nu\) changes the ratio according to where it places
weighted curvature mass.
\end{proof}

%% file: appendix_experiments.tex
\section{Additional Experimental Details}
\label{sec:app-experiments}

This appendix records the protocols and full result tables for the rank sweeps and MNLI
data interventions in Section~\ref{sec:experiments}. These tables report controlled single-seed
runs. Appendix~\ref{sec:app-supplementary-experiments} presents multi-seed evaluations,
landscape diagnostics, and robustness checks, with the settings specified for each study.

\subsection{Shared Protocol}

The MNLI/MPNet and CIFAR-100/ViT rank sweeps and MNLI data interventions use a controlled probe-then-LoRA protocol. We first train a linear classification
head on the frozen backbone, and then keep this head fixed while fine-tuning only the pre-specified
LoRA module. This is related to linear-probing-then-fine-tuning protocols, but differs from variants
that continue updating the classification head during the fine-tuning stage
~\citep{tomihariUnderstandingLinearProbing2024a}. We keep the head fixed to isolate the effect of
the LoRA update itself and keep the experiment close to the single-trainable-block setting analyzed
in the main text. For the experiments in this appendix, unless stated otherwise, we use weight decay
\(\lambda_A=\lambda_B=10^{-4}\), balancing coefficient \(\mu=10^{-4}\), learning rate \(10^{-4}\),
and the initialization scheme in which the \(A\)-factor is Kaiming initialized with scale \(0.01\)
and the \(B\)-factor is initialized to zero. The train/validation split uses seed \(17\), and the
tables in this appendix use optimization seed \(0\). We use Adam with explicit weight decay for these runs.
The diagnostic and robustness studies specify their seeds and protocol differences in
Appendix~\ref{sec:app-supplementary-experiments}.

\begin{table}[H]
\centering
\small
\caption{Protocols for the rank sweeps and MNLI data interventions. Diagnostic and robustness studies are reported in
Appendix~\ref{sec:app-supplementary-experiments}.}
\label{tab:app-exp-protocols}
\resizebox{\textwidth}{!}{%
\begin{tabular}{llll}
\toprule
Experiment & Dataset / model & LoRA target & Main variable \\
\midrule
NLP rank sweep
& MNLI / MPNet-base
& layer 10 attention query
& rank \(r\in\{1,4,16,32,64,128,256,512\}\) \\
CV rank sweep
& CIFAR-100 / ViT-B/16
& query projection in block 10 attention
& rank \(r\in\{1,4,16,32,64,128,256,512\}\) \\
NLP coverage
& MNLI / MPNet-base
& layer 10 attention query
& gradient-group sampling mass and rank \\
NLP saturation
& MNLI / MPNet-base
& layer 10 attention query
& saturation band under fixed gradient-group quotas and rank \\
\bottomrule
\end{tabular}%
}
\end{table}

For the MNLI/MPNet experiments in this appendix, the train/validation/test sizes are \(353{,}432/39{,}270/9{,}815\), the batch size is
\(32\), and the maximum sequence length is \(128\). For CIFAR-100, the train/validation/test sizes
are \(45{,}000/5{,}000/10{,}000\), and the batch size is \(64\). The MNLI rank sweep trains LoRA for
10 epochs; the CIFAR-100 rank sweep trains LoRA for 30 epochs; and the MNLI data-intervention runs
train LoRA for 50 epochs on no-replacement subsets of \(17{,}526\) examples.

\subsection{Compute Resources}
\label{sec:app-compute-resources}

The experiment groups in Table~\ref{tab:app-compute-resources} were run with single-GPU
training workers on NVIDIA RTX 5880 Ada Generation GPUs with \(48\) GB memory. Each rank or
data condition was launched as an independent Python process on one GPU, without multi-node or
distributed training. Up to four GPUs were used in parallel. For these runs, the observed
GPU memory usage per LoRA or linear-probe worker was below \(4\) GB in \texttt{nvidia-smi}, and
data loading used up to four CPU workers per process.

\begin{table}[H]
\centering
\small
\caption{Hardware and training configurations for the rank sweeps and MNLI data interventions.}
\label{tab:app-compute-resources}
\resizebox{\textwidth}{!}{%
\begin{tabular}{lll}
\toprule
Experiment group & Runs & Batch size / epochs \\
\midrule
NLP rank sweep & 1 linear probe + 8 LoRA ranks & LP: 32 / 10; LoRA: 32 / 10 \\
CV rank sweep & 1 linear probe + 8 LoRA ranks & LP: 64 / 20; LoRA: 64 / 30 \\
NLP gradient coverage & 4 data conditions $\times$ 4 ranks & LoRA: 32 / 50; subset size $17{,}526$ \\
NLP saturation bands & 3 bands $\times$ 4 ranks & LoRA: 32 / 50; subset size $17{,}526$ \\
\bottomrule
\end{tabular}%
}
\end{table}

The hardware and memory measurements above apply to the experiment groups in
Table~\ref{tab:app-compute-resources}. Diagnostic and robustness protocols are described in
Appendix~\ref{sec:app-supplementary-experiments}.

\subsection{Rank Sweep Results}

\begin{table}[H]
\centering
\small
\caption{MNLI rank sweep. The test accuracy is reported at the best validation checkpoint.}
\label{tab:app-nlp-rank-sweep}
\resizebox{\textwidth}{!}{%
\begin{tabular}{rrrrrrr}
\toprule
Rank & Best val acc & Test acc & Best val CE & Best train CE & Final train CE & Best epoch \\
\midrule
1   & 0.6047 & 0.6036 & 0.8737 & 0.9506 & 0.9506 & 9  \\
4   & 0.6199 & 0.6167 & 0.8501 & 0.9325 & 0.9325 & 10 \\
16  & 0.6360 & 0.6324 & 0.8318 & 0.9184 & 0.9184 & 10 \\
32  & 0.6476 & 0.6456 & 0.8130 & 0.9032 & 0.9032 & 9  \\
64  & 0.6547 & 0.6506 & 0.8018 & 0.8939 & 0.8939 & 9  \\
128 & 0.6630 & 0.6570 & 0.7883 & 0.8829 & 0.8829 & 7  \\
256 & 0.6696 & 0.6634 & 0.7752 & 0.8719 & 0.8719 & 10 \\
512 & 0.6728 & 0.6686 & 0.7677 & 0.8650 & 0.8650 & 9  \\
\bottomrule
\end{tabular}%
}
\end{table}

\begin{table}[H]
\centering
\small
\caption{CIFAR-100 rank sweep. The train loss improves with rank, while held-out accuracy
is much flatter than in MNLI.}
\label{tab:app-cv-rank-sweep}
\resizebox{\textwidth}{!}{%
\begin{tabular}{rrrrrrr}
\toprule
Rank & Best val acc & Test acc & Best val CE & Best train CE & Final train CE & Best epoch \\
\midrule
1   & 0.7384 & 0.7243 & 1.2393 & 0.1137 & 0.1137 & 2  \\
4   & 0.7392 & 0.7267 & 1.2299 & 0.1101 & 0.1101 & 28 \\
16  & 0.7400 & 0.7251 & 1.2291 & 0.1057 & 0.1057 & 8  \\
32  & 0.7398 & 0.7265 & 1.2323 & 0.1013 & 0.1013 & 4  \\
64  & 0.7414 & 0.7279 & 1.2254 & 0.0940 & 0.0940 & 16 \\
128 & 0.7414 & 0.7296 & 1.2236 & 0.0840 & 0.0840 & 12 \\
256 & 0.7422 & 0.7318 & 1.2244 & 0.0714 & 0.0714 & 14 \\
512 & 0.7464 & 0.7307 & 1.2217 & 0.0584 & 0.0584 & 15 \\
\bottomrule
\end{tabular}%
}
\end{table}

\FloatBarrier

\subsection{Gradient-Coverage Construction}
\label{sec:app-gradient-coverage}

The MNLI coverage experiment changes the data distribution without changing the subset size or
label budget. The candidate pool is sampled once from MNLI train, and the final coverage proxy is
defined by a gradient sketch with respect to the adapted matrix:
\begin{enumerate}[leftmargin=1.5em]
    \item Use the frozen MPNet backbone and the trained linear-probe head as a shared reference
    point.
    \item For each candidate example, choose the gold label and the strongest alternative label, and
    backpropagate the corresponding margin logit to the target query matrix.
    \item Randomly project the activation and gradient-output factors to \(16\times16\), producing a
    \(256\)-dimensional bilinear sketch of the margin-Jacobian feature.
    \item Normalize the sketches and cluster them within each label into eight groups, giving 24
    label-conditioned gradient groups.
    \item Draw no-replacement subsets of \(17{,}526\) examples with the same label budget
    \((5847,5788,5891)\) but different mass over these gradient groups.
\end{enumerate}

The conditions are a balanced distribution with uniform mass over the eight groups within each
label, and two geometric long-tail conditions with decay factors \(0.75\) and \(0.45\).
Figure~\ref{fig:app-coverage-group-mass} shows the resulting group masses.

\begin{figure}[H]
    \centering
    \includegraphics[width=0.95\linewidth]{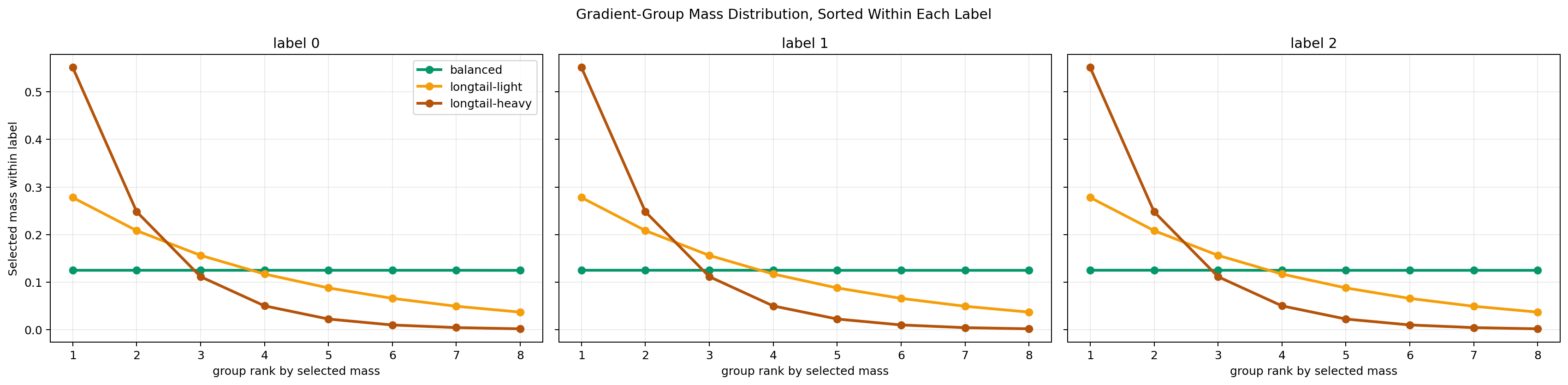}
    \caption{Sampling mass over label-conditioned gradient groups. The heavy long-tail intervention
    collapses most mass onto the first few groups in every label, while the balanced condition spreads
    mass uniformly across the gradient groups.}
    \label{fig:app-coverage-group-mass}
\end{figure}

\begin{table}[H]
\centering
\small
\caption{Group-mass checks for the MNLI gradient-coverage construction. Mass statistics are computed
within label-conditioned gradient groups.}
\label{tab:app-coverage-group-mass}
\resizebox{\textwidth}{!}{%
\begin{tabular}{lrrrrrrr}
\toprule
Condition & Subset size & Unique ratio & Mass min & Mass median & Mass max & Max/min & Label counts \\
\midrule
Balanced        & 17{,}526 & 1.0 & 0.1249 & 0.1250 & 0.1251 & 1.0   & \(5847/5788/5891\) \\
Long-tail Light & 17{,}526 & 1.0 & 0.0371 & 0.1025 & 0.2779 & 7.5   & \(5847/5788/5891\) \\
Long-tail Heavy & 17{,}526 & 1.0 & 0.0020 & 0.0364 & 0.5510 & 270.5 & \(5847/5788/5891\) \\
\bottomrule
\end{tabular}%
}
\end{table}

\begin{table}[H]
\centering
\scriptsize
\caption{MNLI gradient-coverage results. Deltas are relative to the balanced condition at the same
rank.}
\label{tab:app-gradient-coverage}
\resizebox{\textwidth}{!}{%
\begin{tabular}{rlrrrrr}
\toprule
Rank & Condition & Best val acc & Test acc & Final train CE & Best epoch & \(\Delta\) val / test \\
\midrule
4  & Balanced        & 0.5939 & 0.5954 & 0.9547 & 37 &  0.0000 /  0.0000 \\
4  & Long-tail Light & 0.5868 & 0.5875 & 0.9601 & 48 & -0.0071 / -0.0079 \\
4  & Long-tail Heavy & 0.5853 & 0.5838 & 0.9350 & 48 & -0.0086 / -0.0116 \\
16 & Balanced        & 0.6046 & 0.6033 & 0.9353 & 43 &  0.0000 /  0.0000 \\
16 & Long-tail Light & 0.5955 & 0.5940 & 0.9436 & 42 & -0.0091 / -0.0093 \\
16 & Long-tail Heavy & 0.5961 & 0.5922 & 0.9127 & 44 & -0.0085 / -0.0111 \\
32 & Balanced        & 0.6110 & 0.6084 & 0.9237 & 40 &  0.0000 /  0.0000 \\
32 & Long-tail Light & 0.6076 & 0.6016 & 0.9323 & 47 & -0.0034 / -0.0068 \\
32 & Long-tail Heavy & 0.6026 & 0.5946 & 0.9001 & 46 & -0.0084 / -0.0138 \\
64 & Balanced        & 0.6134 & 0.6133 & 0.9176 & 38 &  0.0000 /  0.0000 \\
64 & Long-tail Light & 0.6145 & 0.6067 & 0.9213 & 47 &  0.0011 / -0.0066 \\
64 & Long-tail Heavy & 0.6103 & 0.6023 & 0.8883 & 45 & -0.0031 / -0.0110 \\
\bottomrule
\end{tabular}%
}
\end{table}

\begin{figure}[H]
    \centering
    \includegraphics[width=0.95\linewidth]{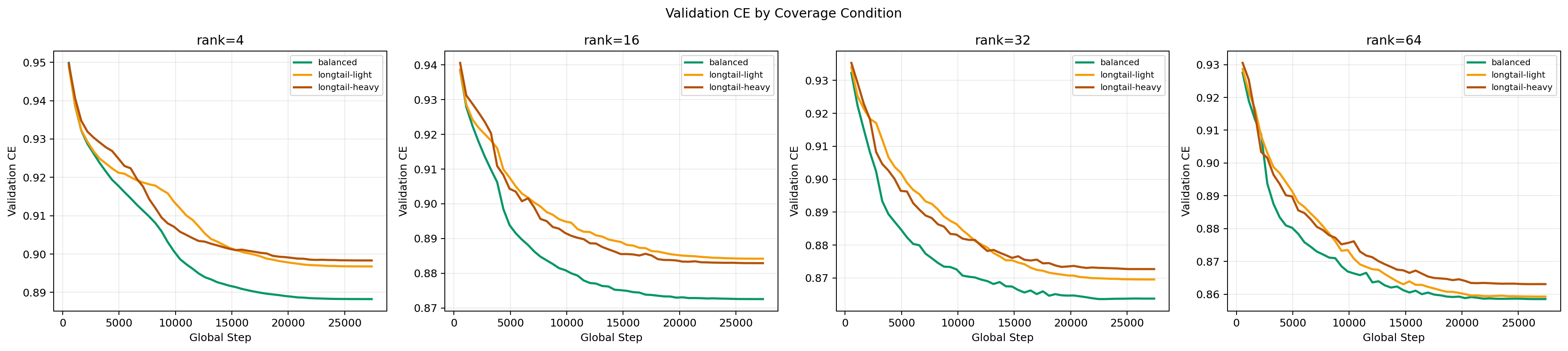}
    \caption{Validation cross-entropy trajectories for the MNLI gradient-coverage intervention. Each
    subplot fixes the LoRA rank and compares balanced group coverage with light and heavy long-tail
    group-mass distributions.}
    \label{fig:app-coverage-val-ce}
\end{figure}

\FloatBarrier

\subsection{Saturation-Band Construction and Results}
\label{sec:app-saturation-band}

The saturation-band experiment uses the same MNLI candidate pool and the same 24
label-conditioned gradient groups as the coverage experiment, but keeps the quota for every
label--gradient group fixed. It therefore changes which examples are selected inside each group,
rather than changing the mass assigned to the groups themselves. For each example, let \(p_{\rm gold}\)
be the model probability of the gold label and let \(p_{\rm alt}\) be the largest non-gold class
probability. We compute the pairwise softmax-curvature score
\[
    p_{\rm pair}=\frac{p_{\rm gold}}{p_{\rm gold}+p_{\rm alt}},
    \qquad
    s=p_{\rm pair}(1-p_{\rm pair}).
\]
This reduces the multiclass prediction to the gold label versus its strongest competitor. In this
pairwise view, if \(m_x(M)\) is the corresponding logit margin, the Gauss--Newton contribution has
the form
\[
G_\nu(M,H)\approx
\mathbb E_\nu\!\left[s\,\bigl(Dm_x(M)[H]\bigr)^2\right].
\]
Thus \(s\) is the softmax-side curvature weight, while \(D m_x(M)[H]\) determines the LoRA
direction receiving that curvature. Within each fixed
label--gradient group, we draw no-replacement examples from one of three quantile bands: easy,
middle, and boundary. Table~\ref{tab:app-saturation-band-checks} records the resulting score
statistics, and Figure~\ref{fig:app-saturation-score-summary} visualizes the separation between the
three selected bands.

\begin{table}[H]
\centering
\small
\caption{Saturation-band checks for the MNLI experiment. The selected subsets all have size
\(17{,}526\), unique ratio \(1.0\), no fallback groups, and fixed label-conditioned gradient-group
quotas.}
\label{tab:app-saturation-band-checks}
\resizebox{\textwidth}{!}{%
\begin{tabular}{lrrrrrr}
\toprule
Condition & Quantile band & Score mean & Score p10 & Score p50 & Score p90 & Subset size \\
\midrule
Easy-band     & \([0.10,0.45]\) & 0.2210 & 0.2044 & 0.2234 & 0.2344 & 17{,}526 \\
Middle-band   & \([0.33,0.67]\) & 0.2389 & 0.2320 & 0.2401 & 0.2445 & 17{,}526 \\
Boundary-band & \([0.62,0.95]\) & 0.2477 & 0.2457 & 0.2481 & 0.2494 & 17{,}526 \\
\bottomrule
\end{tabular}%
}
\end{table}

\begin{figure}[H]
    \centering
    \includegraphics[width=0.8\linewidth]{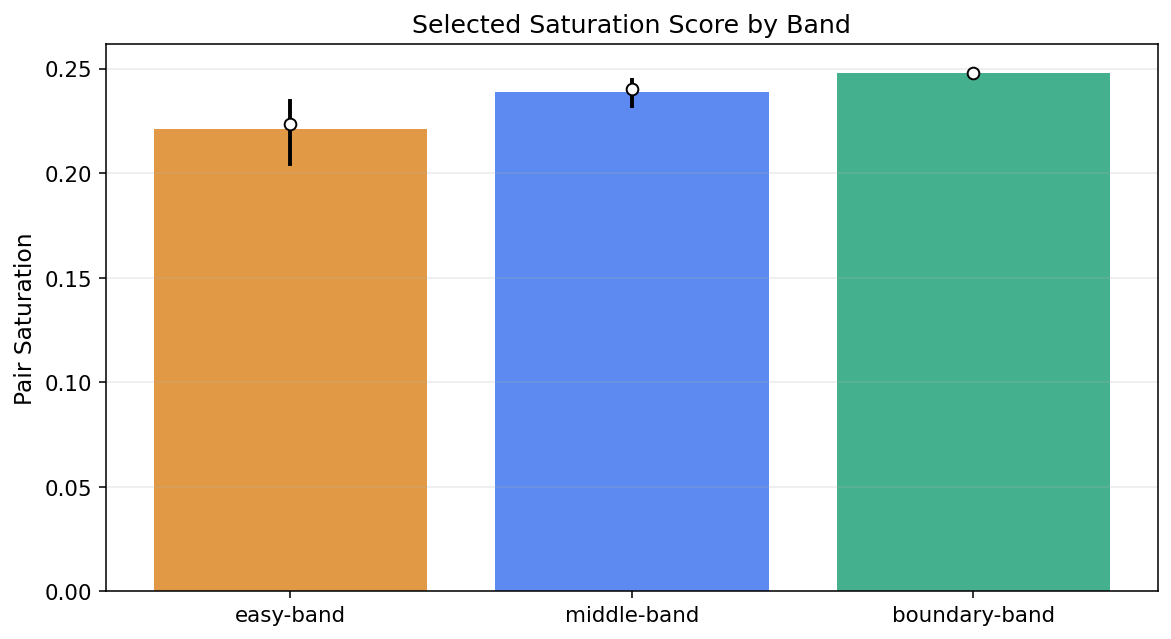}
    \caption{Selected pairwise softmax-curvature scores for the three MNLI saturation bands. The boundary
    band has the highest average score, while the middle band avoids both very easy examples and the
    most boundary-concentrated tail.}
    \label{fig:app-saturation-score-summary}
\end{figure}

\begin{table}[H]
\centering
\scriptsize
\caption{MNLI saturation-band results. All conditions use the same subset size, label budget, and
label-conditioned gradient-group quotas.}
\label{tab:app-saturation-band-results}
\resizebox{\textwidth}{!}{%
\begin{tabular}{rlrrrrr}
\toprule
Rank & Condition & Best val acc & Test acc & Best val CE & Best train CE & Best epoch \\
\midrule
4  & Easy-band     & 0.5853 & 0.5868 & 0.9010 & 0.9509 & 47 \\
4  & Middle-band   & 0.5962 & 0.5984 & 0.8843 & 0.9481 & 42 \\
4  & Boundary-band & 0.5945 & 0.5947 & 0.8860 & 0.9545 & 42 \\
16 & Easy-band     & 0.5967 & 0.5959 & 0.8837 & 0.9326 & 46 \\
16 & Middle-band   & 0.6091 & 0.6074 & 0.8668 & 0.9290 & 39 \\
16 & Boundary-band & 0.6056 & 0.6058 & 0.8676 & 0.9335 & 44 \\
32 & Easy-band     & 0.6032 & 0.6029 & 0.8763 & 0.9251 & 46 \\
32 & Middle-band   & 0.6149 & 0.6168 & 0.8555 & 0.9181 & 49 \\
32 & Boundary-band & 0.6130 & 0.6117 & 0.8590 & 0.9249 & 43 \\
64 & Easy-band     & 0.6183 & 0.6144 & 0.8559 & 0.9094 & 41 \\
64 & Middle-band   & 0.6226 & 0.6186 & 0.8466 & 0.9080 & 49 \\
64 & Boundary-band & 0.6218 & 0.6151 & 0.8478 & 0.9141 & 44 \\
\bottomrule
\end{tabular}%
}
\end{table}

\begin{figure}[H]
    \centering
    \includegraphics[width=0.95\linewidth]{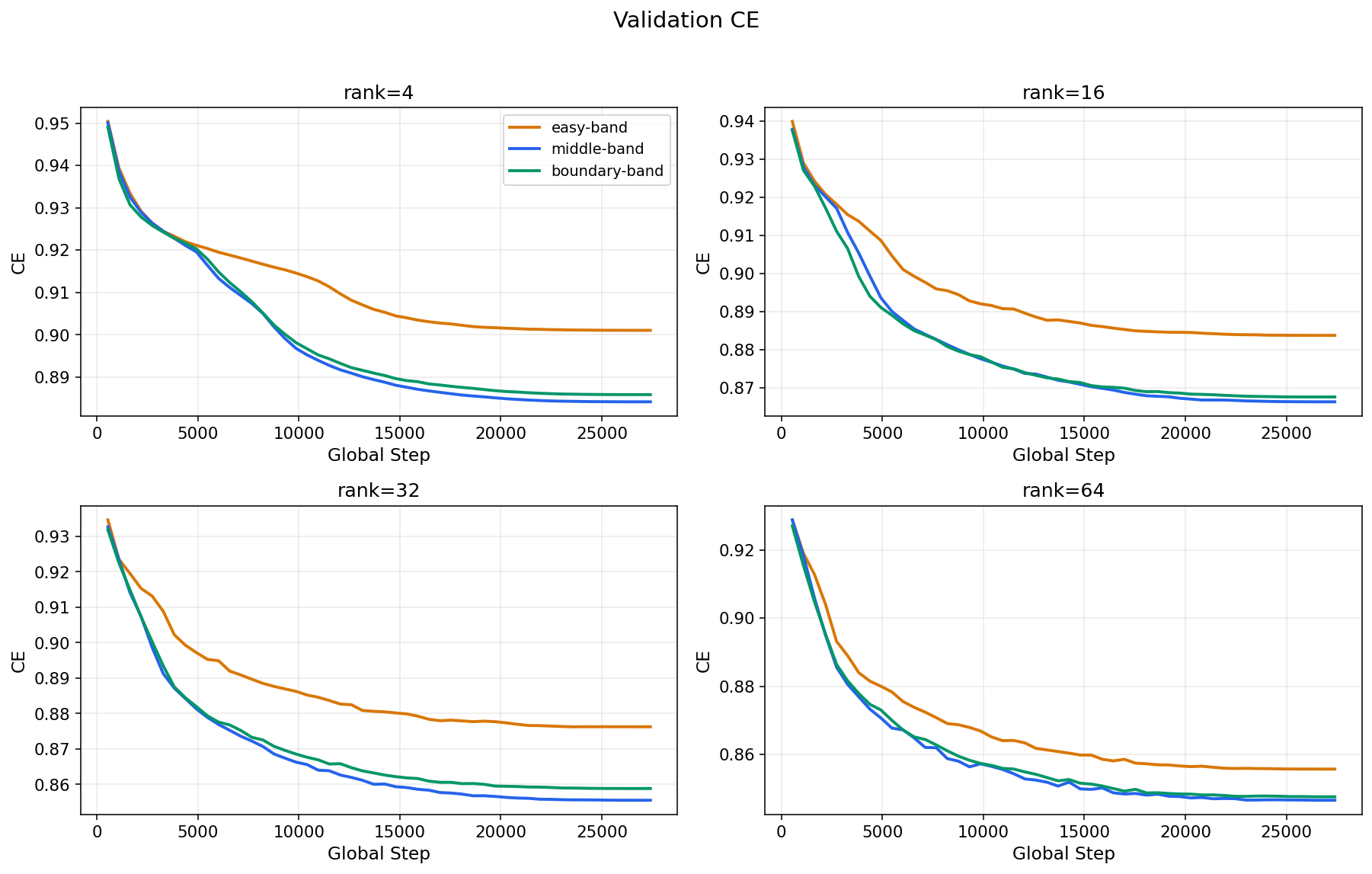}
    \caption{Validation cross-entropy trajectories for the MNLI saturation-band intervention. The
    middle band is consistently strongest, while easy-band examples have lower softmax-side
    curvature leverage under matched gradient-group coverage.}
    \label{fig:app-saturation-val-ce}
\end{figure}

\input{appendix_supplementary_experiments}

\section{Limitations}
\label{app-limit}
Our theoretical analysis focuses on single-matrix LoRA fine-tuning
with cross-entropy loss under low-rank target and LoRA-RIP
assumptions. The resulting rank bounds are sufficient rather
than necessary, leaving room for tighter characterizations of
the relationship between adapter rank and data-induced geometry.
Our empirical evaluation examines this relationship through
controlled classification experiments across datasets and
backbones. Extending the theory to joint multi-layer adaptation
and evaluating the data-selection principles in larger-scale
generative fine-tuning are natural directions for future work.

\FloatBarrier
\clearpage

%% file: appendix_supplementary_experiments.tex
\section{Landscape Diagnostics and Robustness Evaluations}
\label{sec:app-supplementary-experiments}
\setcounter{table}{0}
\setcounter{figure}{0}

This section examines the optimization geometry and stability of LoRA fine-tuning through
rank/checkpoint curvature, ambient matrix-space references, and multi-seed training.
We also study the effects of data composition at a common reference point and along training
trajectories, evaluate coverage across datasets, backbones, and adapter locations, and examine
confidence-band selection on MNLI. The rank-sweep and MNLI data-intervention protocols are described in
Appendix~\ref{sec:app-experiments}.

Together, these evaluations connect rank-dependent optimization gains to data-induced geometry
and the practical effectiveness of landscape-guided selection.

\subsection{Shared diagnostic protocol}
\label{sec:app-diagnostic-protocol}

Unless stated otherwise, the diagnostics use MPNet-base on MNLI, a frozen linear-probe head, a frozen
backbone, and a LoRA update on the layer-10 query projection. Every coverage condition contains
$17{,}526$ unique training examples with label counts $5{,}847/5{,}788/5{,}891$. The shared
data-aligned diagnostic basis has dimension $24$ and is obtained from the eight gradient-sketch
groups within each of the three labels. Each ambient direction represented by this basis has rank at
most $16$. The data-condition diagnostics use $512$ examples per label and three diagnostic-sample
 seeds.

The rank/checkpoint diagnostic uses ranks $4,64,128,256$, optimization seeds $0,1,2$, and start,
middle, and final checkpoints (epochs $0,5,10$). It evaluates projected logit-Jacobian and CE-Hessian
actions in the same $24$-dimensional basis. These actions are estimated by symmetric finite
differences; a step-size check changes the projected matrices by at most $1.9\%$ and does not change
the reported full-Hessian sign counts. Factor regularization and balancing Hessians are excluded
when the ambient update $M$ is studied.

The data-condition diagnostic instead computes exact logit-JVP Gauss--Newton matrices and exact
directional CE-HVPs in the common basis. The zero-update table uses the shared frozen probe point;
the trajectory analysis evaluates matched rank-$16$ checkpoints from the three training seeds. The
reported means and standard deviations are over the three independent diagnostic or training seeds,
as appropriate.

\begin{table}[H]
\centering
\small
\caption{Landscape diagnostic protocols.}
\label{tab:app-diagnostic-protocols}
\resizebox{\textwidth}{!}{%
\begin{tabular}{lllll}
\toprule
Diagnostic & Dataset/model & Adapted matrix & Seeds/checkpoints & Main quantities \\
\midrule
Rank/checkpoint & MNLI/MPNet-base & layer-10 query & $r=4,64,128,256$; 3 seeds; 0/5/10 & projected CE curvature \\
Data condition & MNLI/MPNet-base & layer-10 query & 3 diagnostic seeds & GN condition, stable rank, residual \\
Saturation condition & MNLI/MPNet-base & layer-10 query & 3 diagnostic seeds & softmax weight and GN spectrum \\
Ambient reference & MNLI, CIFAR-100 & layer-10 query & two LoRA warm starts/task & $F_\lambda$ and spectral ranks \\
Multi-initialization & MNLI, CIFAR-100 & layer-10 query & 3 runs/rank & objective and held-out stability \\
\bottomrule
\end{tabular}%
}
\end{table}

\paragraph{Theoretical constants and empirical diagnostics.}
\label{sec:app-theoretical-quantities}
We do not directly compute exact values of
\(\delta_k^{\mathrm{LoRA}}\), \(r_\star\),
\(c_{\nu,k}\), \(C_{\nu,k}\),
\(\omega_{\nu,k}^{\pm}\), or \(\rho_{\nu,k}\).
These quantities parameterize structural guarantees, rather than
serving as per-run statistics. In classical matrix sensing and
low-rank optimization, RIP relates measurement geometry to recovery
and landscape guarantees~\citep{recht2010guaranteed,ge2017nospurious,
zhangGeneralLowrankMatrix2021,zhangSharpGlobalGuarantees2025}.
Its primary value in our framework is likewise the relationship it
establishes between data geometry and rank requirements, not an
isolated numerical constant for a particular run.

Exact certification is already computationally difficult for a fixed
linear compressed-sensing matrix: the condition controls singular
values across all supports of a prescribed size, and computing optimal
RIP constants is NP-hard in general~\citep{tillmann2014complexity}.
Hardness also extends to approximate RIP
certification~\citep{weed2018approximately}.
For LoRA, the defining curvature bounds require uniform control over
the continuous family of target-to-update segments and all admissible
low-rank directions. Evaluating selected checkpoints in a diagnostic
subspace does not compute these extrema. Likewise, determining
\(r_\star\) requires identifying a global minimizer of the nonlinear
regularized matrix-space objective \(F_\lambda\) and certifying its rank.
Local optimization does not establish global optimality or uniqueness;
\(r_\star\) refers to the low-rank target specified in our assumption.
The matrix-space references and threshold or energy ranks reported
below describe numerical solutions, not certified global targets.

The informative consequence of the analysis is how these quantities
interact: The sufficient condition permits larger LoRA-RIP values at larger ranks, while coverage, softmax weighting, and nonlinear residual curvature
enter the data-side certificate in explicit ways.
Our experiments examine these mechanisms through controlled rank and
data interventions, attained objectives, projected GN and CE-Hessian
spectra, and multi-seed performance comparisons. The projected spectra
characterize local geometry, rather than the theorem's uniform
LoRA-RIP constant. Lower attained objectives at larger ranks and the
changes in geometry and performance under coverage and confidence-band
selection provide complementary, qualitative evidence for the
rank--data relationship developed by the theory.

\subsection{Rank and checkpoint curvature}
\label{sec:app-landscape-diagnostics}

Table~\ref{tab:app-rank-checkpoint-full} reports the full checkpoint measurements. The initial
state is shared across ranks because the LoRA initialization has $B=0$ and therefore $M=BA=0$.

\begin{table}[H]
\centering
\scriptsize
\caption{MNLI projected curvature over training checkpoints. Values are mean $\pm$ standard
deviation over three training seeds. GN floor and ceiling are shown in absolute units.}
\label{tab:app-rank-checkpoint-full}
\resizebox{\textwidth}{!}{%
\begin{tabular}{rrcrrrrrr}
\toprule
Rank & Checkpoint & CE & 99\% energy rank & GN floor & GN ceiling & GN condition & Full-H min eig. & Neg. eigs. \\
\midrule
4 & Start  & 0.959$\pm$0.000 & 0.0$\pm$0.0 & $4.911\times10^{-7}$ & $1.606\times10^{-5}$ & 32.70$\pm$0.00 & $-9.012\times10^{-6}$ & 8.00$\pm$0.00 \\
4 & Middle & 0.843$\pm$0.001 & 2.0$\pm$0.0 & $4.136\times10^{-7}$ & $1.764\times10^{-5}$ & 42.62$\pm$2.07 & $-1.914\times10^{-7}$ & 1.33$\pm$0.58 \\
4 & Final  & 0.839$\pm$0.001 & 2.0$\pm$0.0 & $3.967\times10^{-7}$ & $1.754\times10^{-5}$ & 44.22$\pm$2.30 & $-4.915\times10^{-7}$ & 1.67$\pm$0.58 \\
64 & Start  & 0.959$\pm$0.000 & 0.0$\pm$0.0 & $4.911\times10^{-7}$ & $1.606\times10^{-5}$ & 32.70$\pm$0.00 & $-9.012\times10^{-6}$ & 8.00$\pm$0.00 \\
64 & Middle & 0.800$\pm$0.002 & 4.33$\pm$0.58 & $3.445\times10^{-7}$ & $1.966\times10^{-5}$ & 56.83$\pm$9.85 & $3.596\times10^{-7}$ & 0.00$\pm$0.00 \\
64 & Final  & 0.793$\pm$0.002 & 5.0$\pm$0.0 & $3.448\times10^{-7}$ & $1.874\times10^{-5}$ & 54.35$\pm$10.00 & $4.796\times10^{-7}$ & 0.00$\pm$0.00 \\
128 & Start  & 0.959$\pm$0.000 & 0.0$\pm$0.0 & $4.911\times10^{-7}$ & $1.606\times10^{-5}$ & 32.70$\pm$0.00 & $-9.012\times10^{-6}$ & 8.00$\pm$0.00 \\
128 & Middle & 0.789$\pm$0.002 & 6.0$\pm$0.0 & $3.344\times10^{-7}$ & $1.749\times10^{-5}$ & 52.33$\pm$5.05 & $3.085\times10^{-7}$ & 0.00$\pm$0.00 \\
128 & Final  & 0.783$\pm$0.002 & 6.67$\pm$0.58 & $3.358\times10^{-7}$ & $1.808\times10^{-5}$ & 53.80$\pm$3.08 & $2.929\times10^{-7}$ & 0.00$\pm$0.00 \\
256 & Start  & 0.959$\pm$0.000 & 0.0$\pm$0.0 & $4.911\times10^{-7}$ & $1.606\times10^{-5}$ & 32.70$\pm$0.00 & $-9.012\times10^{-6}$ & 8.00$\pm$0.00 \\
256 & Middle & 0.779$\pm$0.001 & 7.0$\pm$0.0 & $3.432\times10^{-7}$ & $1.775\times10^{-5}$ & 51.89$\pm$6.16 & $2.884\times10^{-7}$ & 0.00$\pm$0.00 \\
256 & Final  & 0.771$\pm$0.001 & 7.67$\pm$0.58 & $3.586\times10^{-7}$ & $1.952\times10^{-5}$ & 54.52$\pm$6.76 & $2.360\times10^{-7}$ & 0.00$\pm$0.00 \\
\bottomrule
\end{tabular}%
}
\end{table}

Ranks $64$--$256$ reach positive CE-Hessian curvature throughout the shared projected subspace
by the middle checkpoint and retain it at the final checkpoint, while rank $4$ retains weak negative
curvature. Final CE also decreases from $0.839$ at rank $4$ to $0.771$ at rank $256$.
Together, these trends link the gains from larger search ranks to lower-loss solutions with more
favorable second-order geometry, consistent with the role of over-parameterization in our landscape
analysis. Figure~\ref{fig:app-rank-checkpoint-curvature} visualizes these trends.

\begin{figure}[!htbp]
    \centering
    \includegraphics[width=0.95\linewidth]{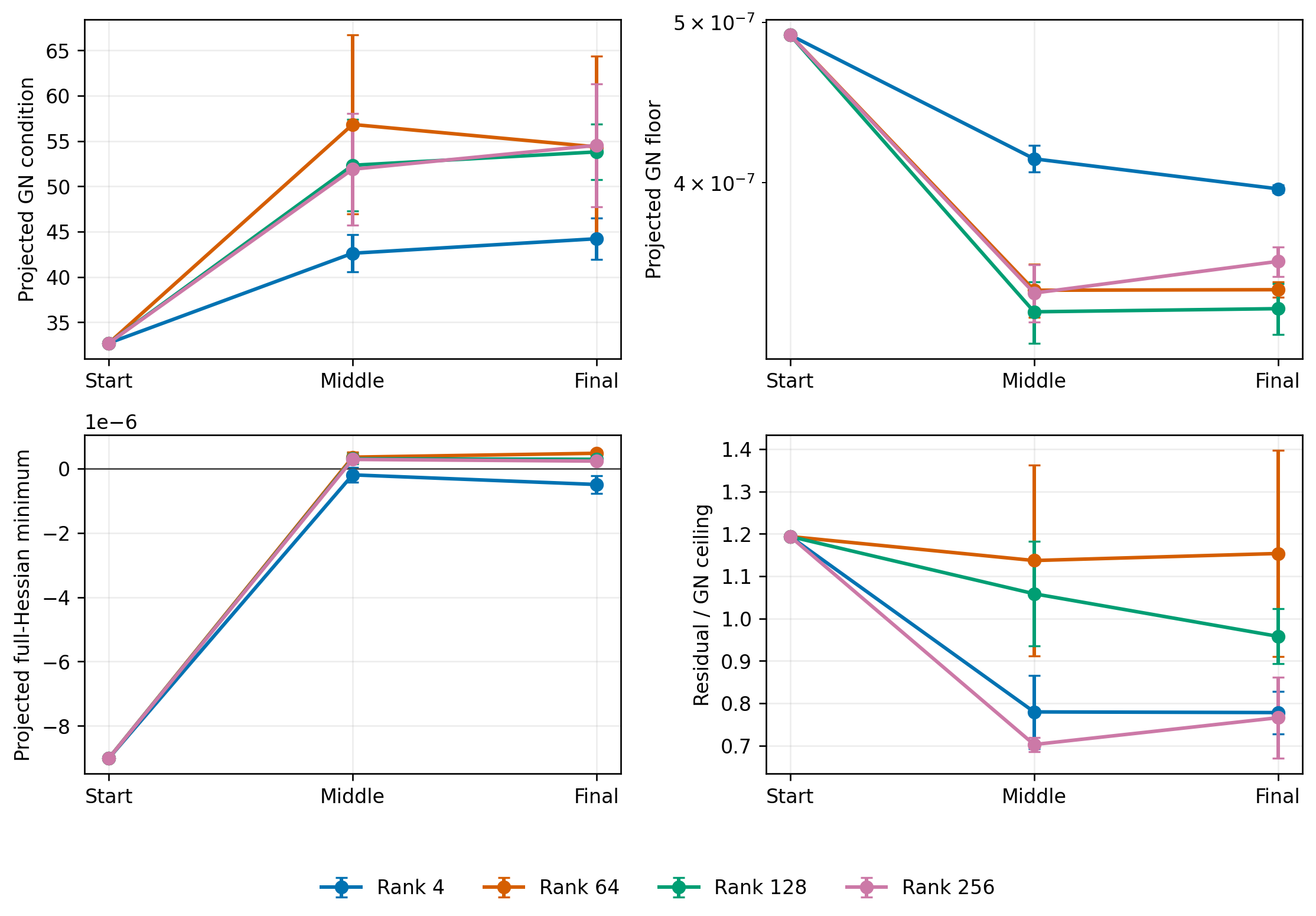}
    \caption{Projected GN and CE-Hessian diagnostics over rank and checkpoint.}
    \label{fig:app-rank-checkpoint-curvature}
\end{figure}

\FloatBarrier

\subsection{Ambient matrix-space reference and effective rank}
\label{sec:app-ambient-reference}

We obtain numerical matrix-space references by directly optimizing
$F_\lambda(M)=L_{\mathrm{train}}(M)+\lambda\|M\|_*$ with $\lambda=10^{-4}$. We use full-batch
gradients, FP64 singular-value thresholding, monotone backtracking, and two existing LoRA solutions
as warm starts for each task. The threshold rank counts singular values above $10^{-4}\sigma_1$;
energy ranks count the number of directions needed to capture the indicated fraction of squared
singular-value energy.

\begin{table}[H]
\centering
\scriptsize
\caption{Ambient matrix-space proximal diagnostics from two LoRA warm starts. Energy ranks
summarize spectral concentration, and the relative mapping is the normalized proximal-gradient
mapping used to monitor the optimization residual.}
\label{tab:app-ambient-reference}
\resizebox{\textwidth}{!}{%
\begin{tabular}{lllrrrrr}
\toprule
Task & Warm start & Final $F_\lambda$ & Threshold rank & $r_{90}$ & $r_{95}$ & $r_{99}$ & Rel. prox. map. \\
\midrule
MNLI & LoRA rank 256 & 0.790509 & 44  & 4  & 6  & 8  & $5.00\times10^{-5}$ \\
MNLI & LoRA rank 64  & 0.809354 & 35  & 3  & 4  & 5  & $7.26\times10^{-5}$ \\
CIFAR-100 & LoRA rank 256 & 0.061640 & 406 & 24 & 33 & 62 & $4.30\times10^{-4}$ \\
CIFAR-100 & LoRA rank 64  & 0.085553 & 457 & 9  & 13 & 25 & $2.95\times10^{-3}$ \\
\bottomrule
\end{tabular}%
}
\end{table}

\begin{table}[H]
\centering
\small
\caption{Best observed LoRA objective gaps to the corresponding ambient reference.}
\label{tab:app-ambient-gaps}
\begin{tabular}{lrrrr}
\toprule
Task & Rank 4 & Rank 64  & Rank 256 \\
\midrule
MNLI & 0.06450 & 0.02130 & 0.00261 \\
CIFAR-100 & 0.04666 & 0.02545 & 0.00038 \\
\bottomrule
\end{tabular}
\end{table}

Table~\ref{tab:app-ambient-reference} shows concentrated update energy: $99\%$ lies in
$5$--$8$ directions on MNLI and $25$--$62$ on CIFAR-100. Table~\ref{tab:app-ambient-gaps}
shows that increasing the search rank from $4$ to $256$ reduces the best observed objective gap
to the matrix-space reference from $0.06450$ to $0.00261$ on MNLI and from $0.04666$ to
$0.00038$ on CIFAR-100. Together, the spectral summaries and shrinking objective gaps connect
compact update structure with the optimization benefits of a larger search rank.

\FloatBarrier

\subsection{Multi-seed and subset-construction robustness}
\label{sec:app-multiinit}

We repeat representative MNLI and CIFAR-100 rank experiments with three independent runs per
rank. The runs jointly vary initialization, minibatch order, and stochastic training operations to
measure the stability of the rank-dependent optimization and performance trends.

\begin{table}[H]
\centering
\small
\caption{Three-run multi-initialization diagnostic. Accuracy is reported as a fraction; values are
mean $\pm$ standard deviation.}
\label{tab:app-multiinit}
\begin{tabular}{llrr}
\toprule
Task & Rank & Best validation accuracy & Final factor objective \\
\midrule
MNLI & 4   & $0.6199\pm0.0009$ & $0.855593\pm0.000605$ \\
MNLI & 64  & $0.6556\pm0.0007$ & $0.812421\pm0.000531$ \\
MNLI & 128 & $0.6626\pm0.0014$ & $0.803074\pm0.001195$ \\
MNLI & 256 & $0.6695\pm0.0009$ & $0.793288\pm0.000174$ \\
CIFAR-100 & 4   & $0.7389\pm0.0009$ & $0.108684\pm0.000127$ \\
CIFAR-100 & 64  & $0.7407\pm0.0005$ & $0.087717\pm0.000223$ \\
CIFAR-100 & 256 & $0.7429\pm0.0010$ & $0.062622\pm0.000067$ \\
\bottomrule
\end{tabular}
\end{table}

Across both tasks, larger ranks consistently achieve lower final factor objectives and higher mean
validation accuracy. Within each rank, validation-accuracy standard deviations remain between
$0.05$ and $0.14$ percentage points. The rank trends are therefore reproducible across training
randomness, reinforcing the empirical connection between search rank and attainable objective
quality. Figure~\ref{fig:app-multiinit} shows the corresponding multi-run diagnostics.

\paragraph{Coverage stability across training and subset-construction seeds.}
The Balanced-versus-Heavy comparison on MNLI also evaluates sensitivity to training randomness
and subset construction. Across six ranks from $4$ to $256$ with three training seeds,
Balanced has higher validation accuracy at all six ranks. With independently reconstructed
subsets, Heavy falls below Balanced in all $12$ paired runs. This consistency across both sources
of randomness shows that preserving gradient-group coverage yields reliable improvements under
a fixed data budget.

\begin{figure}[!htbp]
    \centering
    \begin{minipage}{0.98\linewidth}
        \centering
        \includegraphics[width=\linewidth]{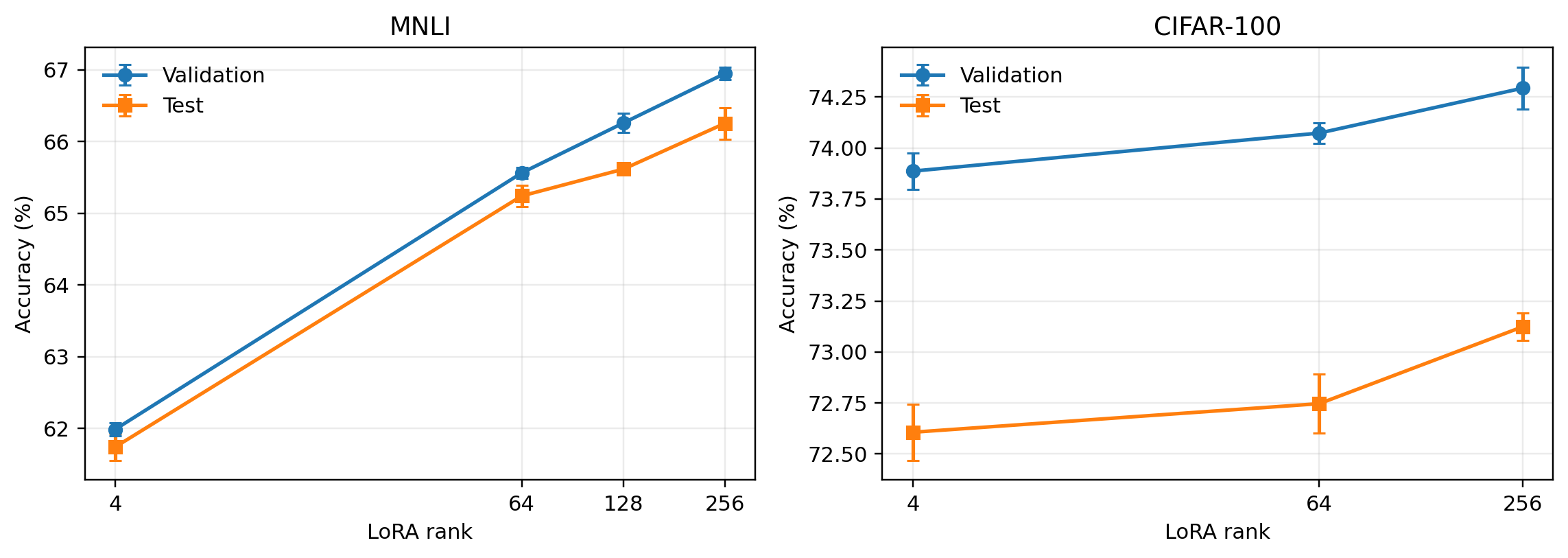}
    \end{minipage}
    \par
    \begin{minipage}{0.98\linewidth}
        \centering
        \includegraphics[width=\linewidth]{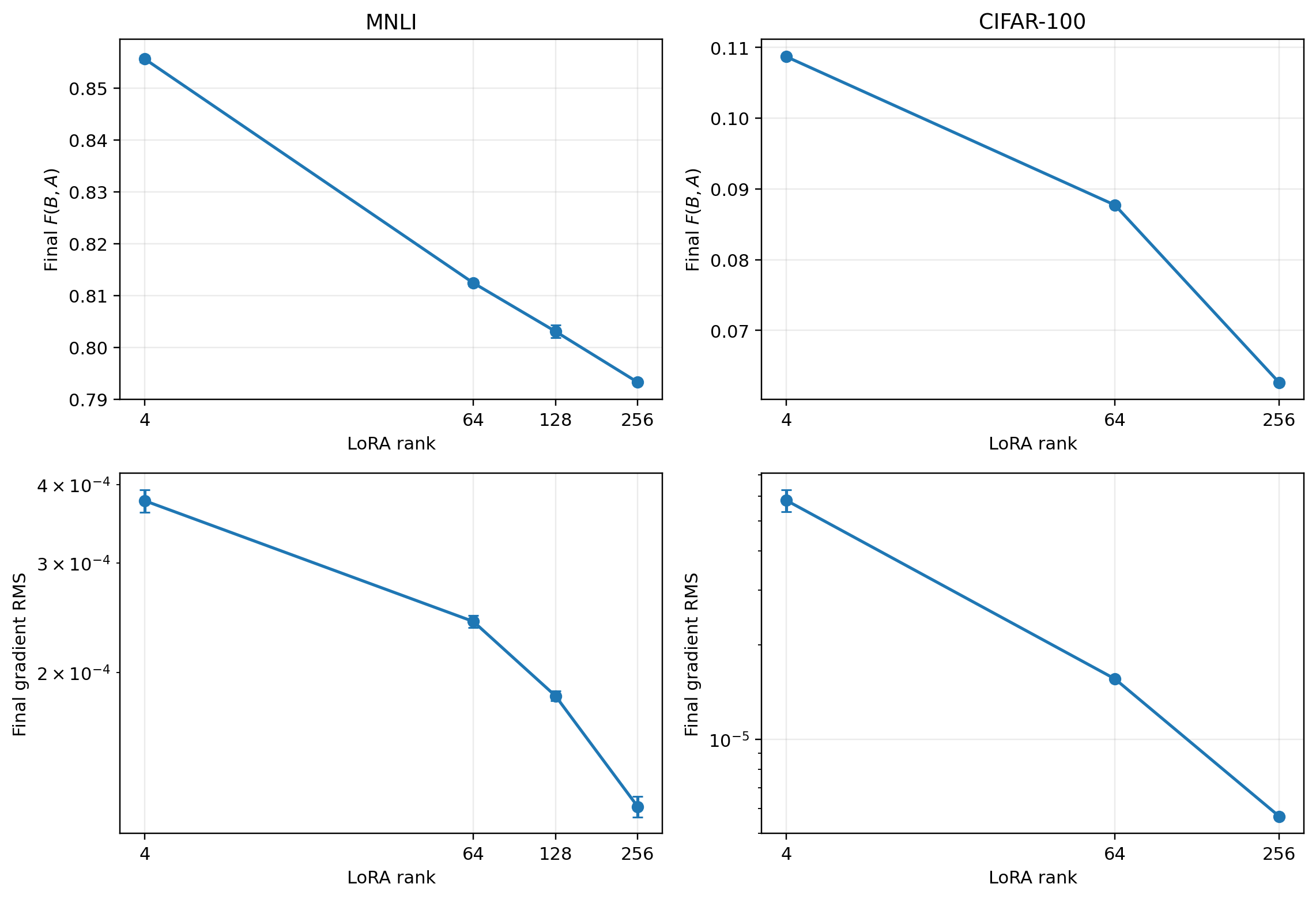}
    \end{minipage}
    \caption{Multi-initialization performance and objective diagnostics.}
    \label{fig:app-multiinit}
\end{figure}

\FloatBarrier

\subsection{Projected data-condition diagnostics}
\label{sec:app-data-diagnostics}

The coverage diagnostic compares the balanced, light long-tail, and heavy long-tail subsets under
the same label budget and subset size. The projected GN condition number is the ratio of the largest
to smallest eigenvalue in the shared span. Stable rank is computed from the projected GN spectrum,
and the residual norm is the operator norm of the nonlinear CE-Hessian residual in that span.

\begin{table}[H]
\centering
\small
\caption{Exact projected data-condition diagnostics at the shared zero-update model. Means and
standard deviations are over three diagnostic seeds; training-trajectory results are reported
separately in Table~\ref{tab:app-trajectory-diagnostics}.}
\label{tab:app-data-condition}
\resizebox{\textwidth}{!}{%
\begin{tabular}{lrrrrr}
\toprule
Condition & GN floor ($\times10^{-7}$) & GN ceiling ($\times10^{-5}$) & GN condition & Stable rank & Residual ($\times10^{-5}$) \\
\midrule
Balanced & $4.66\pm0.14$ & $1.59\pm0.01$ & $34.19\pm1.17$ & $4.81\pm0.02$ & $2.11\pm0.16$ \\
Light long-tail & $4.88\pm0.14$ & $1.73\pm0.02$ & $35.37\pm0.97$ & $4.47\pm0.04$ & $3.27\pm0.17$ \\
Heavy long-tail & $4.93\pm0.10$ & $1.90\pm0.04$ & $38.60\pm1.32$ & $4.04\pm0.05$ & $4.19\pm0.43$ \\
\bottomrule
\end{tabular}%
}
\end{table}

Balanced yields the lowest projected GN condition number and the highest stable rank.
With GN floors on comparable scales, the lower ceiling identifies reduced concentration in dominant directions as the main source of better
conditioning. Balanced also has a smaller nonlinear residual. These results support the coverage mechanism:
spreading the data budget across response modes produces more even directional support and a
more favorable curvature profile.

Figure~\ref{fig:app-data-condition} displays the corresponding GN condition numbers and curvature spectra.

\begin{figure}[!htbp]
    \centering
    \begin{minipage}{0.98\linewidth}
        \centering
        \includegraphics[width=\linewidth]{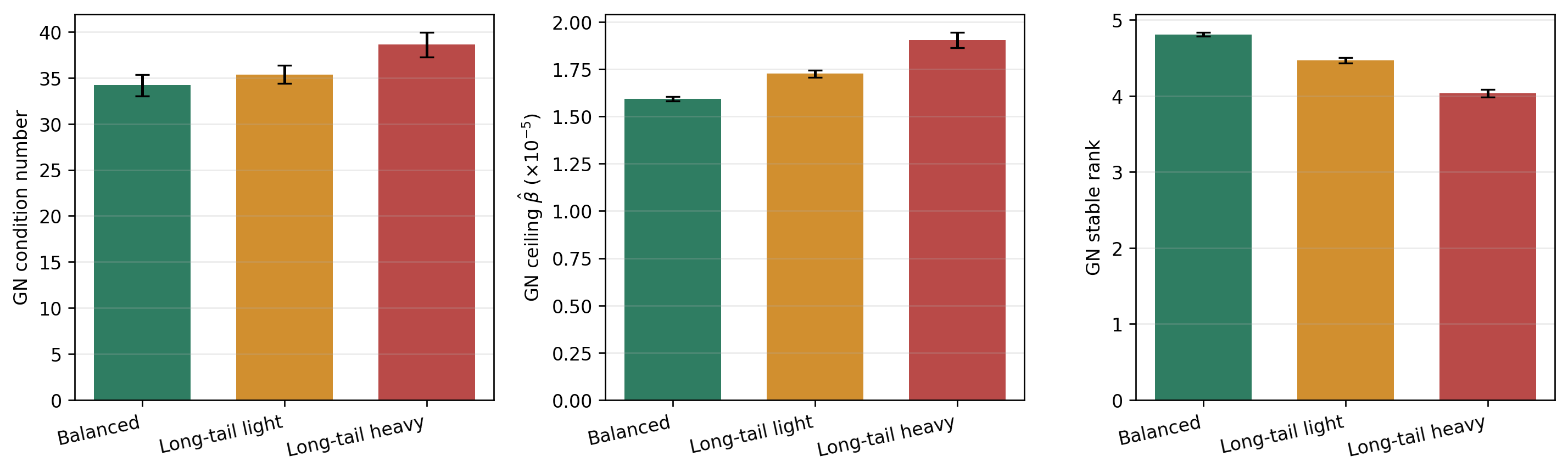}
    \end{minipage}
    \par
    \begin{minipage}{0.98\linewidth}
        \centering
        \includegraphics[width=\linewidth]{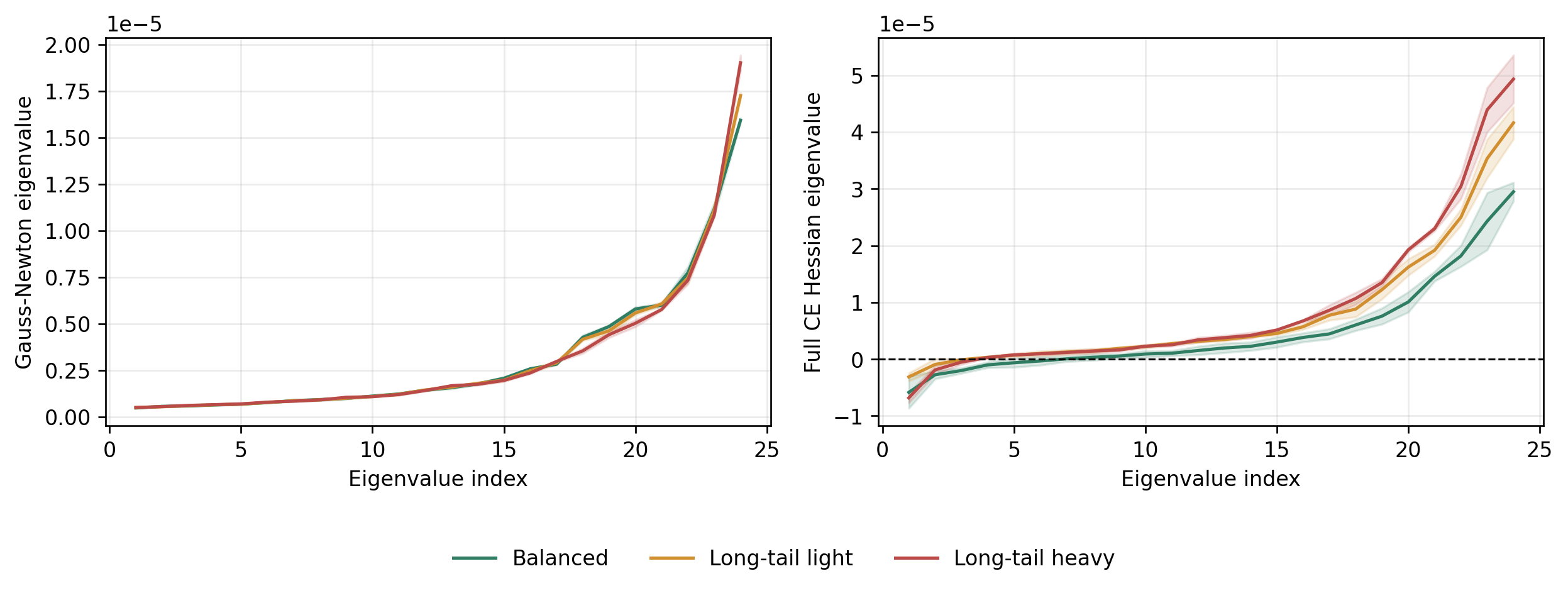}
    \end{minipage}
    \caption{Projected Gauss--Newton condition numbers and restricted curvature spectra for the
    gradient-coverage subsets.}
    \label{fig:app-data-condition}
\end{figure}

\FloatBarrier

\subsection{CE-Hessian geometry along training}
\label{sec:app-trajectory-diagnostics}

We compare the projected CE-Hessian geometry reached by Balanced, Light, and Heavy training
at start, middle, and final checkpoints, using the same $1{,}536$-example diagnostic set and
$24$-dimensional basis. Figure~\ref{fig:app-trajectory-diagnostics} focuses on the minimum
CE-Hessian eigenvalue and the number of negative eigenvalues in this shared subspace.
Table~\ref{tab:app-trajectory-diagnostics} records the complete checkpoint measurements.

\begin{table}[H]
\centering
\scriptsize
\caption{Complete training-trajectory measurements. Reported uncertainties are standard deviations
over three training seeds; start is one shared state. The accompanying figure displays the two
projected CE-Hessian sign diagnostics.}
\label{tab:app-trajectory-diagnostics}
\resizebox{\textwidth}{!}{%
\begin{tabular}{llrrrr}
\toprule
Condition & Checkpoint & GN condition & Full-H min eig. & Negative eigs. & Residual/GN ceiling \\
\midrule
Balanced & Start  & $32.70\pm0.00$ & $-9.012\mathrm{e}{-6}$ & $8.00\pm0.00$ & $1.194\pm0.000$ \\
Balanced & Middle & $57.54\pm1.88$ & $1.816\mathrm{e}{-7}$ & $0.00\pm0.00$ & $0.755\pm0.059$ \\
Balanced & Final  & $59.60\pm1.52$ & $2.096\mathrm{e}{-7}$ & $0.00\pm0.00$ & $0.756\pm0.020$ \\
Light & Middle & $41.92\pm3.85$ & $-1.437\mathrm{e}{-6}$ & $3.67\pm0.58$ & $0.721\pm0.217$ \\
Light & Final  & $40.16\pm2.69$ & $-1.440\mathrm{e}{-6}$ & $3.00\pm1.00$ & $0.799\pm0.178$ \\
Heavy & Middle & $44.44\pm0.58$ & $-2.543\mathrm{e}{-6}$ & $5.00\pm0.00$ & $0.594\pm0.033$ \\
Heavy & Final  & $47.01\pm1.56$ & $-2.375\mathrm{e}{-6}$ & $4.33\pm0.58$ & $0.556\pm0.024$ \\
\bottomrule
\end{tabular}%
}
\end{table}

Balanced reaches a positive mean minimum CE-Hessian eigenvalue by the middle checkpoint,
and all three training runs have zero detected negative eigenvalues at both the middle and final
checkpoints. At the final checkpoint, Light and Heavy retain an average of $3.00$ and $4.33$
negative eigenvalues, respectively. This consistent separation links balanced gradient-group
coverage to training states with more favorable local second-order geometry, complementing
the common-reference coverage analysis.

\begin{figure}[!htbp]
    \centering
    \begin{minipage}[t]{0.49\linewidth}
        \centering
        \includegraphics[width=\linewidth,trim={364.9267bp 289.3230bp 0bp 0bp},clip]{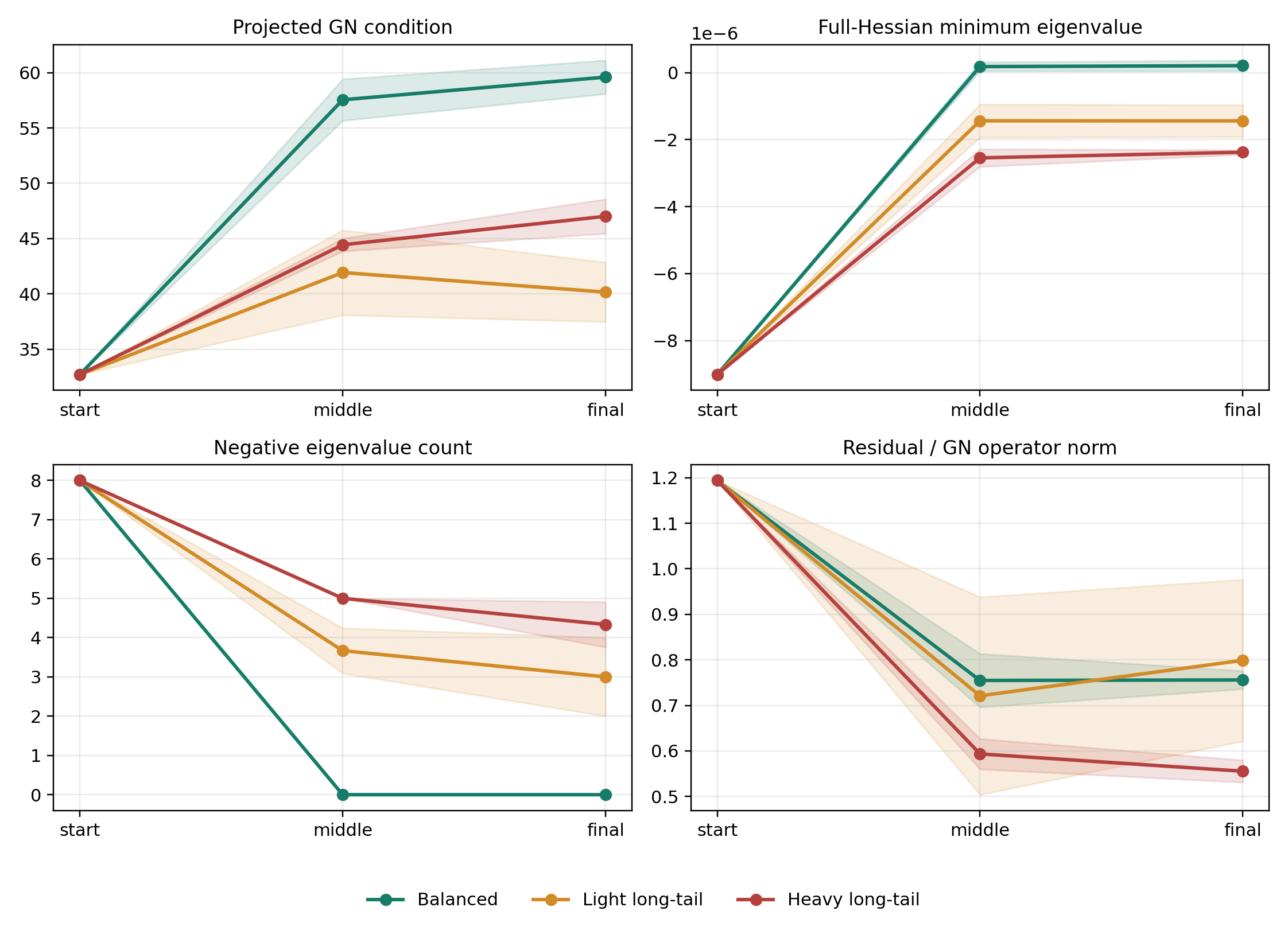}
    \end{minipage}\hfill
    \begin{minipage}[t]{0.49\linewidth}
        \centering
        \includegraphics[width=\linewidth,trim={0bp 49.7479bp 369.5087bp 239.5752bp},clip]{figures/experiments/diagnostics/mnli_data_condition_hessian_trajectory.png}
    \end{minipage}
    \par\smallskip
    \includegraphics[width=0.66\linewidth,trim={201.2824bp 5.8912bp 53.6753bp 497.4785bp},clip]{figures/experiments/diagnostics/mnli_data_condition_hessian_trajectory.png}
    \caption{Projected CE-Hessian geometry along MNLI training trajectories. Left: minimum
    eigenvalue. Right: number of negative eigenvalues. All conditions use the same diagnostic
    set and projection basis; curves and shaded bands summarize three training seeds.}
    \label{fig:app-trajectory-diagnostics}
\end{figure}

\FloatBarrier

\subsection{Confidence-band diagnostics and multi-seed evaluation}
\label{sec:app-saturation-diagnostics}

The MNLI confidence-band comparison fixes subset size, label counts, and gradient-group quotas.
The score is \(s=p_{\rm pair}(1-p_{\rm pair})\), where
\(p_{\rm pair}=p_{\rm gold}/(p_{\rm gold}+p_{\rm alt})\). All bands are evaluated at the same
zero-update model in the common 24-dimensional span. This tests whether changing average
pairwise confidence mainly changes curvature scale or also changes projected conditioning.

\begin{table}[H]
\centering
\small
\caption{Confidence-band curvature summaries at the shared reference model. GN statistics are
means $\pm$ standard deviations over three diagnostic seeds.}
\label{tab:app-saturation-diagnostics}
\begin{tabular}{lrrr}
\toprule
Band & Pairwise score $s$ & GN trace ($\times10^{-5}$) & GN condition \\
\midrule
Easy & 0.2210 & $7.43\pm0.03$ & $32.40\pm0.55$ \\
Middle & 0.2389 & $7.71\pm0.10$ & $33.16\pm1.59$ \\
Boundary & 0.2477 & $7.73\pm0.07$ & $32.59\pm1.21$ \\
\bottomrule
\end{tabular}
\end{table}

With gradient-group quotas fixed, Middle and Boundary provide higher GN trace than Easy in all
three diagnostic samples, while mean projected condition numbers remain close ($32.40$--$33.16$).
This highlights the role of softmax weighting in the amount of curvature supplied to covered
directions, rather than a clear conditioning advantage for a particular band. The training
comparisons below separately evaluate how these confidence-band choices affect fine-tuning performance.

\begin{table}[H]
\centering
\small
\caption{Three-seed MNLI saturation-band training results. Entries are validation accuracy in percent,
reported as mean $\pm$ standard deviation.}
\label{tab:app-saturation-training}
\begin{tabular}{rrrr}
\toprule
Rank & Easy & Middle & Boundary \\
\midrule
4  & $58.83\pm0.47$ & $59.58\pm0.04$ & $59.43\pm0.04$ \\
16 & $60.10\pm0.60$ & $60.89\pm0.07$ & $60.59\pm0.02$ \\
32 & $60.95\pm0.55$ & $61.36\pm0.12$ & $61.36\pm0.06$ \\
64 & $61.65\pm0.33$ & $62.04\pm0.19$ & $62.13\pm0.05$ \\
\bottomrule
\end{tabular}
\end{table}

Table~\ref{tab:app-saturation-training} shows that both Middle and Boundary improve mean
validation accuracy over Easy at all four ranks, with smaller run-to-run variability.
Their relative ordering varies: Middle leads at ranks $4$ and $16$, the reported means tie at
rank $32$, and Boundary has the highest mean at rank $64$. Together with the reference-point
diagnostics, these results support confidence-aware selection that avoids concentrating the
budget on highly saturated examples while maintaining directional coverage, rather than a
universal preference for the middle band.

\FloatBarrier

\subsection{Cross-Backbone and Cross-Dataset Gradient Coverage}
\label{sec:app-coverage-extensions}

We extend the gradient-coverage intervention to a second backbone on MNLI and to QNLI with
MPNet-base. Both experiments compare Balanced, Light long-tail, and Heavy long-tail sampling.
Balanced assigns approximately equal example counts to the eight gradient groups within each
label. At each rank, the three conditions share the same subset size, label budget, target module,
and training settings; the intervention changes how the data budget is allocated across gradient
groups. Both experiments use ranks $r\in\{16,64\}$, seeds $0,1,2$, 50 training epochs, and batch
size 32. Validation and test accuracies are evaluated at the checkpoint with the lowest validation
cross-entropy in each run. The tables report three-seed summaries in percentage units.

\paragraph{DeBERTa-v3-base on MNLI.}
We use \texttt{microsoft/deberta-v3-base}, with LoRA applied only to the layer-10 query projection
and the existing linear-probe head kept fixed. Each selected subset contains $17{,}526$ examples,
with label counts $(5{,}847,5{,}788,5{,}891)$. Normalized gradient sketches are clustered into
eight groups per label, giving 24 label-conditioned groups. The Light and Heavy long-tail
conditions use group-mass decay factors $0.60$ and $0.45$, respectively. The maximum sequence
length is 128.

Table~\ref{tab:app-deberta-mnli-coverage} shows that Balanced achieves the highest validation and
test accuracy at both ranks, followed by Light and then Heavy long-tail. Relative to Heavy
long-tail, Balanced improves test accuracy by $2.18$ percentage points at $r=16$ and $1.62$
points at $r=64$. The Balanced-versus-Heavy advantage across MPNet and DeBERTa supports
gradient-response coverage as a model-aware criterion for selecting fine-tuning data.

\begin{table}[H]
\centering
\small
\caption{DeBERTa-v3-base on MNLI: gradient-coverage results over seeds $0,1,2$ at the best
validation-CE checkpoint of each run. Accuracy is in percent; bold marks the best condition at
each rank.}
\label{tab:app-deberta-mnli-coverage}
\begin{tabular}{rlrr}
\toprule
Rank & Condition & Validation accuracy (\%) & Test accuracy (\%) \\
\midrule
16 & Balanced & $\mathbf{44.54\pm0.06}$ & $\mathbf{45.04\pm0.06}$ \\
16 & Light long-tail & $43.42\pm0.07$ & $43.76\pm0.11$ \\
16 & Heavy long-tail & $42.25\pm0.03$ & $42.86\pm0.03$ \\
\midrule
64 & Balanced & $\mathbf{45.12\pm0.09}$ & $\mathbf{45.16\pm0.18}$ \\
64 & Light long-tail & $44.10\pm0.16$ & $44.38\pm0.03$ \\
64 & Heavy long-tail & $42.92\pm0.11$ & $43.54\pm0.16$ \\
\bottomrule
\end{tabular}
\end{table}

\paragraph{MPNet-base on QNLI.}
We use \texttt{microsoft/mpnet-base} on QNLI, targeting the layer-10 attention query projection.
Each condition contains $5{,}125$ examples with label counts $(2{,}534,2{,}591)$.
Normalized gradient/margin-Jacobian sketches form eight groups per label, giving 16
label-conditioned groups. The model, adapted layer, total subset size, label budget, and training
settings are held fixed across the three sampling conditions at each rank.

As shown in Table~\ref{tab:app-qnli-coverage-extension}, Balanced again outperforms Light
long-tail, which outperforms Heavy long-tail, on both validation and test accuracy. Balanced
improves test accuracy over Heavy long-tail by $5.50$ percentage points at $r=16$ and $6.79$
points at $r=64$. The sustained gains at both ranks show that allocating a fixed sample budget
across complementary response groups remains valuable as the adapter capacity increases.

\begin{table}[H]
\centering
\small
\caption{MPNet-base on QNLI: three-condition gradient-coverage experiment with $5{,}125$ examples.
Results summarize seeds $0,1,2$ at each run's best validation-CE checkpoint. Accuracy is in percent;
bold marks the best condition at each rank.}
\label{tab:app-qnli-coverage-extension}
\begin{tabular}{rlrr}
\toprule
Rank & Condition & Validation accuracy (\%) & Test accuracy (\%) \\
\midrule
16 & Balanced & $\mathbf{71.73\pm0.27}$ & $\mathbf{72.16\pm0.33}$ \\
16 & Light long-tail & $70.63\pm0.32$ & $70.38\pm0.26$ \\
16 & Heavy long-tail & $68.15\pm0.33$ & $66.66\pm0.19$ \\
\midrule
64 & Balanced & $\mathbf{75.01\pm0.29}$ & $\mathbf{76.06\pm0.33}$ \\
64 & Light long-tail & $73.83\pm0.27$ & $73.71\pm0.32$ \\
64 & Heavy long-tail & $71.02\pm0.21$ & $69.27\pm0.20$ \\
\bottomrule
\end{tabular}
\end{table}

Across both evaluations, Balanced at rank $16$ outperforms Heavy at rank $64$, using only one
quarter of the adapter rank. These comparisons show how better coverage can improve the
effectiveness of a smaller adapter, directly supporting the joint consideration of data composition
and rank in resource-efficient LoRA fine-tuning.

\subsection{Adapter-location robustness}
\label{sec:app-robustness}

We evaluate the Balanced-versus-Heavy coverage comparison at several adapter locations on MNLI.
For layers 5, 7, and 8, gradient sketches are recomputed at the corresponding query projection.
The all-layer query/value run evaluates placement transfer using the fixed layer-10-selected
subsets.

\begin{table}[!htbp]
\centering
\small
\caption{Adapter-location robustness on MNLI. Gains are Balanced minus Heavy test accuracy
in percentage points.}
\label{tab:app-robustness}
\begin{tabular}{lr}
\toprule
Adapter location & Balanced gain \\
\midrule
Layer-5 query & $+0.890$ \\
Layer-7 query & $+1.063$ \\
Layer-8 query & $+1.939$ \\
All-layer query/value & $+1.124$ \\
\bottomrule
\end{tabular}
\end{table}

Balanced improves test accuracy at every tested location, by $0.890$--$1.939$ percentage points
for single-layer query adapters and $1.124$ points for all-layer query/value adaptation. The
all-layer gain uses the fixed layer-10-selected subsets, showing that the coverage benefit
transfers across adaptation configurations.

\FloatBarrier